\documentclass[10pt]{article} % For LaTeX2e
\usepackage{amsmath,amsfonts,bm,bbm,amsthm,mathtools,amssymb}

\def\eqref#1{equation~\ref{#1}}
\def\1{\bm{1}}

\DeclareMathAlphabet{\mathsfit}{\encodingdefault}{\sfdefault}{m}{sl}
\SetMathAlphabet{\mathsfit}{bold}{\encodingdefault}{\sfdefault}{bx}{n}

\newcommand{\Eb}{\mathbb{E}}
\newcommand{\Pb}{\mathbb{P}}
\newcommand{\Qb}{\mathbb{Q}}
\newcommand{\Rb}{\mathbb{R}}

\newcommand{\ThetaB}{\boldsymbol{\Theta}}

\newcommand{\Sc}{\mathcal{S}}
\newcommand{\Ac}{\mathcal{A}}
\newcommand{\Oc}{\mathcal{O}}
\newcommand{\Nc}{\mathcal{N}}

\newcommand{\Fc}{\mathcal{F}}
\newcommand{\Hc}{\mathcal{H}}
\newcommand{\Pc}{\mathcal{P}}

\newcommand{\mbf}{\mathbf{m}}
\newcommand{\Mbf}{\mathbf{M}}
\newcommand{\TV}{\mathtt{TV}}
\newcommand{\norm}[1]{\left\lVert#1\right\rVert}

\usepackage{hyperref}
\usepackage{url}

\usepackage{subfigure}
\usepackage{geometry}
\usepackage{natbib}

\usepackage{cleveref}
\usepackage{algorithm}
\usepackage{algpseudocode}
\usepackage{overpic}
\usepackage{enumitem}
\usepackage{subfigure, wrapfig}

\newtheorem{lemma}{Lemma}
\newtheorem{prop}{Proposition}
\newtheorem{theorem}{Theorem}
\newtheorem{definition}{Definition}

\newtheorem{assmp}{Assumption}

\newtheorem{example}{Example}
\Crefname{assmp}{Assumption}{Assumptions}

\newcommand{\yl}[1]{{\color{red}(YL: #1)}}

\newcommand{\jing}[1]{{\color{magenta}(JY: #1)}}

\newcommand{\RM}[1]{\left(\romannumeral#1\right)}
\newcommand{\DTV}{\mathtt{D}_{\TV}}

\title{Provable Benefits of Multi-task RL under Non-Markovian Decision Making Processes}

\author{Ruiquan Huang\thanks{
  The Pennsylvania State University,
  State College, PA, 16801, USA. \texttt{\{rzh5514, yangjing\}@psu.edu} },\quad Yuan Cheng\thanks{National University of Singapore, 119077, Singapore. \texttt{yuan.cheng@u.nus.edu, vtan@nus.edu.sg}}, \quad Jing Yang$^*$,\\
  \quad Vincent Tan$^{\dagger}$, \quad 
   Yingbin Liang\thanks{The Ohio State University, Columbus, OH, 43210, USA. \texttt{liang.889@osu.edu}},\quad 
}
\date{}
\begin{document}

\maketitle

\begin{abstract}

In multi-task reinforcement learning (RL) under Markov decision processes (MDPs), the presence of shared latent structures among multiple MDPs has been shown to yield significant benefits to the sample efficiency compared to single-task RL. In this paper, we investigate whether such a benefit can extend to more general sequential decision making problems, such as partially observable MDPs (POMDPs) and more general predictive state representations (PSRs). The main challenge here is that the large and complex model space makes it hard to identify what types of common latent structure of multi-task PSRs can reduce the model complexity and improve sample efficiency.
To this end, we posit a {\em joint model class} for tasks and use the notion of $\eta$-bracketing number to quantify its complexity; this number also serves as a general metric  to capture the similarity of tasks and thus determines the benefit of multi-task over single-task RL. We first study  upstream multi-task learning over PSRs, in which all tasks share the same observation and action spaces. We propose a provably efficient algorithm  UMT-PSR for finding near-optimal policies for all PSRs, and demonstrate that the advantage of multi-task learning manifests if the joint model class of PSRs has a smaller $\eta$-bracketing number compared to that of individual single-task learning. We also provide several example multi-task PSRs with small $\eta$-bracketing numbers, which reap the benefits of multi-task learning. We further investigate downstream learning, in which the agent needs to learn a new target task that shares some commonalities with the upstream tasks via a similarity constraint. By exploiting the learned PSRs from the upstream, we develop a sample-efficient algorithm that provably finds a near-optimal policy. Upon specialization to the examples used to elucidate the $\eta$-bracketing numbers, our downstream results further highlight the benefit compared to directly learning the target PSR without upstream information. Ours is the first theoretical study that quantifies the benefits of multi-task RL with PSRs over its single-task counterpart. 
\end{abstract}

\section{Introduction}

Multi-task sequential decision making, or multi-task reinforcement learning (MTRL) is a subfield of reinforcement learning (RL) that extends the learning process across multiple tasks. Many real-world applications can be modeled by MTRL. For instance, in robotics and autonomous driving, different types of robots and vehicles in a shared environment can have different observational capabilities based on their sensors and learning goals. Other applications include personalized healthcare, weather forecasting across different regions, and manufacturing quality control on different types of products.
%
%In MTRL, the tasks can be related in various ways. They could share the same underlying environment dynamics but have different goals, or they could be operating under different environments while attempting to achieve the same objective. These variations give rise to a rich set of problems and methodologies in MTRL. 
%
The fundamental idea behind MTRL is to leverage the inherent similarities 
%and differences 
among a set of tasks in order to improve the overall learning efficiency and performance. 
%By sharing knowledge across different but related tasks, the agent can generalize its learning, which can significantly reduce the time and samples required for learning each task independently.
For Markov decision processes (MDPs), 
%the most celebrated model of RL, 
a line of works~\citep{pathak2017curiosity,tang2017exploration,oord2018representation,laskin2020curl,lu2021power,cheng2022provable,agarwal2022provable,pacchiano2022joint} have explored multi-task representation learning and shown its benefit both practically and theoretically.

However, it is still an open question whether such a benefit can extend to more general sequential decision making problems, even in partially observable MDPs (POMDPs), let alone more general predictive state representations (PSRs). In this context, it is even unclear: 
\begin{center}
    \textit{When can latent similarity structure encompassed by multiple PSRs be potentially  beneficial?}
\end{center}
 The challenges mainly emanate from two aspects. First, \textit{the large and complex model space makes it hard to identify what types of common latent structure of multi-task PSRs can reduce the model complexity}. The non-Markovian property of these problems implies that the sufficient statistics or belief about the current environmental state encompasses all the observations and actions from past interactions with the environment. This dramatically increases the statistical complexity. Even for a finite observation space and action space, model complexity can be exponentially large in the number of observations and actions. Such a complex parameter space makes it difficult to identify what types of latent similarity structure of multi-task PSRs reduce the model complexity.
 Second, \textit{reduced {\bf model complexity} does not necessarily result in benefit in {\bf statistical efficiency} gain of RL}. In RL, model learning and data collection are intertwined. The agent has to choose an exploration policy in each iteration based on the model learned in the past. Such iterative process introduces temporal dependence to the collected data, which makes the analysis of multi-task PSRs complicated.

 In this paper, we answer the question above with upstream multi-task learning and downstream transfer learning. We summarize our contributions below.

 \begin{enumerate}[topsep=0pt, left=0pt] 
     \item To deal with the first challenge, we propose a unified approach to characterize the effect of task similarity on model complexity by introducing the notion of the {\em $\eta$-bracketing number} for the {\em joint model space} of multiple tasks. Regardless of whether the concrete form of task similarity is implicit or explicit, desirable task similarity should contribute to reduce the $\eta$-bracketing number compared to that without similarity structures. This significantly generalizes  existing studies of multi-task MDPs that considered only specific task similarity structures.
     \item We deal with the second challenge in both upstream and downstream learning. For the former, we propose a novel multi-task PSRs algorithm called UMT-PSR, which features a pairwise additive distance-based optimistic planning and exploration as well as confidence set construction based on the bracketing number of the {\em joint} model class. We then prove that if the bracketing number of the multi-task model class normalized by the number of tasks is lower than that of a single task, UMT-PSR benefits from multi-task learning with these novel designs. We then provide several specific multi-task POMDP/PSR examples with low bracketing number to demonstrate that UMT-PSR is often more efficient than single-task learning. 
     \item We further employ the upstream learning to downstream learning by connecting upstream and downstream models via similarity constraints. We show that the downstream learning can identify a near-accurate model and find a near-optimal policy. Upon specialization to the examples used to elucidate the $\eta$-bracketing numbers, our downstream results further highlight the benefit in comparison to directly learning parameters of PSRs without upstream information. Our analysis here features a novel technique of using R\'enyi divergence to measure the approximation error which guarantees the sub-optimality bound without requiring the realizability condition. 
 \end{enumerate}

Our work is the first theoretical study that characterizes the benefits of multi-task RL with PSRs/POMDPs over its single-task counterpart. %Further, our unified way of quantifying the general task similarity via the bracketing number for general multi-task problems significantly generalizes existing studies of multi-task MDPs that considered only specific task similarity structures.

\section{Related Work}
 % and meta RL
\textbf{MTRL under MDPs:} 
Multitask representation learning and transfer learning have been extensively studied in RL, particularly under MDPs. 
% \cite{DBLP:conf/icml/HuCJL021} explored the benefit of multi-task learning under linear MDPs. 
\cite{DBLP:conf/icml/AroraDKLS20} demonstrated that representation learning can reduce sample complexity for imitation learning.
% , and \cite{DBLP:journals/ia/CalandrielloLR15} studied multi-task learning in linear MDPs with known representation. 
\cite{DBLP:conf/icml/HuCJL021} analyzed MTRL with low inherent Bellman error~\citep{zanette2020learning} and known representation. \cite{zhang2021provably} studied multi-task learning under similar transition kernels. In contrast, \cite{DBLP:conf/uai/BrunskillL13} studied the benefit of MTRL when each task is independently sampled from a distribution over a finite set of MDPs. 
% Meta-learning idea has also been applied to study MTRL, and this line of research is known as meta-RL. Meta-RL for linear mixture models was explored in \cite{DBLP:journals/corr/abs-2201-08732}. 
Recent studies have also considered the case where all tasks share a common representation, including \cite{DBLP:conf/iclr/DEramoTBR020} which demonstrated the convergence rate benefit on value iteration, and \cite{lu2021power} which proved the sample efficiency gain of MTRL under low-rank MDPs. 
%However, none of the aforementioned works took the impact of sequential exploration and temporal dependence in data into account. Specifically, \cite{lu2021power} relied on the accessibility of a generative model, while \cite{DBLP:conf/iclr/DEramoTBR020} assumed that the dataset is provided in advance. 
Some recent work further took the impact of sequential exploration and temporal dependence in data into account. Considering sequential exploration with shared unknown representation, \cite{cheng2022provable,agarwal2022provable} studied reward free MTRL under low-rank MDPs as upstream learning and applied the learned representation from upstream to downstream RL. \cite{pacchiano2022joint} focused on a common low-dimensional linear representation and investigated MTRL under linearly-factored MDPs. \cite{lu2022provable} explored MTRL with general function approximation.

Note that all the above studies considered specific common model structures shared among tasks, whereas our paper proposes a unified way to characterize the similarity among tasks. Further, none of the existing studies considered multi-task POMDPs/PSRs, which is the focus of our paper.

%\ylfixed{cite the paper by Berkeley people}

%\ylfixed{what about \citet{zhang2021provably}}

% multi-task POMDPs \citep{li2009multi}

% \textbf{single-task RL under POMDPs:}

\textbf{Single-task RL with PSRs and general sequential decision making problems:} A general decision making framework PSR~\citep{littman2001predictive} was proposed to generalize MDPs and POMDPs. Since then, various approaches have been studied to make the problem tractable with polynomial sample efficiency. These methods include spectral type of techniques~\citep{boots2011closing,hefny2015supervised,jiang2018completing,zhang2022reinforcement}, methods based on optimistic planning and maximum log-likelihood estimators together with confidence set-based design~\citep{zhan2022pac,liu2022optimistic}, the bonus-based approaches~\citep{huang2023provably},  value-based actor-critic approaches~\citep{uehara2022provably}, posterior sampling methods~\citep{zhong2022posterior}.
\citet{chen2022partially} further improved the sample efficiency for previous work including OMLE~\citep{liu2022optimistic}, MOPS~\citep{agarwal2022model}, and E2D~\citep{foster2021statistical}.  

%PSRs were proposed to model general dynamic systems~\citep{littman2001predictive}. Using spectral techniques, previous works~\citep{boots2011closing,hefny2015supervised,zhang2022reinforcement} only derived polynomial sample complexity for PSRs even with observability assumption. Recently, adopting optimistic planning and maximum log-likelihood estimators to design confidence set-based algorithms, \cite{zhan2022pac,liu2022optimistic} learned regular, and well-conditioned PSRs, respectively. From another perspective, \citet{huang2023provably} provided a bonus-based approach for learning PSRs. For general sequential decision making problems, \citet{uehara2022provably} proposed a PO-bilinear class capturing POMDPs and PSRs and then designed an value based actor-critic algorithm.  \citet{zhong2022posterior} proposed a new low generalized Eluder coefficient framework and addressed it via posterior sampling. \citet{chen2022partially} further improved the sample efficiency for previous work including OMLE~\citep{liu2022optimistic}, MOPS~\citep{agarwal2022model}, and E2D~\citep{foster2021statistical}.  
%\yl{rewrite this paragraph }

\section{Preliminaries}

%, while the hellinger-squared distance is defined as $\mathtt{D}_{\mathtt{H}}^2 (\Pb,\Qb) = \frac{1}{2}\sum_x (\sqrt{\Pb(x)} - \sqrt{\Qb(x)} )^2$. 
%\cy{where to put the notation}

% In this section, we first introduce the formulation of the low-rank sequential decision making problem in \Cref{subsec: dm prob}, and then present the learning procedure of multi-task PSRs, including upstream learning and downstream learning in \Cref{subsec: up} and \Cref{subsec: down}, respectively.

\textbf{Notations.} For any positive integer $N$, we use $[N]$ to denote the set $\{1,\cdots,N\}$. For any vector $x$, %$\|x\|_A$ stands for $\sqrt{x^{\top}Ax}$, and 
the $i$-th coordinate of $x$ is represented as $[x]_i$. For a set $\mathcal{X}$, the Cartesian product of $N$ copies of $\mathcal{X}$ is denoted by $\mathcal{X}^N$. For probability distributions $\mathbb{P}$ and $\mathbb{Q}$ supported on a countable set $\mathcal{X}$, the total variation distance between them   is $\mathtt{D}_{\TV}\left(\Pb,\Qb\right) = \sum_x|\Pb(x)-\Qb(x)|$, and the R\'enyi divergence of order $\alpha$, for $\alpha>1$, between them is  $\mathtt{D}_{\mathtt{R},\alpha}(\mathbb{P},\mathbb{Q}) = \frac{1}{\alpha-1}\log \Eb_{\mathbb{P}}[(\mathrm{d}\mathbb{P}/\mathrm{d}\mathbb{Q})^{\alpha-1}]$. 
 
\subsection{The Non-markovian decision making problem}\label{subsec: dm prob}

We consider an episodic decision making process, which is not necessarily Markovian, with an observation space $\mathcal{O}$ and a finite action space $\mathcal{A}$. We assume that the process is episodic and each episode contains $H$ steps, i.e., with horizon $H$. At each step, the evolution of the process is controlled by an underlying distribution $\Pb$, where $\Pb(o_h|o_1,\ldots,o_{h-1},a_1,\ldots,a_{h-1})$ is the probability of visiting $o_h$ at step $h$ given that the learning agent has observed $o_t\in\mathcal{O}$ and taken action $a_t\in\mathcal{A}$ for previous steps $t\in [h-1]$. And the learning agent receives a reward at each episode determined by the reward function $R: (\mathcal{O}\times\mathcal{A})^H \rightarrow [0,1]$. We denote such a process compactly as $\mathtt{P} = (\mathcal{O},  \mathcal{A}, H, \Pb, R)$. For each step $h$, we denote historical trajectory as $\tau_h:=(o_1,a_1,\ldots,o_h,a_h)$,
the set of all possible historical trajectories as $\mathcal{H}_h = (\mathcal{O}\times\mathcal{A})^{h}$, the future trajectory as $\omega_h:=(o_{h+1},a_{h+1},\ldots, o_H,a_H)$, and the set of all possible future trajectories as $\Omega_h = (\mathcal{O}\times\mathcal{A})^{H-h}$. %We define a policy $\pi=\{\pi_h\}_{h=1}^H$ as a collection of $H$ distributions. 

The agent interacts with the environment in each episode as follows. At step 1, a fixed initial observation $o_1$ is drawn. At each step $h \in [H]$, due to the non-Markovian nature, the action selection and  environment transitions are based on whole history information. Specifically, the agent can choose an action $a_h$ based on the history $\tau_{h-1}$ and the current observation $o_h$ with strategy $\pi_h(a_h|\tau_{h-1},o_h)$. We denote such a strategy as a policy, and collect the policies over $H$ steps into $\pi=\{\pi_h\}_{h=1}^H$, and denote the set of all feasible policies as $\Pi$.
%Specially, the agent takes action $a_h$ according to $\pi_h(a_h|\tau_{h-1},o_h)$, and
Then the environment takes a transition to $o_{h+1}$ based on $\Pb(o_{h+1}|\tau_{h})$. The episode terminates after $H$ steps.

%The policy $\pi=\{\pi_h\}_{h=1}^H$ of the agent is a collection of $H$ distributions where $\pi_h(a_h|\tau_{h-1},o_h)$ is the probability of choosing action $a_h$ at time step $h$ given the history $\tau_{h-1}$ and the current observation $o_h$. 
 For any historical trajectory $\tau_h$, we further divided it into    $\tau_h^o = (o_{1},\ldots,o_h)$ and $\tau_h^a = (a_{1},\ldots,a_h)$ which is observation and action sequences contained in $\tau_h$, respectively. Similar to $\tau_h$, for the future trajectories $\omega_h$, we denote $\omega_h^o$ as the observation sequence in $\omega_h$, and $\omega_h^a$ as the action sequence in $\omega_h$. For simplicity, we write $\pi(\tau_h) = \pi(a_h|o_h,\tau_{h-1})\cdots\pi(a_1|o_1)$ to denote the probability of choosing the sequence of actions $\tau_h^a$ given the observations $\tau_h^o$ under the policy $\pi$. %and similarly we use $\pi(\omega_h)$ to denote the probability of choosing the sequence of actions $\omega_h^a$ given the observations $\omega_h^o$ under the policy $\pi$. 
 We denote $\Pb^{\pi}$ as the distribution of the trajectories induced by the policy $\pi$ under the dynamics $\Pb$. The value function of a policy $\pi$ under $\Pb$ and the reward $R$ is denoted by $V_{\Pb,R}^{\pi} = \Eb_{\tau_H\sim\Pb^{\pi}}[R(\tau_H)]$. The primary learning goal is to find an $\epsilon$-optimal policy $\bar{\pi}$, which is one that satisfies $\max_{\pi}V_{\Pb,R}^{\pi} - V_{\Pb,R}^{\bar{\pi}} \leq \epsilon$.

Given that addressing a general decision-making problem entails an exponentially large sample complexity in the worst case, this paper focuses on the {\it low-rank} class of problems as in \cite{zhan2022pac,liu2022optimistic,chen2022partially}. Before formal definition of the low-rank problem, we introduce the dynamics matrix $\mathbb{D}_h \in \mathbb{R}^{|\mathcal{H}_h|\times |\Omega_h|}$ for each $h$, where we use $\tau_h \in \mathcal{H}_h$ and $\omega_h \in \Omega_h$ to index the rows and columns of the matrix $\mathbb{D}_h$, respectively, and the entry at the $\tau_h$-th row and $\omega_h$-th column of $\mathbb{D}_h$ equals to the conditional probability $\Pb(\omega_h^o, \tau_h^o | \tau_h^a, \omega_h^a)$.
\begin{definition}[Rank-$r$ sequential decision making problem]
    A sequential decision making problem is rank $r$ if for any $h$, the model dynamics matrix $\mathbb{D}_h$ has rank {at most} $r$.
\end{definition}

As a result, for each $h$, the probability of observing $\omega_h^{o}$ can be represented by a linear combination of probabilities on a set of future trajectories known to the agent called {\em core tests} $\mathcal{Q}_h = \{\mathbf{q}_{h}^{1},\ldots, \mathbf{q}_h^{d_h}\} \subset \Omega_h$, where $d_h\geq r$. %, which determines the system dynamics of future trajectories conditioning on any historical trajectories. 
Specifically, there exist functions $\mbf: \Omega_h \to \Rb^{d_h}, \psi: \Hc_h \to \Rb^{d_h}$ such that $\RM{1}$ the value of the $\ell$-th coordinate of $\psi(\tau_h)$ equals to the conditional probability 
$\Pb(\mathbf{o}_h^{\ell},\tau_h^o | \mathbf{a}_h^{\ell}, \tau_h^a)$ on $(\mathbf{q}_h^{\ell}, \tau_h )$, where $\mathbf{o}_h^{\ell}$ and $\mathbf{a}_h^{\ell}$ to denote the observation and the action sequences of $\mathbf{q}_h^{\ell}$, and $\RM{2}$ for any $\omega_h \in \Omega_h, \tau_h \in \Hc_h$, the conditional probability can be factorized as 
\begin{align}
    \textstyle \Pb(\omega_h^o,\tau_h^o | \tau_h^a,\omega_h^a) =  \mbf(\omega_h)^{\top}\psi(\tau_h). \label{eq: lr-decom}
\end{align}

% In other words, for any $\omega_h \in \Omega_h, \tau_h \in \Hc_h$, $\mathcal{Q}_h$ allows the system dynamics to be factorized as, where $\mathbf{m}(\omega_h), \psi(\tau_h)\in\mathbb{R}^{d_h}$ and 

%Specifically, the value of the $\ell$-th coordinate of $\psi(\tau_h)$ equals to the joint probability of $\tau_h$ and the $\ell$-th core test $\mathbf{q}_h^{\ell}$. Formally, if we use $\mathbf{o}_h^{\ell}$ and $\mathbf{a}_h^{\ell}$ to denote the observation and the action sequences of $\mathbf{q}_h^{\ell}$, then $  [\psi(\tau_h)]_{\ell}= \Pb(\mathbf{o}_h^{\ell}, \tau_h^o | \tau_h^a, \mathbf{a}_h^{\ell})$.
%Moreover, $r \leq d_h$ by the definition of matrix rank.

{\bf Predictive State Representation.} %It is known that low-rank decision making problem admits Predictive State Representations (PSRs)\citep{liu2022optimistic,zhong2022posterior}. 
Following from Theorem C.1 in \cite{liu2022optimistic}, 
given core tests $\{\mathcal{Q}_h\}_{h=1}^H$, any low rank decision making problem admits a (self-consistent) \textbf{predictive state representation} (PSR) $\theta = \left\{\left(\phi_h,\Mbf_h\right)\right\}_{h=1}^H$, such that Eq.~\ref{eq: lr-decom} can be reparameterized by $\theta$. Mathematically, For any $h \in [H]$, $\tau_h \in \Hc_h, \omega_h \in \Omega_h$:
\begin{align*}
    \textstyle \mbf(\omega_h)^{\top} = \phi_H^{\top}\Mbf_H(o_H,a_H)\cdots\Mbf_{h+1}(o_{h+1},a_{h+1}),\quad \psi(\tau_h) = \Mbf_h(o_h,a_h)\cdots\Mbf_1(o_1,a_1)\psi_0, 
\end{align*}
and $\sum_{(o_h,a_h)\in\Oc\times\Ac}\phi_{h+1}^{\top}\Mbf_h(o_h,a_h) = \phi_h^{\top}$. 
For ease of the presentation, we assume $\psi_0$ is known.\footnote{The sample complexity of learning $\psi_0$ if it is unknown is relatively small compared to the learning of the other parameters~\citep{liu2022optimistic}.} 

The following assumption is standard in the literature~\citep{liu2022optimistic,chen2022partially}.

\begin{assmp}[$\gamma$-well-conditioned PSR]\label{assmp:well-condition} We assume any PSR $\theta = \{(\phi_h,\Mbf_h)\}_{h=1}^H$ considered in this paper is $\gamma$-well-conditioned for some $\gamma > 0$, i.e.
%, a self-consistent PSR $\theta = \{(\phi_h,\Mbf_h)\}_{h=1}^H$ is said to be $\gamma$-well-conditioned if
\begin{align}
    \textstyle \forall h \in [H],~~ \max_{x\in\mathbb{R}^{d_h}:\|x\|_1\leq 1} \max_{\pi\in\Pi} \sum_{\omega_h \in \Omega_h}\pi(\omega_h)|\mbf(\omega_h)^{\top}x|\leq \frac{1}{\gamma}.
\end{align}
%A PSR $\theta$ is {\em well-conditioned} if it is $\gamma$-well-conditioned for some $\gamma \in (0,\infty)$.
\end{assmp}

%\Cref{assmp:well-condition} guarantees that the error of estimating $\theta$ does not significantly blow up when the estimation error $x$ of estimating the probability of core tests is small. We denote $\Theta$ the class of all well-conditioned PSRs.\cy{error and definition of $\Theta$}

In the following context, we use %$\theta = \{(\phi_h,\Mbf_h)\}_{h=1}^H$ to denote the PSR of a low-rank sequential decision making problem and 
$\Pb_{\theta}$ to indicate the model determined by the PSR $\theta$. For simplicity, we denote $V_{\Pb_{\theta}, R}^{\pi}$ as $V_{\theta,R}^{\pi}$. 
%We omit the subscript $h$ of $\Mbf_h,\mbf_h,\psi_h$ when the dependence over $h$ is clear from $o_h, a_h, \tau_h$. 
Moreover, let $\mathcal{Q}_h^A = \{\mathbf{a}_h^{\ell}\}_{\ell=1}^{d_h}$ be the action sequence set from core tests which is constructed by eliminating any repeated action sequence. The set $\mathcal{Q}_h^A$ is also known as the core action sequence set. The set of all rank-$r$ and $\gamma$-well-conditioned PSRs is denoted by $\Theta$.
% \subsection{Multi-task PSRs Learning and Transfer Learning}\label{sec: upstream and downstream}

\subsection{Upstream Multi-task Learning}\label{subsec: up}
In \textbf{upstream multi-task} learning, the agent needs to solve $N$ 
% well-conditioned 
low-rank decision making problems (also known as source tasks) at the same time instead of only one single problem (task). The set of $N$ source tasks is denoted by $\{\mathtt{P}_n\}_{n=1}^N$, where $\mathtt{P}_n = (\mathcal{O},\mathcal{A},H, \Pb_{\theta_n^*}, R_n)$, and $\theta_n^*\in\Theta$.\footnote{For simplicity, we assume all tasks have the same rank and $\gamma$, but have different core test sets. The extension to different ranks and $\gamma$'s is straightforward.} In other words, all $N$ tasks are identical except for their model parameters $\theta_n^* = \{(\phi_h^{n,*}, \Mbf_{h}^{n,*})\}_{h=1}^H$, and reward functions $R_n$. Moreover, we denote the model class of multi-task
% well-conditioned 
PSRs as  $\boldsymbol{\Theta}_{\mathrm{u}}$ (the subscript stands for \underline{u}pstream), a subset of $\Theta^N$.

The goal of the upstream learning consists of two parts: (i) Finding near-optimal policies for all $N$ tasks on average. Mathematically, given an accuracy level $\epsilon$, the set of $N$ policies that are produced by the algorithm $\{\bar{\pi}^1,\ldots,\bar{\pi}^N \}$ should satisfy $\frac{1}{N} \sum_{n=1}^N(\max_{\pi}V_{\theta_n^* ,R_n}^{\pi} - V_{\theta_n^* ,R_n}^{\bar{\pi}^n} ) \leq \epsilon$;  (ii) Characterizing the theoretical benefit of multi-task PSRs learning in terms of the sample complexity, compared to learning each task individually.

\subsection{Bracketing Number of Joint Parameter Space}

One critical factor that affects the efficiency of multi-task learning compared to separate task learning is the presence of shared latent structure among the tasks, %Regardless of whether such structures are explicit or implicit, they may contribute to 
which yields a reduced model space in multi-task PSRs learning, as compared to separately learning single tasks over the Cartesian product of $N$ model spaces (see \Cref{fig:1} for an illustration in 2 dimensions). %\textcolor{brown}{the {\em boundary case}: separately single-task learning over a Cartesian product of $N$ model spaces .} 
Consequently, this reduction in model complexity can ultimately lead to improved sample efficiency. 
%To evaluate the influence of latent structure on the model complexity of PSRs and further characterize the benefit on sample complexity, we introduce the notion of {\bf $\eta$-bracketing number} as following.
Unlike the specific shared model structures among multiple tasks that the previous works studied, such as shared representation in \citet{cheng2022provable} and similar transition kernels in \citet{zhang2021provably}, here we focus on a general shared model space and use the notion of the {\bf $\eta$-bracketing number} to quantify the complexity of the joint model space. Such a notion plays a central role in capturing the benefit of multi-task PSR learning over single-task learning.

\begin{wrapfigure}{r}{0.36\textwidth}
%\begin{figure}[htbp]
\centering  
\vspace{-20pt}
\subfigure[Independent learning of each task where  the joint model class is $\Pc^2$]{   
\centering    
\begin{overpic}[width=0.34\textwidth]{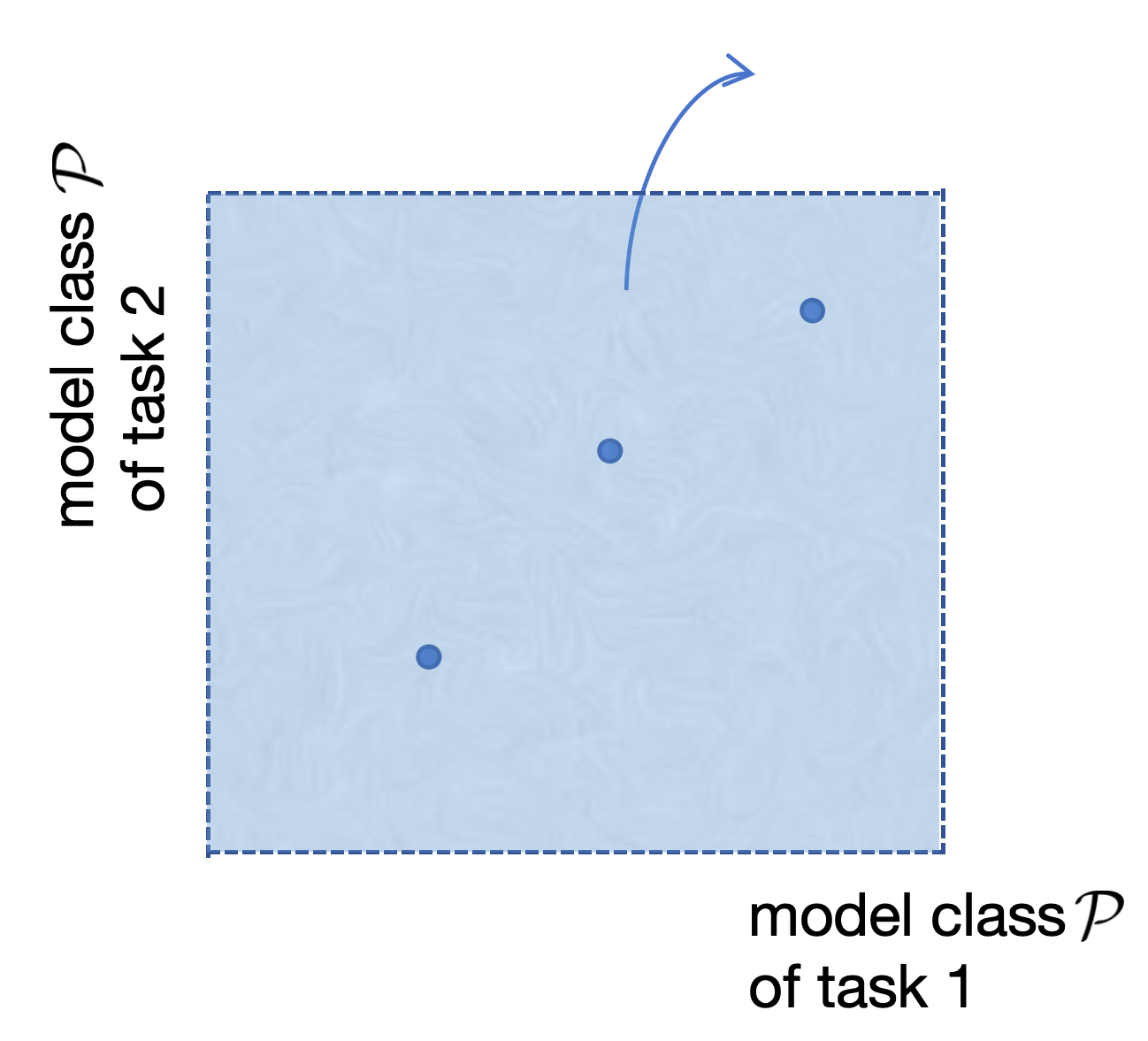} 
    % \put(35,52){\fontsize{3pt}{10pt}\selectfont$\mathbf{f}^*=\left(P^{(*,1)},  P^{(*,2)}\right)$}
    \put(35,52){\fontsize{3pt}{10pt}\selectfont$\left(P^{(*,1)},  P^{(*,2)}\right)$}
    \put(67,83){\footnotesize$\mathcal{P}^2$}
\end{overpic} 
\label{fig:1-a}\vspace{-25pt}
}
\subfigure[Joint learning of tasks with shared   latent  model structure. The joint model class is $\Fc_{\mathrm{u}}$, a strict subset of $\Pc^2$.]{ 
\centering   
\begin{overpic}[width=0.36\textwidth]{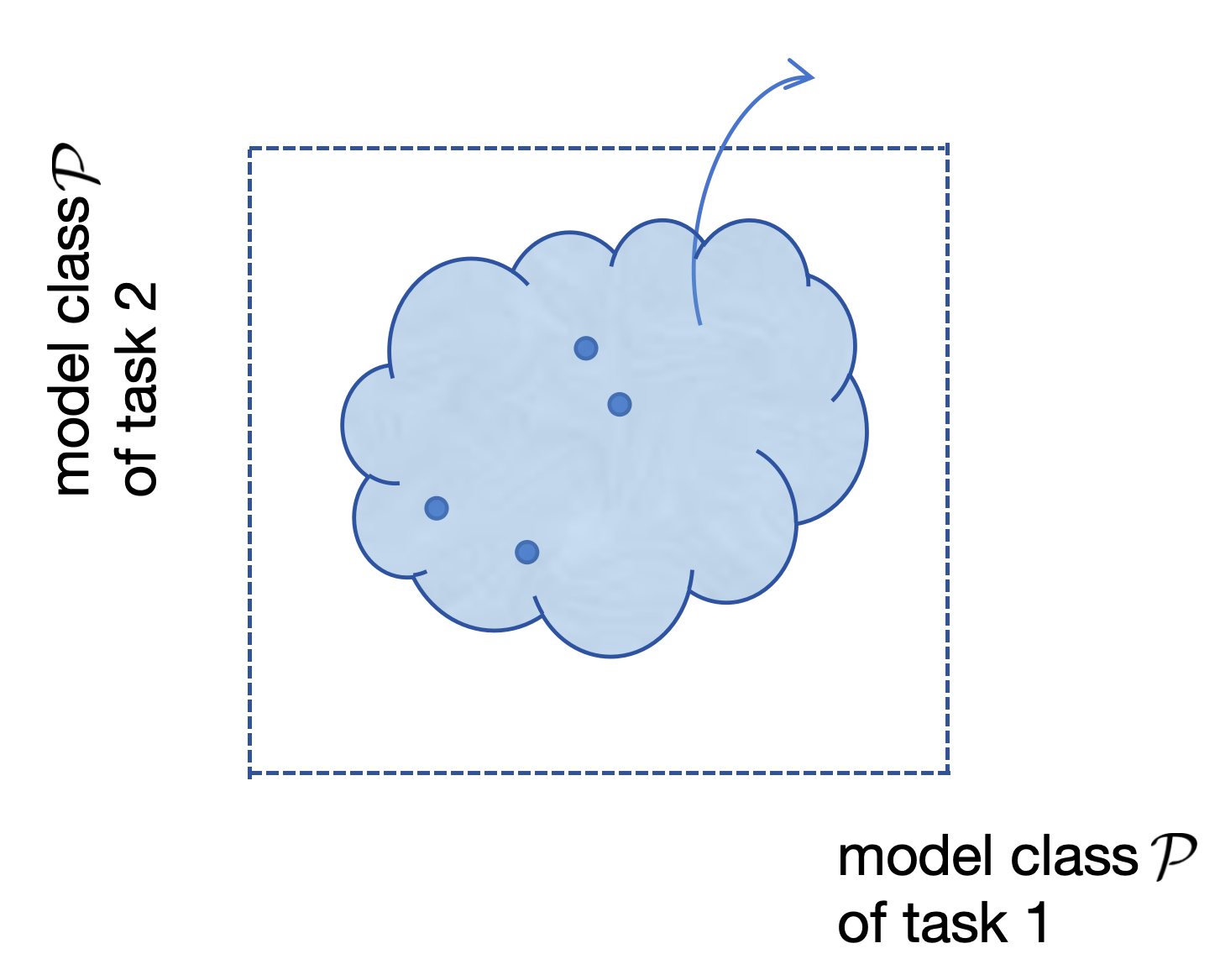} 
    % \put(32,38){\fontsize{2pt}{10pt}\selectfont$\mathbf{f}^*=\left(P^{(*,1)},  P^{(*,2)}\right)$}
    \put(33,42){\fontsize{2pt}{10pt}\selectfont$\left(P^{(*,1)},  P^{(*,2)}\right)$}
    \put(68,74){\footnotesize$\Fc_{\mathrm{u}}\subsetneq \mathcal{P}^2$}
\end{overpic} 
\label{fig:1-b} 
} 
\vspace{-10pt}
\caption{Reduction in $\eta$-bracketing number when $\Fc_{\mathrm{u}}\subsetneq\Pc^2$ }
\vspace{-20pt}
\label{fig:1}    
%\end{figure}
\end{wrapfigure}

We start with a domain $\mathcal{X}$ and a single task function class $\mathcal{F}$, in which each element $f:\mathcal{X}\rightarrow \mathbb{R}_+$. For the multi-task case, the function class is a subset $\boldsymbol{\mathcal{F}}_{\mathrm{u}}$ of $\mathcal{F}^{N}$. 

\begin{definition}[$\eta$-Bracketing number of vector-valued function class $\boldsymbol{\mathcal{F}}_{\mathrm{u}}$ w.r.t. $\norm{\cdot}$]
   Given two vector-valued functions $\mathbf{l}$ and $\mathbf{g}: \mathcal{X} \to \Rb^N_{+}$, the bracket $[\mathbf{l},\mathbf{g}]$ is the set of all functions $\mathbf{f} \in \boldsymbol{\mathcal{F}}_{\mathrm{u}}$ satisfying $\mathbf{l}\leq \mathbf{f}\leq\mathbf{g}$.\footnote{We write that two vectors $\mathbf{a},\mathbf{b}$ satisfy  $\mathbf{a}\leq \mathbf{b}$ if  $\mathbf{b} - \mathbf{a}$ is coordinate-wise nonnegative.} An $\eta$-bracket is a bracket $[\mathbf{l},\mathbf{g}]$ with $\norm{\mathbf{g}-\mathbf{l}} < \eta$. The bracketing number $\mathcal{N}_{\eta}(\boldsymbol{\Fc}_{\mathrm{u}},\norm{\cdot})$ is the minimum number of $\eta$-brackets needed to cover $\boldsymbol{\Fc}_{\mathrm{u}}$.\footnote{We say that a collection of   sets $S_1\ldots,S_n$ {\em cover} a set $S$  if  $S\subset\cup_{i=1}^n S_i$.}
\end{definition}

In this paper, we are interested in the bracketing number of the {\em joint} model space, i.e., distribution spaces over $(\mathcal{O}\times\mathcal{A})^H$ parameterized by $\boldsymbol{\Theta}_{\mathrm{u}}$. For simplicity,
 we use $\Nc_{\eta}(\boldsymbol{\Theta}_{\mathrm{u}})$ to denote the $\eta$-bracketing number of $\{ (\mathbb{P}_{\theta_1},\ldots,\Pb_{\theta_N}) | \boldsymbol{\theta}\in\boldsymbol{\Theta}_{\mathrm{u}}\}$ w.r.t.\ the $\ell_\infty$ policy weighted norm $\norm{\cdot}_{\infty}^{\mathrm{p}}$, where the $\ell_\infty$ policy weighted norm between two vector-valued functions $\mathbf{l}=\{l_1,\ldots,l_N\}$ and $\mathbf{g}=\{g_1,\ldots,g_N\}$ defined on $(\Oc \times \Ac)^H$ is equal to $\norm{\mathbf{g} - \mathbf{l}}_{\infty}^{\mathrm{p}} = \max_{i \in [N]}\max_{\pi_i \in \Pi}\sum_{\tau_H}|l_i(\tau_H)-g_i(\tau_H)|\pi_i(\tau_H)$. As we will show later, a lower $\eta$-bracketing number of the joint model space results in a lower sample complexity in multi-task PSR learning. 
  %In multi-task PSRs learning, we focus on the model class with low $\eta$-bracketing number w.r.t. $\norm{\cdot}_{\infty}^{\mathrm{w}}$. Such model classes are expressive and encompass a variety of scenarios. To validate it, we present some examples here.   

  In practice, it is common tasks share certain common model structures and hence their joint model space will have a much lower $\eta$-bracketing number compared to the product of model spaces (i.e, treating the model of each task separately). We provide several such examples of non-Markovian decision processes 
  %informally below, and formal version can be found 
  in \Cref{subsec: example-up}. We provide more examples of MDPs with their $\eta$-bracketing numbers in Appx. \ref{Appd_sub: F-1}. Notably, there can be much richer scenarios beyond these examples. 
\subsection{Downstream Transfer Learning}\label{subsec: down}
In {\bf downstream learning}, the agent is assigned with a new target task $ \mathtt{P}_{0} = (\mathcal{O},\mathcal{A}, H, \Pb_{\theta_0^*}, R_0) $, where $\theta_0^*\in\Theta$, which shares some similarities with source tasks to benefit from upstream learning. Here, we capture the shared structure between upstream and downstream tasks via the {\bf similarity constraint}  $\mathtt{C}(\theta_0,\theta_1^*,\ldots\theta_N^*)\leq \boldsymbol{0}$ where  $\mathtt{C}:\Theta^{N+1}\rightarrow \mathbb{R}^{n_d}$, $n_d\in\mathbb{N}$. The similarity constraint establishes the relationship between the downstream target task and the upstream source tasks. Hence, the downstream model class is given by $\Theta_0^{\mathrm{u}} = \{\theta_0\in\Theta| \mathtt{C}(\theta_0,\theta_1^*,\ldots\theta_N^*)\leq \boldsymbol{0}\}$.  
%
%To utilize the knowledge from upstream learning, we define the downstream model class as $\Theta_0^{\mathrm{u}} = \{\theta_0\in\Theta| \mathtt{C}(\theta_0,\theta_1^*,\ldots\theta_N^*)\leq \boldsymbol{0}\}$, where $\mathtt{C}:\Theta^{N+1}\rightarrow \mathbb{R}^{n_d}$, $n_d\in\mathbb{N}$. The map $\mathtt{C}$ establishes the relationship between the downstream target task and the upstream source tasks. \yl{call $\mathtt{C}$ as connection map or something that reflects its nature of connecting upstream and downstream parameters?}
%
%We remark that almost all relationships in existing literature can be described in terms of a specific map $\mathtt{C}$. For example, \citet{cheng2022provable} consider the case when the downstream task shares the same representation as upstream tasks, which is equivalent to assuming $[\mathtt{C}(\theta_0,\ldots,\theta^*_N)]_n = \|\phi^{(*,n)} - \phi_0\|_2$, where $n\in[N]$  and $\phi^{(*,n)}$ is the representation of task $n$. We also note that existing works usually establish the same commonalities for the downstream and upstream, i.e., the commonalities shared by the downstream task and upstream tasks usually take the same form as that in the upstream tasks. In \citet{cheng2022provable}, if the upstream tasks share the same representation, then the downstream task also possesses the same representation. However, our setting allows different types of commonalities. For instance, while the upstream tasks share the same representation, the downstream task might only share similar transition dynamics as source tasks.
%
We note that the similarity constraint is general enough to capture various relationships between upstream and downstream tasks. For example, \citet{cheng2022provable} consider the case when the downstream task shares the same representation as upstream tasks, which is equivalent to assuming $[\mathtt{C}(\theta_0,\ldots,\theta^*_N)]_n = \|\phi^{(*,n)} - \phi_0\|_2$, where $n\in[N]$  and $\phi^{(*,n)}$ is the representation of task $n$. However, the similarity constraint allows much richer beyond the above example, for example, downstream tasks can have similar but not the same representations as the upstream, or may share only some representation features, but not all of them. 
%We also note that existing works usually establish the same commonalities for the downstream and upstream, i.e., the commonalities shared by the downstream task and upstream tasks usually take the same form as that in the upstream tasks. In \citet{cheng2022provable}, if the upstream tasks share the same representation, then the downstream task also possesses the same representation. However, our setting allows different types of commonalities. For instance, while the upstream tasks share the same representation, the downstream task might only share similar transition dynamics as source tasks.

The goal of the downstream learning is to find a near-optimal policy, by exploiting the constraint similarity with upstream tasks and utilizing upstream knowledge to achieve better sample efficiency compared with learning without upstream knowledge.  

%Equipped this assumption, the agent then interacts with $\mathtt{P}_{0}$ in an online manner to find a near-optimal policy for the new task. By utilizing the upstream knowledge, the agent is expected to achieve better sample efficiency in the downstream compared with learning without prior knowledge.  

\if{0}
\subsection{Bracketing number}

To explicitly characterize the benefit brought by the constraints, we introduce the notion of bracketing number \citep{geer2000empirical}, which serves as a complexity measure of function classes.

Suppose a single-task function class is $\mathcal{F}$, in which each element $f:\mathcal{X}\rightarrow \mathbb{R}_+$. A multi-task function class is a subset $\boldsymbol{\mathcal{F}}_{\mathrm{u}}$ of $\mathcal{F}^{N}$. We also refer $\mathcal{X}$ to the domain of the function class.

%\textcolor{red}{Connect $\mathcal{F}$ to $\{f_1, \ldots, f_N\}$, e.g., say that $\{f_1, \ldots, f_N\}\in\mathcal{F}$ . Each $l_i, g_i$ is a nonnegative function, i.e., $l_i:(\mathcal{O}\times\mathcal{A})^H\to \mathbb{R}_+$. }

\begin{definition}[$\eta$-Bracketing number of vector-valued function class $\boldsymbol{\mathcal{F}}_{\mathrm{u}}$ w.r.t. $\norm{\cdot}$]
   Given two vector-valued functions $\mathbf{l}$ and $\mathbf{g}: \mathcal{X} \to \Rb^N_{+}$, the bracket $[\mathbf{l},\mathbf{g}]$ is the set of all functions $\mathbf{f} \in \boldsymbol{\mathcal{F}}_{\mathrm{u}}$ satisfying $\mathbf{l}\leq \mathbf{f}\leq\mathbf{g}$.\footnote{If two vectors $\mathbf{a},\mathbf{b}$ satisfy that $\mathbf{a}\leq \mathbf{b}$, then each coordinate of $\mathbf{a} - \mathbf{b}$ is less than or equal to 0.} An $\eta$-bracket is a bracket $[\mathbf{l},\mathbf{g}]$ with $\norm{\mathbf{g}-\mathbf{l}} < \eta$. The bracketing number $\mathcal{N}_{\eta}(\boldsymbol{\Fc}_{\mathrm{u}},\norm{\cdot})$ is the minimum number of $\epsilon$-brackets needed to cover\footnote{If sets $S_1\ldots,S_n$ cover a set $S$, then $S\subset\cup_{i=1}^n S_i$.} $\boldsymbol{\Fc}_{\mathrm{u}}$.
\end{definition}

In this paper, we are interested in the bracketing number of the distribution spaces over $(\mathcal{O}\times\mathcal{A})^H$ parameterized by $\boldsymbol{\Theta}_{\mathrm{u}}$. For simplicity,
we use $\Nc_{\eta}(\boldsymbol{\Theta}_{\mathrm{u}})$ to denote the $\eta$-bracketing number of $\{ (\mathbb{P}_{\theta_1},\ldots,\Pb_{\theta_N}) | \boldsymbol{\theta}\in\boldsymbol{\Theta}_{\mathrm{u}}\}$ w.r.t. $\ell_\infty$ policy weighted norm $\norm{\cdot}_{\infty}$, where the $\ell_\infty$ policy weighted norm between two vector-valued functions $\mathbf{l}=\{l_1,\cdots,l_N\}$ and $\mathbf{g}=\{g_1,\cdots,g_N\}$ defined on $(\Oc \times \Ac)^H$ is equal to $\norm{\mathbf{g} - \mathbf{l}}_{\infty} = \max_{i \in [N]}\max_{\pi_i}\sum_{\tau_H}|l_i(\tau_H)-g_i(\tau_H)|\pi_i(\tau_H)$.

% \begin{definition}[Bracketing number of the multi-task function class] \label{def: joint complexity}

% Suppose $\mathcal{F}^{(N)} $ is a multi-task function class over a domain $\mathcal{X}$. Given $2N$ non-nagetive functions $\{ l_1,\ldots,l_N \}$ and $\{g_1,\ldots,g_N\}$ whose domain is $\mathcal{X}$, the bracket $\mathbf{B}_{l_{1:N}}^{g_{1:N}}$formed by them is a set of all tuples $\{f_1,\ldots,f_N\} \in \Fc^{(N)}$ satisfying that $l_n(x) \leq f_n(x) \leq g_n(x)$ holds for all $n \in [N]$ and $x \in \mathcal{X}$. It is called $\eta$-bracket if $\|l_n-g_n\|\leq \epsilon \int_{x \in \mathcal{X}}|l_n(x)-g_n(x)| \leq \epsilon $,  $\sum_{x\in\mathcal{X}} |l_n(x) - g_n(x)|\leq \eta$ holds for all $n \in [N]$. The bracketing number $\mathcal{N}_{\eta}(\mathcal{F}^{(N)})$ is the minimum number of $\eta$-brackets needed to cover $\mathcal{F}^{(N)}$.

% \textcolor{red}{``cover'' should probably be defined as well. }

% \end{definition}
\textcolor{red}{add definition of $\mathcal{F}$ and $f$, meaning of bracket number, link with sharing parameter}

%With a little abuse of notation, we use $\mathcal{N}_{\eta}( \Theta)$ to denote the bracketing number of the distribution space $\{\Pb_{\theta}^{\pi}: \forall \theta\in\Theta, \forall \pi\}.$ 
% Since the multi-task measure space $\mathcal{F}$ can be parameterized by $\boldsymbol\Theta$, we abuse the notation $\mathcal{N}_{\eta}(\boldsymbol\Theta)$ to denote the $\eta$-bracketing number of the corresponding distribution space.  

\textcolor{red}{definition, add other cases. How would bracketing number change v.s. how shared information formulate}
\fi

\section{Upstream learning over multi-task PSRs}

We present our upstream algorithm in \Cref{subsec: alg-up}, characterize its theoretical performance in \Cref{subsec: result-up} and present examples to validate the benefit of upstream multi-task learning in \Cref{subsec: example-up}.

We use bold symbol to represent the multi-task parameters or policy. Specifically, $\boldsymbol{\theta} = (\theta_1,\ldots,\theta_N)$, and $\boldsymbol{\pi} = (\pi_1,\ldots,\pi_N)$. We define $Q_A = \max_n\max_h|\mathcal{Q}_h^{n,A}|$, where $Q_h^{n,A}$ is the core action sequence set of task $n$ at step $h$. The policy, denoted by $\nu_h(\pi,\pi')$, takes $\pi$ at the initial $h-1$ steps and switches to $\pi'$ from the $h$-th step. 
%For any model $\theta$ and its feature $\psi(\tau_h)$, the ``conditioned'' feature is expressed as $\bar{\psi}(\tau_h) = \psi(\tau_h)/\phi_h^{\top}\psi(\tau_h)$, because given any policy $\pi$, we have $\phi_h^{\top}\psi(\tau_h)\pi(\tau_h) = \Pb_{\theta}^{\pi}(\tau_h)$. 
Lastly, $\mathtt{u}_{\mathcal{X}}$ represents the uniform distribution over the set $\mathcal{X}$. 

\subsection{Algorithm: Upstream Multi-Task PSRs (UMT-PSR)}\label{subsec: alg-up}

We provide the pseudo-code of our upstream multi-task algorithm called \textit{Upstream Multi-Task PSRs} (UMT-PSR) in \Cref{alg:upstream}. This iterative algorithm   consists of three main steps as follows.

\begin{algorithm}[htbp]
\caption{\textbf{U}pstream \textbf{M}ulti-\textbf{T}ask \textbf{PSR}s (UMT-PSR)}\label{alg:upstream}
\begin{algorithmic}[1]
\State {\bf Input:} $\boldsymbol{\mathcal{B}}_1= \ThetaB$ model class, estimation   margin $\beta^{(N)}$, maximum iteration number $K$.

\For{$k=1,\ldots, K$}

\State Set $
    \textstyle \boldsymbol{\pi}^k = \arg\max_{ \boldsymbol{\pi}\in\Pi^N} \max_{ \boldsymbol\theta,  \boldsymbol\theta' \in\boldsymbol{\mathcal{B}}_k  } \sum_{n \in [N]} \mathtt{D}_{\TV}\left(\Pb_{ \theta_n }^{\pi^n}, \Pb_{ \theta'_{n}}^{\pi^n}\right)
$

\For{$n, h \in [N] \times [H]$}

\State Use $\nu_h^{\pi^{n,k}}$ to collect data $\tau_H^{n,k,h} $.
%\State Use $\nu(\pi^{n,k},\mathcal{U}_{\mathcal{Q}_h^A})$ to collect data $\tau_H^{n,k}$

\EndFor

\State Construct $\boldsymbol{\mathcal{B}}_{k+1} =$ 
\begin{align*}
    \bigg\{ \boldsymbol{\theta} \in \boldsymbol{\Theta}_{\mathrm{u}}\!:\!  \sum_{  t \in [k],h \in [H] \atop n \in [N]} \log \Pb_{\theta_n}^{\nu^{\pi^{n,t}}_{h}} (\tau_H^{n,t,h}) \geq \max_{ \boldsymbol{\theta}' \in\boldsymbol{\Theta}_{\mathrm{u}}}  \! \sum_{ t \in [k],h \in [H] \atop n \in [N]} \!\log \Pb_{\theta_n'}^{\nu^{\pi^{n,t}}_{h}} (\tau_H^{n,t,h})  \!-\! \beta^{(N)} \bigg\}\cap \boldsymbol{\mathcal{B}}_k.
    % \label{eqn:confidence set}
\end{align*}

\EndFor
\State {\bf Output:} Any $\overline{\boldsymbol{\theta}} \in \boldsymbol{\mathcal{B}}_{K+1}$, and a greedy multi-task policy $\overline{\boldsymbol{\pi} }=\arg\max_{\boldsymbol{\pi}}\sum_{n \in [N]} V_{\bar{\theta}_n, R_n}^{\pi^n} $
\end{algorithmic}
\end{algorithm}

%{\bf Distance design.} One key challenge is how to measure the distance between two multi-task models. A natural choice is the distance between product distributions $\Pb_{\theta_1}^{\pi_1}\times\cdots\times \Pb_{\theta_N}^{\pi_N}$ and $\Pb_{\theta_1'}^{\pi_1}\times\cdots\times \Pb_{\theta_N'}^{\pi_N}$. Unfortunately, this ``product distance'' is not sufficient to capture the accuracy of each individual model. The intuition is that, since the learning objective is the summation of value functions of all tasks, we need a distance that can upper bound each distance between models within a single task. Therefore, we propose to use the ``additive distance'', defined as $\mathtt{D}_{\boldsymbol{\pi}} ( \boldsymbol{\theta},  \boldsymbol{\theta}' ) \triangleq \sum_{n \in [N]}\mathtt{D}_{\TV}( \Pb_{ \theta_n}^{\pi^n} , \Pb_{ \theta_n'}^{\pi^n} ). $

{\bf Pairwise additive distance based multi-task planning} (Line 3): To promote joint planning among tasks, a natural choice to measure the distance between two multi-task models is the distance between the two product distributions $\Pb_{\theta_1}^{\pi_1}\times\cdots\times \Pb_{\theta_N}^{\pi_N}$ and $\Pb_{\theta_1'}^{\pi_1}\times\cdots\times \Pb_{\theta_N'}^{\pi_N}$. 
% \jing{is $\Pb_{\theta_1}^{\pi_1}$ explicitly defined?}
However, such a  ``distance between product distributions'' is not sufficient to guarantee the accuracy of the individual models of each task, which is needed in the analysis of the sum of the individual value functions. Hence, we propose to use the ``pairwise additive distance'' for our planning, defined as $\mathtt{D}_{\boldsymbol{\pi}} ( \boldsymbol{\theta},  \boldsymbol{\theta}' ) \triangleq \sum_{n \in [N]}\mathtt{D}_{\TV}( \Pb_{ \theta_n}^{\pi^n} , \Pb_{ \theta_n'}^{\pi^n} ). $
%
%At each iteration $k$, UMT-PSR seeks $N$ planning policies for each of the $N$ tasks. Specifically, the algorithm computes   the discrepancy between any two multi-task parameters $ \boldsymbol{\theta}$ and $ \boldsymbol{\theta}'$ within a confidence set $\mathcal{B}_k$ (which will be specified later), where the {\em discrepancy} is defined as $\mathtt{D}_{\boldsymbol{\pi}} ( \boldsymbol{\theta},  \boldsymbol{\theta}' ) \triangleq \sum_{n \in [N]}\mathtt{D}_{\TV}( \Pb_{ \theta_n}^{\pi^n} , \Pb_{ \theta_n'}^{\pi^n} ). $
%

More specifically, at each iteration $k$, UMT-PSR selects a multi-task policy $\boldsymbol{\pi}^k= (\pi^{1,k},\ldots,\pi^{N,k})$ that maximizes the largest pairwise additive distance $\max_{\boldsymbol{\theta} , \boldsymbol{\theta}'   } \mathtt{D}_{\boldsymbol{\pi}}(\boldsymbol{\theta} , \boldsymbol{\theta}' )$ within the confidence set $\boldsymbol{\mathcal{B}}_k$ (which will be specified later). An important property of $\boldsymbol{\mathcal{B}}_k$ is that it contains the true model $\boldsymbol{\theta}^*$ with high probability. Using this property, the largest pairwise additive distance serves as an {\em optimistic} value of the uncertainty $\mathtt{D}_{\boldsymbol\pi}\left( \boldsymbol{\theta}^*, \boldsymbol{\theta}\right)$ for any multi-task model $\boldsymbol{\theta}\in\boldsymbol{\mathcal{B}}_k$.

%An important property of the confidence set $\mathcal{B}_k$ is that it contains the true model $\boldsymbol{\theta}^*$ with high probability. Using this property, UMT-PSR selects a multi-task policy $\boldsymbol{\pi}^k= (\pi^{1,k},\ldots,\pi^{N,k})$ that maximizes the highest discrepancy $\max_{\boldsymbol{\theta} , \boldsymbol{\theta}'   } \mathtt{D}_{\boldsymbol{\pi}}(\boldsymbol{\theta} , \boldsymbol{\theta}' )$ within  $\mathcal{B}_k$. Note that the highest discrepancy serves an {\em optimistic} value of the uncertainty $\mathtt{D}_{\boldsymbol\pi}\left( \boldsymbol{\theta}^*, \boldsymbol{\theta}\right)$ for any multi-task model $\boldsymbol{\theta}\in\mathcal{B}_k$.

%, which is smaller than $\mathtt{D}_{\boldsymbol{\pi}}(\boldsymbol{\theta},\boldsymbol{\theta'})$.

{\bf Multi-task exploration} (Line 5): Building upon the planning policy $\boldsymbol{\pi}^k$, for each task $n$ and each step $h$, UMT-PSR executes the policy $\pi^{n,k}$ for first $h-1$ steps, and then uniformly selects an action sequence in $\mathcal{A}\times \mathcal{Q}_h^{n,A}$ for the following $H-h+1$ steps. In particular, at step $h$, UMT-PSR uniformly takes an action in $\mathcal{A}$, and then uniformly chooses a core action sequence $\mathbf{a}_h$ such that regardless of what the observation sequence is, UMT-PSR always plays the action in the sampled core action sequence. In summary, for each $(n,h)\in[N]\times[H]$, UMT-PSR adopts the policy $\nu_h(\pi^{n,k},\mathtt{u}_{\mathcal{A}\times \mathcal{Q}_h^{n,A}})$ to collect a sample trajectory $\tau_H^{n,k,h}$. We abbreviate $\nu_h(\pi^{n,k},\mathtt{u}_{\mathcal{A}\times \mathcal{Q}_h^{n,A}})$ as $\nu^{\pi^{n,k}}_{h}$.

{\bf Confidence set construction via bracketing number of joint model class} (Line 7): Given the sampled trajectories, UMT-PSR calls a maximum likelihood estimation oracle to construct the multi-task confidence set. A novel element here is the use of the bracketing number of the {\em joint model class} to characterize estimation margin $\beta^{(N)}$, which is an upper bound of the gap between the maximum log-likelihood within $\boldsymbol{\Theta}_u$ and the log-likelihood of the true model.  Such a design provides a unified way for any MTRL problem and avoids individual design for each problem in a case-by-case manner. 
% \ylfixed{define $\beta^{(N)}$}

\if{0}
\begin{align}
    \boldsymbol{\mathcal{B}}_{k+1} = \bigg\{ \boldsymbol{\theta} \in \boldsymbol{\Theta}_{\mathrm{u}}:  \sum_{  t\leq k,h\leq H \atop n\leq N}   \log \Pb_{\theta_n}^{\nu^{\pi^{n,t}}_{h}} (\tau_H^{n,t,h}) \geq \max_{ \boldsymbol{\theta}' }   \sum_{ t\leq k,h\leq H \atop n\leq N} \log \Pb_{\theta_n'}^{\nu^{\pi^{n,t}}_{h}} (\tau_H^{n,t,h})  - \beta^{(N)} \bigg\}\cap\boldsymbol{\mathcal{B}}_k,\label{eqn:confidence set}
\end{align}
\fi
%where the estimation margin $\beta^{(N)}$ is a pre-defined constant.

\subsection{Main Theoretical Result}\label{subsec: result-up}

%Before proceeding to our main result of multi-task PSRs learning, we assume that PSRs studied in this paper is $\gamma$-well-conditioned specified as follows. Such an assumption and its variants are commonly adopted in the study of PSRs \citep{liu2022optimistic,chen2022partially,zhong2022posterior}.

%In this section, we provide the theoretical guarantees of UMT-PSR. First, we introduce the notion of bracketing number. 

The following theorem characterizes the guarantee of the model estimation and the sample complexity to find a near-optimal multi-task policy.
\begin{theorem}\label{thm:upstream}
     Under \Cref{assmp:well-condition} 
     % and \Cref{assmp: upstream constraint}
     , for any fixed $\delta > 0$, let $\boldsymbol{\Theta}_{\mathrm{u}}$ be the multi-task parameter space,  $\beta^{(N)} = c_1(\log\frac{KHN}{\delta} + \log\mathcal{N}_{\eta}(\boldsymbol{\Theta}_{\mathrm{u}}))$, where $c_1 > 0$ and $\eta\leq \frac{1}{KHN}$. Then with probability at least $1-\delta$, UMT-PSR finds a multi-task model $\overline{\boldsymbol{\theta}} = (\bar{\theta}_1,\ldots,\bar{\theta}_N)$ such that
    \begin{align*}
       \textstyle \sum_{n=1}^{N}  \max_{\pi^n\in \Pi}\mathtt{D}_{\TV}\left(\Pb_{ \bar{\theta}_n}^{\pi^n}, \Pb_{\theta_n^*}^{\pi^n}\right) \leq \tilde{O}\left(  \frac{Q_A}{\gamma}\sqrt{ \frac{rH|\mathcal{A}|N \beta^{(N)} }{ K} }  \right).
    \end{align*}
    In addition, if $K=\frac{c_2r|\mathcal{A}|Q_A^2 H \beta^{(N)} }{ N\gamma^2\epsilon^2}$ for large enough $c_2 > 0$, UMT-PSR produces a multi-task policy $\overline{\boldsymbol{\pi}} = (\bar{\pi}^1,\ldots,\bar{\pi}^N)$ such that the average sub-optimality gap is at most $\epsilon$, i.e.
    \begin{align}
        \textstyle \frac{1}{N} \sum_{n=1}^N \left( \max_{\pi\in \Pi} V_{\theta_n^*, R_n}^{\pi} - V_{\theta^*, R_n}^{\bar{\pi}^n} \right) \leq \epsilon.
    \end{align}
\end{theorem}

{\bf Benefits of multi-task learning:} \Cref{thm:upstream} shows that with the sample complexity  $\tilde{O}(\frac{r|\mathcal{A}|Q_A^2 H^2 \beta^{(N)} }{ N\gamma^2\epsilon^2})$, UMT-PSR identifies an $\epsilon$-optimal  multi-task policy. As a comparison, the best known sample complexity of a single-task PSR RL is given by $O(\frac{r|\mathcal{A}|Q_A^2 H^2 \beta }{ \gamma^2\epsilon^2}  )$ in \citet{chen2022partially}, where $\beta^{(1)}=\tilde{O}(r^2|\Oc||\Ac|H^2)$ scales the logarithm of the bracketing number of a single-task PSR with rank $r$. It is clear that as long as $\beta^{(N)} < N\beta^{(1)}$, then UMT-PSR enjoys multi-task benefit in the sample complexity. In \Cref{subsec: example-up}, we will provide several example multi-task POMDPs/PSRs to illustrate that such a condition can be satisfied broadly.

Next, we make a few comparisons concerning $\beta^{(N)}$. {\bf (i)} If $N=1$, \Cref{thm:upstream} matches the best known sample complexity given in \citet{chen2022partially}. {\bf (ii)} If none of tasks share any similarity, i.e., $\boldsymbol{\Theta}_{\mathrm{u}} = \Theta^N$, we have $\beta^{(N)} = N\beta^{(1)}$, and the sample complexity does not exhibit any benefit compared to learning the tasks separately. This coincides with the intuition that in the worst case, multi-task learning is not required. {\bf (iii)} The benefits of multi-task learning  are more evident when ${\beta}^{(N)}/N$ decreases. An extreme example is that when all tasks also share the same dynamics, leading to ${\beta}^{(N)} = \beta^{(1)}$. In this case, multi-task learning reduces to the batch setting and as the batch size increases, the iteration number   decreases linearly in $N$.

%{\bf (i)} When $N=1$, UMT-PSR reduces to the single-task setting, and the sample complexity $O\left(\frac{r|\mathcal{A}|Q_A^2 H^2 \beta }{ \gamma^2\epsilon^2}  \right)$ matches the best known result in \citet{chen2022partially}, where $\beta=\tilde{O}(r^2OAH^2)$ scales as the logarithm of the bracketing number of a single-task PSR with rank $r$. {\bf (ii)} If none of tasks share any similarity, i.e., $\boldsymbol{\Theta}_{\mathrm{u}} = \Theta^N$, we have $\beta^{(N)} = N\beta$, and the sample complexity does not exhibit any benefit compared to learning the tasks separately. This coincides with the intuition that in the worst case, multi-task learning is not required. {\bf (iii)} The benefits of multi-task learning  are more evident when ${\beta}^{(N)}/N$ decreases. An extreme example is that when all tasks also share the same dynamics, leading to ${\beta}^{(N)} = \beta$. In this case, multi-task learning reduces to the batch setting and as the batch size increases, the iteration number   decreases linearly in $N$. 

\subsection{Important Examples of Multi-task PSRs}\label{subsec: example-up}
As shown in \Cref{subsec: up} and \Cref{thm:upstream}, for multi-task models with low $\eta$-bracketing number, i.e., satisfying $\beta^{(N)} < N\beta^{(1)}$, UMT-PSR exhibits better sample complexity than single-task learning. In this section, we provide example multi-task POMDPs and PSRs and show that their $\eta$-bracketing number satisfies the condition.
%satisfy formally to demonstrate the efficiency of UMT-PSR compared to single-task learning. 
Detailed proofs for these examples can be found in \Cref{Appd_sub: F-3}.  

% We first show examples of multi-task POMDPs in \Cref{eg:same transtion,eg:emission}, where different tasks explicitly share some common features. In \Cref{eg:M+delta,eg: linear com}, we show the multi-task examples where implicit constrains exist in the parameters space. In these scenarios, explicitly or implicitly shared commonalities benefit on sample efficiency.

%We defer the detailed discussion of medel complexity of these examples in \Cref{Appd: F bracketing number}.

\if{0}
{\bf POMDPs and multi-task learning.} We first consider tabular POMDPs, which is a classic subclass of PSRs. Specifically, the dynamics in POMDPs consists of $H$ transition distributions $\{\mathbb{T}_h: \mathcal{S}\times\mathcal{A}\times\mathcal{S}\rightarrow[0,1] \}_{h=1}^H$, and $H$ emission distributions $\{\mathbb{O}_h: \mathcal{S}\times\mathcal{O}\rightarrow[0,1] \}_{h=1}^H$, where $\mathcal{S}$ is a finite state space. The states capture the entire system information, but are unobservable. In POMDPs, at each step $h$, if the current system state is $s_h$, the agent observes $o_h$ with probability $\mathbb{O}_h(o_h|s_h)$. Then, if the agent takes an action $a_h$ based on previous observations $o_{h},\ldots,o_1$ and actions $a_{h-1},\ldots,a_1$, the system state transits to $s_{h+1}$ with probability $\mathbb{T}_h(s_{h+1}|s_h,a_h)$. We use the notation $\mathtt{P}_{\mathrm{po}} = (\mathcal{O},\mathcal{A}, H, \mathcal{S}, \mathbb{T},\mathbb{O}, R)$ to represent a POMDP instance. Note that the tuple $(\mathcal{S}, \mathbb{T},\mathbb{O})$ in POMDPs determine the general dynamics $\Pb$ in PSRs. Suppose all tasks share the same state, observation, and action spaces, then, $\mathtt{P}_{\mathrm{po}}^n = (\mathcal{O},\mathcal{A}, H, \mathcal{S} , \mathbb{T}^n,\mathbb{O}^n, R)$ represents the model of task $n$.
\fi

% \begin{example}[Same transition]\label{eg:same transtion A.1} We first examine the case where all task share the same transition function, i.e. there exists a set of unique transition distributions $\{\mathbb{T}_h^*\}_{h=1}^H$ such that $\mathbb{T}_h^n = \mathbb{T}_h^*$ for all $n\in[N]$ and $h\in[H]$. In other words, we consider a scenario of multi-task POMDPs where the underlying state transitions are the same but the emission distributions can be different.  Then, the log bracketing number is at most $\log \Nc_{\eta}(\ThetaB) = O ( H(|\mathcal{S}|^2|\mathcal{A}| + |\mathcal{S}||\mathcal{O}|N)\log\frac{H|\mathcal{O}||\mathcal{A}||\mathcal{S}|}{\eta})$, and the multi-task \textcolor{blue}{estimation margin} $\beta^{(N)}$ satisfies that $0 < \beta^{(N)} <\tilde{O}(H(|\mathcal{S}|^2|\mathcal{A}| + |\mathcal{S}||\mathcal{O}|N))$. Compared to not sharing any commonalities with the estimation margin satisfying $0<\beta^N=N\beta \leq \tilde{O}()$.  
% \end{example}

{\bf Muli-task POMDPs.} We consider tabular POMDPs, which is a classic subclass of PSRs. Specifically, the dynamics in POMDPs consist of $H$ transition distributions $\{\mathbb{T}_h: \mathcal{S}\times\mathcal{A}\times\mathcal{S}\rightarrow[0,1] \}_{h=1}^H$, and $H$ emission distributions $\{\mathbb{O}_h: \mathcal{S}\times\mathcal{O}\rightarrow[0,1] \}_{h=1}^H$, where $\mathcal{S}$ is a finite state space. The states capture the entire system information, but are not directly observable. In POMDPs, at each step $h$, if the current system state is $s_h$, the agent observes $o_h$ with probability $\mathbb{O}_h(o_h|s_h)$. Then, if the agent takes an action $a_h$ based on previous observations $o_{h},\ldots,o_1$ and actions $a_{h-1},\ldots,a_1$, the system state transits to $s_{h+1}$ with probability $\mathbb{T}_h(s_{h+1}|s_h,a_h)$. We use the notation $\mathtt{P}_{\mathrm{po}} = (\mathcal{O},\mathcal{A}, H, \mathcal{S}, \mathbb{T},\mathbb{O}, R)$ to represent a POMDP instance. Note that the tuple $(\mathcal{S}, \mathbb{T},\mathbb{O})$ in POMDPs determine the general dynamics $\Pb$ in PSRs. If all tasks share the same state, observation, and action spaces, then $\mathtt{P}_{\mathrm{po}}^n = (\mathcal{O},\mathcal{A}, H, \mathcal{S} , \mathbb{T}^n,\mathbb{O}^n, R)$ represents the model of task $n$.

\begin{example}[Multi-task POMDP with common transition kernels]\label{eg:same transtion} All tasks (i.e., all POMDPs) share the same transition kernel, i.e., there exists a set of transition distributions $\{\mathbb{T}_h^*\}_{h=1}^H$ such that $\mathbb{T}_h^n = \mathbb{T}_h^*$ for all $n\in[N]$ and $h\in[H]$. 
%In other words, we consider a scenario of multi-task POMDPs where the underlying state transitions are the same but 
The emission distributions can be different. Such a scenario arises if the agent observes the same environment from different angles and hence receives different observations. Then, 
%the log $\eta$-bracketing number $\log \Nc_{\eta}(\ThetaB)$ of the joint model class is at most   
$\beta^{(N)}$ is at most $O ( H(|\mathcal{S}|^2|\mathcal{A}| + |\mathcal{S}||\mathcal{O}|N)\log\frac{H|\mathcal{O}||\mathcal{A}||\mathcal{S}|}{\eta})$, whereas the single task 
%compared to that of the product space: 
$\beta^{(1)}$ is given by $O ( H(|\mathcal{S}|^2|\mathcal{A}| + |\mathcal{S}||\mathcal{O}|)\log\frac{H|\mathcal{O}||\mathcal{A}||\mathcal{S}|}{\eta})$. Clearly, $\beta^{(N)} < N\beta^{(1)}$. 
\end{example}

\if{0}
\begin{example}[Multi-task POMDP with common emission]\label{eg:emission} We consider the POMDPs where all tasks share the same emission distributions, i.e., there exists a set of unique emission distributions $\{\mathbb{O}_h^*\}_{h=1}^H$ such that $\mathbb{O}_h^n = \mathbb{O}_h^*$ for all $n\in[N]$ and $h\in[H]$.  Then, 
%the log bracketing number 
$\beta^{(N)}$ is at most $ O ( H(|\mathcal{S}|^2|\mathcal{A}|N + |\mathcal{S}||\mathcal{O}|)\log\frac{H|\mathcal{O}||\mathcal{A}||\mathcal{S}|}{\eta})$. Clearly, $\beta^{(N)} < N\beta^{(1)}$ is satisfied. 
\end{example}

\yl{provide a scenario to justify the above multi-task POMDP}
\fi

\textbf{Multi-task PSRs:} We next provide two example multi-task PSRs, in which tasks do not share common model parameters. In these examples, the similarities among tasks could alternatively be established via implicit connections and correlations in latent spaces, which reduce the complexity of the joint model class, hence the estimation margin and the sample complexity of algorithms significantly compared with separately learning each single task. 

% the \textcolor{blue}{boundary case}. Note that in boundary cases, all tasks share no similarities and estimation margin $\beta^{(N)}$ satisfying that $0 < \beta^{(N)} \leq O( \log\frac{KHN}{\delta}+
%    N r^2 |\mathcal{O}| |\mathcal{A}| H^2\log\frac{ |\mathcal{O}| |\mathcal{A}| }{\eta})$.  
% \yl{I don't see the point to introduce the notion of boundary case.}
   
% are present as following where exists no explicit common parameters. In these examples, the commonalities among multiple tasks could alternatively be established via implicit connections and correlations in latent spaces, which reduce the complexity of the parameter space, hence the estimation margin and the sample complexity of algorithms significantly compared with \textcolor{blue}{boundary case}. 

\begin{example}[Multi-task PSR with perturbed models]\label{eg:M+delta}
     Suppose there exist a latent base task $\mathtt{P}_{\mathrm{b}}$, and a finite noisy perturbation space $\boldsymbol{\Delta}$. Each task $n \in [N]$ is a noisy perturbation of the latent base task and can be parameterized into two parts: the base task plus a task-specified noise term. Specifically, for each step $h \in [H]$ and task $n \in [N]$, any $(o,a) \in \Oc \times \Ac$, we have 
     \begin{align*}
         \textstyle \Mbf_h^n(o_h,a_h)=\Mbf_h^\mathrm{b}(o_h,a_h)+\Delta^n_h(o_h,a_h), \quad \Delta^n_h \in \boldsymbol{\Delta}.
     \end{align*}    
Such a multi-task PSR satisfies that $\beta^{(N)} = \tilde{O}( r^2 |\mathcal{O}| |\mathcal{A}| H^2  + HN\log|\boldsymbol{\Delta}|)$, whereas $\beta^{(1)}$ for a single task is given by $ \tilde{O}( r^2 |\mathcal{O}| |\mathcal{A}| H^2  )$.
Clearly, $\beta^{(N)} \ll N\beta^{(1)}$  holds if $\log |\boldsymbol{\Delta}| \ll \tilde{O}(r^2|\mathcal{O}||\Ac|H)$, which can be easily satisfied for small-size perturbation environments. Hence, the multi-task PSR  benefits from a significantly reduced sample complexity compared to single-task learning.

     %The log bracketing number $ \log \Nc_{\eta}(\ThetaB)=O( r^2 |\mathcal{O}| |\mathcal{A}| H^2\log\frac{ |\mathcal{O}| |\mathcal{A}| }{\eta}   + H(N-1)\log|\boldsymbol{\Delta}| )$, and the multi-task estimation margin satisfies that $0 < \beta^{(N)} \leq O( \log\frac{KHN}{\delta}+r^2 |\mathcal{O}| |\mathcal{A}| H^2\log\frac{ |\mathcal{O}| |\mathcal{A}| }{\eta}   + HN\log|\boldsymbol{\Delta}|)$. Compared with the boundary case, the multi-task case has significantly lower sample complexity if $\log |\boldsymbol{\Delta}| \ll \tilde{O}(r^2|\mathcal{O}||\Ac|H)$.  
% Motivative application: autonomous driving car may  
\end{example}   

\begin{example}[Multi-task PSR: Linear combination of core tasks]\label{eg: linear com}
Suppose that the multi-task PSR lies in the linear span of $m$ core tasks $\{\mathtt{P}_1,\ldots,\mathtt{P}_m\}$.
% , i.e., there exist a set of core tasks indexed by $\{1,2,\ldots,m\}$ such that each PSR can be represented as a linear combination of those $m$ core tasks. 
Specifically, for each task $n \in [N]$, there exists a coefficient  vector $\boldsymbol{\alpha}^n=(\alpha_1^n,\cdots,\alpha_m^n)^\top \in \Rb^m$ s.t. for any $h \in [H]$ and $(o_h,a_h) \in \Oc \times \Ac$,
    \begin{align*}
        \textstyle \phi_h^n(o_h,a_h)=\sum_{l=1}^m\alpha_l^n \phi_h^l(o_h,a_h), \quad \Mbf_h^n(o_h,a_h)=\sum_{l=1}^m\alpha_l^n \Mbf_h^l(o_h,a_h).
    \end{align*}
    For regularization, we assume $0 \leq \alpha^n_l$ for all $l \in [m]$ and $n\in [N]$, and $\sum_{l=1}^m\alpha^n_l=1$ for all $n \in [N]$. It can be shown that 
    %the log bracketing number is 
    $\beta^{(N)}=O(   m( r^2  |\mathcal{O}| |\mathcal{A}| H^2+ N) )$, whereas $\beta^{(1)}= \tilde{O}(r^2|\mathcal{O}||\mathcal{A}|H^2)$.
Clearly, $\beta^{(N)} \ll N\beta^{(1)}$ holds if $m \leq \min\{N,r^2  |\mathcal{O}| |\mathcal{A}| H^2\}$, which is satisfied  in practice.    
\end{example}

\section{Downstream learning for PSRs}

In downstream learning, the agent is assigned   a new task $\mathtt{P}_0 = (\mathcal{O},\mathcal{A},H,\Pb_{\theta_0^*}, R_0)$, where $\theta_0^*\in\Theta_0^{\mathrm{u}}$, and $\Theta_0^{\mathrm{u}}$ is defined in \Cref{subsec: down}. As explained in \Cref{subsec: down}, upstream and downstream tasks are connected via the {\bf similarity constraint}  $\mathtt{C}(\theta_0,\theta_1^*,\ldots\theta_N^*)\leq \boldsymbol{0}$.
%While the agent does not have access to $\Theta_0^{\mathrm{u}}$ due to unknown upstream models $\theta_1^*,\ldots,\theta_N^*$, upstream learning has provided good estimators $\bar{\theta_1},\ldots,\bar{\theta}_N$. 
Therefore, the agent can use the estimated model parameter $\bar{\theta}_1,\ldots,\bar{\theta}_N$ in the upstream to construct an empirical candidate model class for the downstream task as $\widehat{\Theta}_0^{\mathrm{u}} = \{\theta_0\in\Theta | \mathtt{C}(\theta_0,\bar{\theta}_1,\ldots,\bar{\theta}_N)\leq 0 \}$. Then for downstream learning, we adopt the standard  OMLE~\citep{liu2022optimistic,chen2022partially} for the model class $\hat{\Theta}_0^{\mathrm{u}}$. 

The sample complexity of downstream learning will be determined by the bracketing number of $\hat{\Theta}_0^{\mathrm{u}}$, which is nearly the same as that of the ground truth $\Theta_0^{\mathrm{u}}$. Since the similarity constraint will significantly reduces the complexity of the model parameter space, the bracketing number of $\hat{\Theta}_0^{\mathrm{u}}$ should be much smaller than that of the original parameter space $\Theta$. In this way, the downstream can benefit from the upstream learning with reduced sample complexity. In the following subsections, we first characterize the performance guarantee for downstream learning in terms of the bracketing number of $\hat{\Theta}_0^{\mathrm{u}}$, and then show that the similarity constraint  reduces the bracketing number for the examples given in \Cref{subsec: example-up}.
%\mathcal{N}_{\eta}(\Theta_0^{\mathrm{u}})$. 

\subsection{Theoretical Guarantee for Downstream Learning}
%\subsection{Realizable Case}

%The performance of downstream learning highly depends on the approximation error of the model class $\hat{\Theta}_0^{\mathrm{u}}$.
One main challenge in the downstream learning is that the true model may not lie in $\hat{\Theta}_0^{\mathrm{u}}$. To handle this, we employ R\'enyi divergence to measure
the ``distance'' from the model class to the true model as follows, mainly because its unique advantage under the MLE oracle: the R\'enyi divergence of order $\alpha$ with $\alpha\ge1$ serves as an upper bound on the TV distance and the KL divergence, and thus has more robust performance.
%which is defined as follows. 
\begin{definition}\label{assmp:renyi}
Fix $\alpha>1$. The approximation error of $\hat{\Theta}_0^{\mathrm{u}}$ under $\alpha$-R\'enyi divergence is defined as $\mathrm{e}_{\alpha}(\hat{\Theta}_0^{\mathrm{u}}) = \min_{\theta_0\in\hat{\Theta}_0^{\mathrm{u}}}\max_{\pi \in \Pi}\mathtt{D}_{\mathtt{R},\alpha}( \Pb_{\theta_0^*}^{\pi}, \Pb_{\theta_0}^{\pi}  )$.

%There exist $\alpha > 1$ and $\theta_0^{\mathrm{c}} \in \hat{\Theta}_0^{\mathrm{u}}$ such that $\max_{\pi \in \Pi}\mathtt{D}_{\mathtt{R},\alpha}( \Pb_{\theta_0^*}^{\pi}(\tau_H), \Pb_{\theta_0^{\mathrm{c}}}^{\pi} (\tau_H) ) \leq \epsilon_0$ for some fixed $\epsilon_0>0$.
\end{definition}

%We remark that, regarding measuring the ``distance'' from the model class to the true model, R\'enyi divergence presents unique advantage over total variation distance and KL divergence under the MLE oracle. This is because when $\alpha\ge1$, the R\'enyi divergence of order $\alpha$ is an upper bound of TV distance and KL divergence, and thus is robust when true model does not lie in $\hat{\Theta}_0^{\mathrm{u}}$.

\if{0}
The performance of downstream learning highly depends on whether $\hat{\Theta}_0^{\mathrm{u}}$ 
satisfies the realizability condition, i.e., $\theta_0^*\in \hat{\Theta}_0^{\mathrm{u}}$. We first provide the following theoretical guarantee for the simpler realizable setting, and then study the more challenging non-realizable setting.
\begin{theorem}[Realizable setting]\label{thm:downstraem OMLE}
    Let $\beta_0 =c( \log\mathcal{N}_{\eta}(\hat{\Theta}_0^{\mathrm{u}}) )$, where $\eta\leq\frac{1}{KH}$ and $c$ is some absolute constant. Under \Cref{assmp:well-condition}, with probability at least $1-\delta$, the output of \Cref{alg:DMT-PSR} satisfies that
    \begin{align}
         \textstyle \max_{\pi \in \Pi}\mathtt{D}_{\TV}\left(\Pb_{ \bar{\theta}_0}^{\pi}, \Pb_{\theta_0^*}^{\pi}\right) \leq \tilde{O}\left(  \frac{Q_A}{\gamma}\sqrt{ \frac{rH|\mathcal{A}|  \beta_0 }{K} }  \right).
    \end{align}
\end{theorem}
\fi

\begin{theorem}\label{thm:downstream l2}
    Fix $\alpha>1$. Let $\epsilon_0 = \mathrm{e}_{\alpha}(\hat{\Theta}_0^{\mathrm{u}})$, $\beta_0 = c_0( \log\mathcal{N}_{\eta}(\hat{\Theta}_0^{\mathrm{u}}) + \epsilon_0 KH + (\frac{\mathbf{1}_{\{\epsilon_0\neq0\}}}{\alpha -1} +1)\log\frac{KH}{\delta})$ for some large $c_0$, where $\eta\leq\frac{1}{KH}$. Under \Cref{assmp:well-condition}, with probability at least $1-\delta$, the output of \Cref{alg:DMT-PSR} satisfies that
    \begin{align*}
         \textstyle \max_{\pi \in \Pi}\mathtt{D}_{\TV}\left(\Pb_{ \bar{\theta}_0}^{\pi}, \Pb_{\theta_0^*}^{\pi}\right) \leq \tilde{O}\left(  \frac{Q_A}{\gamma}\sqrt{ \frac{r|\mathcal{A}|H  \beta_0 }{K} }  +\sqrt{\epsilon_0} \right).
    \end{align*}
\end{theorem}

{\bf Benefits of downstream transfer learning:} \Cref{thm:downstream l2} shows that when $\epsilon_0<\epsilon^2/4$, with sample complexity at most $\tilde{O}(\frac{rQ_A^2|\mathcal{A}|H\beta_0}{\gamma^2\epsilon^2})$, OMLE identifies an $\epsilon$-optimal policy for the downstream task. As a comparison, the best known sample complexity for single-task PSR RL without transfer learning is $\tilde{O}(\frac{rQ_A^2|\mathcal{A}|H\beta}{\gamma^2\epsilon^2})$, where $\beta = \tilde{O}(\log\mathcal{N}_{\eta}(\Theta))$ \citep{chen2022partially}. It is clear that as long as $\beta_0<\beta$, then downstream learning enjoys transfer benefit in the sample complexity. 

Notably, in the realizable case when $\epsilon_0 = 0$, i.e. $\theta_0^*\in \hat{\Theta}_0^{\mathrm{u}}$, we must have $\beta_0 = \tilde{O}(\log\mathcal{N}_{\eta}(\hat{\Theta}_0^{\mathrm{u}}))\leq \beta$, since $\hat{\Theta}_0^{\mathrm{u}}\subset\Theta$. In the non-realizable case when $\epsilon_0>0$, compared to the realizable case, the estimation error of $\bar{\theta}_0$ has an additive factor of $\tilde{O}(\sqrt{\epsilon_0} + \sqrt{1/(K(\alpha-1))})$ after hiding system parameters. We remark that this factor shrinks if the approximation error of $\hat{\Theta}_0^{\mathrm{u}}$ decreases and the order of R\'enyi divergence grows, which coincide with the intuition.

%Compared to the case when $\hat{\Theta}_0^{\mathrm{u}}$ is realizable, the estimation error has an additional factor  of $\tilde{O}\big(\frac{Q_A\sqrt{r|\mathcal{A}|H}}{\gamma}\sqrt{\epsilon_0 H + \frac{1}{K(\alpha-1)}} \big)$. 
%\ylfixed{comment on how $\alpha$ affects the bound and the benefit} {\color{magenta}: We need to comment on the significance of this ``additive'' factor. }

%Compared to the best known sample complexity for learning PSRs \citep{chen2022partially}, \Cref{thm:downstraem OMLE} shows that downstream learning benefits from having fewer parameters that need to be estimated. This benefit is characterized by the bracketing number $\mathcal{N}_{\frac{1}{KH}}(\hat{\Theta}_0^{\mathrm{u}})$, which can be significantly small compared to the bracketing number $\mathcal{N}_{\frac{1}{KH}}( \Theta)$ of a single-task. \yl{write the paragraph in the way of the paragraph below eq (4).}

%\subsection{Non-realizable Case}

%We next consider the non-realizable case. To characterize the performance, we make the following assumption on the downstream model class $\hat{\Theta}_0^{\mathrm{u}}$. \yl{a couple of sentences to explain why such an assumption is needed, and justify why this can be satisfied}

%Based on the assumption that the R\'enyi divergence between the true model and the model class $\hat{\Theta}_0^{\mathrm{u}}$ is small, we have the following theoretical guarantee.

\subsection{Examples in  Downstream Learning Tasks}\label{sec:downstream example}
We revisit the examples presented in upstream multi-task learning, specifically \Cref{eg:same transtion,eg:M+delta,eg: linear com}, and subsequently extend their application in downstream tasks under the realizable setting. With the prior knowledge obtained from upstream learning, these examples exhibit reduced $\eta$-bracketing number, and hence benefit in the sample efficiency. Detailed proofs are in Appx.~\ref{Appd_sub: F-4}.  

%To characterize the advantage of downstream learning, we define the ratio $b_r$ as the ratio of the \textit{upper bound} of sample complexity of our algorithms and the \textit{upper bound} of SOTA sample complexity learning a single-task. Note that $b_r\in[0,1]$, and the smaller $b_r$, the more efficient the downstream learning is.

\setcounter{example}{0}

%For the example of POMDPs, without prior knowledge, we have $\beta^{\mathrm{po}} = \tilde{O}(H(|\mathcal{S}|^2|\mathcal{A}|+|\mathcal{S}||\mathcal{O}|))$.

\begin{example}[Multi-task POMDP 
 with Common transition kernels] Suppose $\hat{\mathbb{T}}$ is the output from UMT-PSR. In this case, the downstream $\hat{\Theta}_0^{\mathrm{u}}$ is constructed by combining $\hat{\mathbb{T}}$ and all possible emission distributions. Then $\beta_0 = \tilde{O}(H|\mathcal{S}||\mathcal{O}|)$. However, for POMDP without prior knowledge, $\beta = \tilde{O}(H(|\mathcal{S}|^2|\mathcal{A}|+|\mathcal{S}||\mathcal{O}|))$. Clearly, $\beta_0\leq \beta$, indicating the benefit of downstream learning. 
%which is smaller than $\beta^{\mathrm{po}}$. 

\end{example}

\if{0}

\begin{example}[Multi-task POMDP with common transition kernels] We then examine the case where each task shares the same emission distributions, i,e. there exists a set of unique emission distributions $\{\mathbb{O}_h^*\}_{h=1}^H$ such that $\mathbb{O}_h^n = \mathbb{O}_h^*$ for all $n\in[N]$ and $h\in[H]$.   Then, the log $\eta$-bracketing number is at most $ O ( H|\mathcal{S}|^2|\mathcal{A}| \log\frac{H |\mathcal{A}||\mathcal{S}|}{\eta})$.

%Hence, the benefit ratio is $O(\frac{|\mathcal{S}||\mathcal{A}| }{ |\mathcal{S}||\mathcal{A}| + |\mathcal{O}| })$ in the realizable setting, and $O(\frac{|\mathcal{S}||\mathcal{A}| + \epsilon_0K/|\mathcal{S}| + \frac{1}{(\alpha-1)H|\mathcal{S}|}  }{ |\mathcal{S}||\mathcal{A}| + |\mathcal{O}| })$ in the non-realizable setting.
\end{example}
\fi

For PSRs without prior knowledge, we have $\beta^{\mathrm{PSR}} = \tilde{O}(r^2|\mathcal{O}||\mathcal{A}|H^2 )$.

\begin{example}[Multi-task PSR with perturbed models]\label{eg:M+delta-ds}The downstream task $\mathtt{P}_0$ is a noisy perturbation of a base task $\mathtt{P}_{\mathrm{b}}$. Specifically, for each step $h \in [H]$, any $(o,a) \in \Oc \times \Ac$, we have 
     \begin{align}
         \textstyle \phi_H^0=\phi_H^{\mathrm{b}}, \Mbf_h^{0}(o_h,a_h)=\Mbf_h^\mathrm{b}(o_h,a_h)+\Delta^{0}_h(o_h,a_h), \quad \Delta^{0}_h \in \boldsymbol{\Delta}. \nonumber
     \end{align} 
Then, $\beta_0 = \tilde{O}(H\log|\boldsymbol{\Delta}|)$, which is much lower than $\beta^\mathrm{PSR}$ if $\log |\boldsymbol{\Delta}| \ll \tilde{O}(r^2|\mathcal{O}||\Ac|H)$.

%The benefit ratio is $O(\frac{\log|\Delta| }{ r^2|\mathcal{O}||\mathcal{A}|  })$ in the realizable setting, and $O(\frac{ \log|\Delta| + \epsilon_0K + \frac{1}{(\alpha-1)H}  }{ r^2|\mathcal{O}||\mathcal{A}|})$ in the non-realizable setting.
\end{example}

\begin{example}[Multi-task PSR: Linear combination of core tasks]\label{eg: linear com-ds}
The downstream task $\mathtt{P}_0$ lies in the linear span of $L$ upstream tasks (e.g. the firs $L$ source tasks). 
%, span of $m$ core tasks. Moreover, assume the upstream tasks are diverse enough to span the whole core tasks space. As a result, the downstream task can be represented as a linear combination of a subset of the upstream tasks. Specifically, there exists a constant $L$ satisfying $m \leq L \leq N$ and 
Specifically, there exists a coefficient vector $\boldsymbol{\alpha}^{0}=(\alpha_1^{0},\cdots,\alpha_L^{0})^\top \in \Rb^L$ s.t. for any $h \in [H]$ and $(o_h,a_h) \in \Oc \times \Ac$,
\begin{align}
        \textstyle \phi_H^{0}=\sum_{l=1}^L\alpha_l^{0} \phi_H^l, \quad \Mbf_h^{0}(o_h,a_h)=\sum_{l=1}^L\alpha_l^{0} \Mbf_h^l(o_h,a_h). \nonumber
\end{align}
    For regularization, we assume $0 \leq \alpha^{0}_l$ for all $l \in [L]$, and $\sum_{l=1}^L\alpha^{0}_l=1$. Then  $\beta_0 = \tilde{O}(LH)$, which is much smaller than $\beta^\mathrm{PSR}$ if $L \leq \min\{N,r^2  |\mathcal{O}| |\mathcal{A}| H^2\}$.
\end{example}

% \ylfixed{write this section following the same structure as Section 4.3}

\section{Conclusion}
In this paper, we study multi-task learning on general non-markovian low-rank decision making problems. Given that all tasks share the same observation and action spaces, using the approach of PSRs, we theoretically characterize that multi-task learning presents benefit over single-task learning if the joint model
class of PSRs has a smaller $\eta$-bracketing number. We also provide specific example multi-task PSRs with small $\eta$-bracketing numbers. Then, with prior knowledge from the upstream, we show that downstream learning is more efficient than learning from scratch.

\if{0}
\subsubsection*{Author Contributions}
If you'd like to, you may include  a section for author contributions as is done
in many journals. This is optional and at the discretion of the authors.

\subsubsection*{Acknowledgments}
Use unnumbered third level headings for the acknowledgments. All
acknowledgments, including those to funding agencies, go at the end of the paper.
\fi

\bibliography{MultiTaskPSR}
\bibliographystyle{apalike}

\newpage

\appendix

\linewidth\hsize {\centering
	{\Large\bfseries  Supplementary Materials \par}}

\allowdisplaybreaks

\section{Multi-task MLE Analysis}

The following lemma shows that the true model lies in the confidence set with high probability.

\begin{prop}[Confidence Set]\label{prop:MT MLE1} For all $k\in [K], $ and for any $\boldsymbol{\theta}=(\theta_1,\ldots,\theta_N) \in \ThetaB$, let $\eta\leq 1/(NKH)$, ${\beta}^{(N)} = c\log\left(\mathcal{N}_{\eta}(\ThetaB)NKH/\delta\right)$ for some $c \geq 0$. With probability at least $1-\delta$, for any $k\in[K]$, we have
    \begin{align}
        \sum_{n\leq N}\sum_{t\leq k} \sum_{h\leq H} \log  \Pb_{\theta_n^*}^{\nu^{\pi^{n,t}}_{h}}(\tau_H^{n,t,h}) \geq \sum_{n\leq N} \sum_{t \leq k} \sum_{h\leq H} \log   \Pb_{\theta_n}^{\nu^{\pi^{n,t}}_{h}}(\tau_H^{n,t,h})   - \beta^{(N)}.
    \end{align}
\end{prop}

\begin{proof} 

Consider a set of $\eta$-brackets, denoted by $\boldsymbol{\Theta}_{\eta}$, that covers $\boldsymbol{\Theta}$. For any $\boldsymbol{\theta}$, we can find two measures in $\boldsymbol{\Theta}_{\eta}$ parameterized by $\underline{\boldsymbol{\theta}}$ and $\overline{\boldsymbol{\theta}}$ such that
\begin{align}
        \forall n,\tau_H,\pi, \quad &\Pb_{\overline{\theta}_n}^{\pi} (\tau_H) \geq \Pb_{\theta_n }^{\pi}(\tau_H) \geq \Pb_{\underline{\theta}_n} (\tau_H), \label{Eq: prop1-1}\\
        \forall n, \quad &\mathtt{D}_{\TV}\left( \Pb_{\overline{\theta}_n}^{\pi} (\tau_H) , \Pb_{\underline{\theta}_n }^{\pi}(\tau_H) \right) \leq \eta. \nonumber
\end{align}

Note that the above two inequalities imply that $\sum_{\tau_H} \Pb_{\overline{\theta}_n}^{\pi} (\tau_H) \leq    \eta + \sum_{\tau_H}\Pb_{\underline{\theta}_n }^{\pi}(\tau_H) \leq 1 +\eta.$

Then, we have
    \begin{align*}
        \Eb&\left[\exp\left(\sum_n\sum_{t=1}^{k}\sum_h\log   \frac{\Pb_{\overline{\theta}_n}^{\nu^{\pi^{n,t}}_{h}}(\tau_H^{n,t,h})}{\Pb_{\theta_n^*}^{\nu^{\pi^{n,t}}_{h}}(\tau_H^{n,t,h})}   \right)\right] \\
        & = \Eb\left[\exp\left(\sum_n\sum_{t=1}^{k-1}\sum_h\log  \frac{\Pb_{\overline{\theta}_n}^{\nu^{\pi^{n,t}}_{h}}(\tau_H^{n,t,h})}{\Pb_{\theta_n^*}^{\nu^{\pi^{n,t}}_{h}}(\tau_H^{n,t,h})}  \right)\Eb\left[\prod_n\prod_h \frac{\Pb_{\overline{\theta}_n}^{\nu^{\pi^{n,k}}_h}(\tau_H^{n,k,h})}{\Pb_{\theta_n^*}^{\nu^{\pi^{n,k}}_h}(\tau_H^{n,k,h})}    \right] \right] \\
        & =\Eb\left[\exp\left(\sum_n\sum_{t=1}^{k-1}\sum_h\log  \frac{\Pb_{\overline{\theta}_n}^{\nu^{\pi^{n,t}}_{h}}(\tau_H^{n,t,h})}{\Pb_{\theta_n^*}^{\nu^{\pi^{n,t}}_{h}}(\tau_H^{n,t,h})}  \right)\prod_n\prod_h \sum_{\tau_H^{n,k,h}}\Pb_{\overline{\theta}_n}^{\nu^{\pi^{n,k}}_h}(\tau_H^{n,k,h})\right]  \\
        & \leq \Eb\left[\exp\left(\sum_n\sum_{t=1}^{k-1}\sum_h \log \left[\frac{\Pb_{\overline{\theta}_n}^{\nu^{\pi^{n,t}}_{h}}(\tau_H^{n,t,h})}{\Pb_{\theta_n^*}^{\nu^{\pi^{n,t}}_{h}}(\tau_H^{n,t,h})}\right] \right)\prod_n \prod_h \left(1+\frac{1}{NKH} \right)\right]\\
        & \leq \left(1+\frac{1}{NKH}\right)^{NKH} \leq e.
    \end{align*}
    
Therefore, by Chernoff type bound, we have 
    \begin{align*}
        &\Pb\left( \sum_{t\leq k, h\leq H \atop n\leq N} \log  \frac{\Pb_{\overline{\theta}_n}^{\nu^{\pi^{n,t}}_{h}}(\tau_H^{n,t,h})}{\Pb_{\theta_n^*}^{\nu^{\pi^{n,t}}_{h}}(\tau_H^{n,t,h})}  \geq \log(1/\delta)\right)   \leq \frac{\Eb\left[\exp\left(\sum_{t\leq k, h\leq H \atop n\leq N} \log  \frac{\Pb_{\overline{\theta}_n}^{\nu^{\pi^{n,t}}_{h}}(\tau_H^{n,t,h})}{\Pb_{\theta_n^*}^{\nu^{\pi^{n,t}}_{h}}(\tau_H^{n,t,h})}\right)\right]}{{1/\delta}} \leq e \delta.
    \end{align*}
    Taking a union bound over all $(\overline{\theta},k) \in \bar{\boldsymbol\Theta}_{\epsilon} \times [K]$ and rescaling $\delta$, we have for any $\boldsymbol{\theta} \in \boldsymbol{\Theta}$, 
    \begin{align*}
        \Pb\left(\forall \overline{\boldsymbol{\theta}}\in\boldsymbol{\Theta}_{\eta}, k\in[K],~~ \sum_{t\leq k, h\leq H \atop n\leq N} \log  \frac{\Pb_{\overline{\theta}_n}^{\nu^{\pi^{n,t}}_{h}}(\tau_H^{n,t,h})}{\Pb_{\theta_n^*}^{\nu^{\pi^{n,t}}_{h}}(\tau_H^{n,t,h})}   \geq \log(eK\mathcal{N}_{\eta}(\boldsymbol\Theta)/\delta)\right) \leq \delta.
    \end{align*}
The proof is finished by noting that $\overline{\boldsymbol{\theta}}$ is an optimistic measure (see \Cref{Eq: prop1-1}).
\end{proof}

The following lemma establishes the relationship between the Hellinger-squared distance and the difference of log likelihood functions between true parameters and any possible parameters from the model class. 
% \jing{log likelihood ratio between xxx?}.

\begin{prop}\label{prop: hellinger< log}

Let $\eta\leq 1/(N^2K^2H^2)$, ${\beta}^{(N)} = c\log\left(\mathcal{N}_{\eta}(\ThetaB)NKH/\delta\right)$ for some $c \geq 0$. Then, with probability at least $1-\delta$, we have, for any $\boldsymbol{\theta}=(\theta_1,\ldots,\theta_N) \in \ThetaB$, the following inequality holds.
    \begin{align*}
        \sum_{t\leq k, h\leq H \atop n\leq N} \mathtt{D}_{\mathtt{H}}^2 \left( \Pb_{ \theta_n }^{ \nu^{\pi^{n,t}}_{h} } , \Pb_{\theta_n^* }^{\nu^{\pi^{n,t}}_{h} }\right) \leq \sum_{t\leq k, h\leq H \atop n\leq N} \log\frac{ \Pb_{\theta_n^*}^{\nu^{\pi^{n,t}}_{h}} (\tau_H^{n,t,h}) }{ \Pb_{\theta_n}^{ \nu^{\pi^{n,t}}_{h} } (\tau_H^{n,t,h}) } + {\beta}^{(N)}.
    \end{align*}
\end{prop}

\begin{proof}

By the definition of $\eta$-bracket, for any multi-task parameter $\boldsymbol{\theta}$, we can find $\overline{\boldsymbol\theta}$ within a finite set of $\eta$-brackets such that $\sum_{\tau_H}\left|\Pb_{ \theta_n}^{\pi}(\tau_H) -  \Pb_{\bar{\theta}_n}^{\pi}(\tau_H) \right|\leq\eta$. 
Then, for any $n$ and $\pi$, we have

\begin{align}
     \mathtt{D}_{\mathtt{H}}^2 & ( \Pb_{\theta_n}^{\pi} , \Pb_{\theta_n^*}^{\pi} ) \nonumber\\
    & =  1 -  \sum_{\tau_H} \sqrt{ \Pb_{\theta_n}^{\pi}(\tau_H) \Pb_{\theta_n^*}^{\pi} (\tau_H)  }   \nonumber\\
    & =   1 -  \sum_{\tau_H} \sqrt{ \Pb_{\overline{\theta}_n}^{\pi}(\tau_H) \Pb_{\theta^*}^{\pi} (\tau_H)  +  \left(\Pb_{ \theta_n }^{\pi}(\tau_H) -  \Pb_{\overline{\theta}_n }^{\pi}(\tau_H) \right)\Pb_{\theta_n^*}^{\pi} (\tau_H)}  \nonumber\\
    & \overset{\RM{1}}{\leq} 1 - \sum_{\tau_H} \sqrt{ \Pb_{\overline{\theta}_n}^{\pi}(\tau_H) \Pb_{\theta_n^*}^{\pi} (\tau_H)  } + \sum_{\tau_H}\sqrt{ \left|\Pb_{ \theta_n }^{\pi}(\tau_H) -  \Pb_{\overline{\theta}_n }^{\pi}(\tau_H) \right| \Pb_{\theta_n^*}^{\pi} (\tau_H)} \nonumber\\
    & \overset{\RM{2}}{\leq} - \log \mathop{\Eb}_{\tau_H\sim\Pb_{\theta_n^*}^{\pi}(\cdot)} \sqrt{ \frac{ \Pb_{ \overline{\theta}_n }^{\pi}(\tau_H) }{ \Pb_{\theta_n^*}^{\pi}(\tau_H) } } + \sqrt{\sum_{\tau_H}\left|\Pb_{ \theta_n }^{\pi}(\tau_H) -  \Pb_{\overline{\theta}_n}^{\pi}(\tau_H) \right| } \nonumber\\
    & \overset{\RM{3}}{\leq}  - \log \mathop{\Eb}_{\tau_H\sim\Pb_{\theta_n^*}^{\pi}(\cdot)} \sqrt{ \frac{ \Pb_{ \overline{\theta}_n }^{\pi}(\tau_H) }{ \Pb_{\theta_n^*}^{\pi}(\tau_H) } } + \sqrt{\eta}, \label{Eq: prop2-1}
\end{align}
where $\RM{1}$ follows from the fact that $\sqrt{a+b}\leq \sqrt{|a|}+\sqrt{|b|}$, $\RM{2}$ follows from the fact that $1-x \leq \log x$ for $x \geq 0$ (the first term) and the Cauchy-Schwarz inequality (the second term), and $\RM{3}$ follows from the definition of $\eta$-bracket.

%Extend Lemma 15 in \cite{liu2022partially} to the multitask scenario, 
Then, for any fixed $\Bar{\theta}_n$, we have
\begin{align}
    \Eb&\left[\exp\left( \frac{1}{2} \sum_{n\leq N} \sum_{t\leq k}\sum_{h\leq H} \log\frac{ \Pb_{\overline{\theta}_n}^{\nu^{\pi^{n,t}}_{h}}(\tau_H^{n,t,h}) }{\Pb_{\theta_h^*}^{\nu^{\pi^{n,t}}_{h}}(\tau_H^{n,t,h})}  - \sum_n\sum_{ \pi \in \Pi_n^k}\log \mathop{\Eb}_{\tau_H\sim\Pb_{\theta_n^*}^{\pi}(\cdot)} \sqrt{ \frac{ \Pb_{ \overline{\theta}_n}^{\pi}(\tau_H) }{ \Pb_{\theta_n^*}^{\pi}(\tau_H) } } \right)\right] \nonumber\\
    & = \frac{ \Eb\left[ \prod_{n\leq N} \prod_{ t\leq k } \prod_{h\leq H} \sqrt{ \frac{ \Pb_{ \overline{\theta}_n }^{\nu^{\pi^{n,t}}_{h}}(\tau_H) }{ \Pb_{\theta_n^*}^{\nu^{\pi^{n,t}}_{h}}(\tau_H) } } \right] }{ \Eb\left[ \prod_{n\leq N} \prod_{t\leq k} \prod_{h\leq H} \sqrt{ \frac{ \Pb_{\overline{\theta}_n}^{\nu^{\pi^{n,t}}_{h}}(\tau_H) }{ \Pb_{\theta_n^*}^{\nu^{\pi^{n,t}}_{h}}(\tau_H) } } \right] } \nonumber\\
    & = 1. \label{Eq: prop2-2}
\end{align}

Hence, by taking union bound over the finite set of $\eta$-brackets and $k\in[K]$, with probability at least $1-\delta$, we have for any $k\in[K]$
\begin{align*}
    \sum_{n\leq N} \sum_{t\leq k} \sum_{h\leq H} & \mathtt{D}_{\mathtt{H}}^2   ( \Pb_{\theta_n}^{\nu^{\pi^{n,t}}_{h}} , \Pb_{\theta_n^*}^{\nu^{\pi^{n,t}}_{h}} ) \\
    & \overset{\RM{1}}{\leq} \sum_{n\leq N} \sum_{t\leq k} \sum_{h\leq H}- \log \mathop{\Eb}_{\tau_H\sim\Pb_{\theta_n^*}^{\nu^{\pi^{n,t}}_{h}}(\cdot)} \sqrt{ \frac{ \Pb_{ \overline{\theta}_n }^{\nu^{\pi^{n,t}}_{h}}(\tau_H) }{ \Pb_{\theta_n^*}^{\nu^{\pi^{n,t}}_{h}}(\tau_H) } } + NKH\sqrt{\eta}\\
    & \overset{\RM{2}}{\leq} NKH\sqrt{\eta} + \frac{1}{2} \sum_{n\leq N} \sum_{t\leq k}\sum_{h\leq H}   \log\frac{ \Pb_{ \theta_n^*}^{\nu^{\pi^{n,t}}_{h}}(\tau_H^{n,t,h}) }{\Pb_{\overline{\theta}_n}^{\nu^{\pi^{n,t}}_{h}}(\tau_H^{n,t,h})} + \log\frac{K\mathcal{N}_{\eta}(\boldsymbol{\Theta})}{\delta} \\
    &\overset{\RM{3}}{\leq} 1 + \frac{1}{2} \sum_{n\leq N} \sum_{t\leq k}\sum_{h\leq H}   \log\frac{ \Pb_{ \theta_n^*}^{\nu^{\pi^{n,t}}_{h}}(\tau_H^{n,t,h}) }{\Pb_{\overline{\theta}_n}^{\nu^{\pi^{n,t}}_{h}}(\tau_H^{n,t,h})}  + \log\frac{K\mathcal{N}_{\eta}(\boldsymbol{\Theta})}{\delta}.
\end{align*}
where $\RM{1}$ follows from \Cref{Eq: prop2-1}, $\RM{2}$ follows from \Cref{Eq: prop2-2}, the Chernoff’s method and the union bound, and $\RM{3}$ follows from that $\eta \leq 1/
(N^2K^2H^2)$.

The proof is finished by noting that $\overline{\theta}_n$ is an optimistic measure.

\end{proof}

\section{Properties of PSRs}

\if{0}
Useful idendities
\begin{align}
    &\bar{\psi}(\tau_h) = \frac{\psi(\tau_h)}{\phi_h^{\top}\psi(\tau_h)} \\
    & \frac{\phi_h^{\top}\psi(\tau_h)}{\phi_{h-1}^{\top}\psi(\tau_{h-1})} = \frac{\phi_h^{\top}\psi(\tau_h)\pi(\tau_h)}{\phi_{h-1}^{\top}\psi(\tau_{h-1})\pi(\tau_h)} = \frac{\Pb_{\theta}^{\pi}(\tau_h)}{ \pi(a_h|o_h,\tau_{h-1}) \Pb_{\theta}^{\pi}(\tau_{h-1}) } = \Pb_{\theta}(o_h|\tau_{h-1}) \\
    & \Mbf(o_h,a_h)\bar{\psi}(\tau_{h-1}) = \frac{\psi(\tau_h)}{\phi_{h-1}^{\top}\psi(\tau_{h-1})} = \bar{\psi}(\tau_h)\frac{\phi_h^{\top}\psi(\tau_h)}{\phi_{h-1}^{\top}\psi(\tau_{h-1})} = \bar{\psi}(\tau_h) \Pb_{\theta}(o_h|\tau_{h-1})
\end{align}
\fi

First, for any model $\theta = \{\phi_h, \Mbf_h(o_h,a_h)\}$, we have the following identity
\begin{align}
    &\Mbf_h(o_h,a_h) \bar{\psi}(\tau_{h-1}) = \Pb_{\theta}(o_h|\tau_{h-1})\bar{\psi}(\tau_h).\label{eqn:Mpsi is conditional probability}
\end{align}

The following proposition is adapted from Lemma C.3 in \citet{liu2022optimistic} and Proposition 1 in \citet{huang2023provably}.
\begin{prop}[TV-distance $\leq$ Estimation error]\label{prop:TV less than Estimation error}
For any task $n \in [N]$, policy $\pi$, and any two parameters $\theta,\theta'\in \Theta$, we have
    \begin{align*}
        \mathtt{D}_{\TV} \left( \Pb_{\theta'}^{\pi} , \Pb_{\theta}^{\pi} \right) \leq    \sum_{h=1}^H \sum_{\tau_H} \left| \mbf'(\omega_h)^{\top} \left(  \Mbf_h'(o_h,a_h) - \Mbf_h(o_h,a_h) \right) \psi(\tau_{h-1}) \right|\pi(\tau_H).
    \end{align*}
\end{prop}

\section{Proofs for Upstream learning: Proof of \Cref{thm:upstream}}

In this section, we first prove two lemmas, and then provide the proof for \Cref{thm:upstream}.

First, by the algorithm design and the construction of the confidence set, we have the following estimation guarantee.
\begin{lemma}[Estimation Guarantee in Upstream Learning]\label{lemma:Est guarantee upstream}
Let $\eta\leq 1/(N^2K^2H^2)$, ${\beta}^{(N)} = c\log\left(\mathcal{N}_{\eta}(\ThetaB)NKH/\delta\right)$ for some $c \geq 0$. With probability at least $1-\delta$, for any $k$ and any $\hat{\boldsymbol{\theta}} = (\hat{\theta}_1\ldots,\hat{\theta}_n) \in \boldsymbol{\mathcal{B}}_k$, we have
    \begin{align*}
        \sum_{n\leq N}\sum_{t\leq k-1}\sum_{h\leq H} \mathtt{D}_{\mathtt{H}}^2\left(\Pb_{\hat{\theta}_n}^{\nu^{\pi^{n,t}}_{h} } , \Pb_{\theta_n^* }^{\nu^{\pi^{n,t}}_{h} }\right) \leq 2\beta^{(N)}.
    \end{align*}

\end{lemma}

\begin{proof}
    The proof follows directly by combining \Cref{prop:MT MLE1} and \Cref{prop: hellinger< log}, and the optimality of the confidence set $\boldsymbol{\mathcal{B}}_k$. 
\end{proof}

Then, we show that estimation error can be upper bounded by the norm of prediction features.

By Lemma G.3 in \cite{liu2022optimistic}, for any task $n$, and step $h$, we can find a projection $\mathbf{A}_h^n \in \mathbb{R}^{d_{h-1}\times r}$ such that 

\begin{equation}
    (i): \psi^{n,*}(\tau_{h-1}) = \mathbf{A}_h^n (\mathbf{A}_h^n)^{\dagger} \psi^{n,*}(\tau_{h-1}), \quad  (ii): \|\mathbf{A}_h^n\|_1 \leq 1. \label{eqn:projector}
\end{equation}

\begin{lemma}\label{lemma:Up TV<feature norm}

Let $\mathbf{A}_h^n \in \mathbb{R}^{d_{h-1}\times r}$ be the projector satisfying \Cref{eqn:projector}.  
Fix $k\in[K]$. For any $\hat{\boldsymbol{\theta}}^k = (\hat{\theta}^{k}_1\ldots,\hat{\theta}^{k}_n) \in\boldsymbol{\mathcal{B}}_k$ and any multi-task policy $\boldsymbol{\pi} = (\pi_1,\ldots,\pi_n)$, we have
\begin{align*}
    \sum_{n\leq N} \mathtt{D}_{\TV}\left(\Pb_{\hat{\theta}^{k}_n}^{\pi_n}, \Pb_{\theta_n^*}^{\pi_n} \right) \leq \sqrt{\sum_{n\leq N} \sum_{h\leq H} \mathop{\Eb}_{\tau_{h-1}\sim\Pb_{\theta_n^*}^{\pi_n } } \left[ \left\| (\mathbf{A}_h^n)^{\dagger} \bar{\psi}^{n,*}(\tau_{h-1}) \right\|^2_{(U_{k,h}^n)^{-1}} \right] },
\end{align*}
where 
\begin{align*}
    U_{k,h}^n = \lambda I + (\mathbf{A}_h^n)^\dagger\sum_{t<k}\mathop{\Eb}_{\tau_{h-1}\sim \Pb_{\theta_n^*}^{\pi^{n,t}}} \left[ \bar{\psi}^{n,*}(\tau_{h-1})\bar{\psi}^{n,*}(\tau_{h-1})^{\top} \right] ((\mathbf{A}_h^n)^\dagger)^{\top}.
\end{align*}
 
\end{lemma}

\begin{proof}

By \Cref{prop:TV less than Estimation error}, we have
\begin{align*}
    \mathtt{D}_{\TV} \left(\Pb_{\hat{\theta}^k_n}^{\pi_n}, \Pb_{\theta_n^*}^{\pi_n} \right) \leq \sum_{h=1}^H\sum_{\tau_H}\left|\hat{\mbf}^{n,k}(\omega_h)^{\top}\left(\hat{\Mbf}_h^{n,k}(o_h,a_h) - \Mbf^{n,*}_h(o_h,a_h) \right) \psi^{n,*}(\tau_{h-1})\right|\pi_n(\tau_H).
\end{align*}

For ease of presentation, we fix a task index $n$. Index $\tau_{h-1}$ by $i$, $\omega_{h-1}$ by $j$. Denote $(\mathbf{A}_h^n)^{\dagger}\bar{\psi}^{n,*}(\tau_{h-1})$ by $x_i^n$, $\hat{\mbf}^{n,k}(\omega_h)^{\top} \left(\hat{\Mbf}_h^{n,k}(o_h,a_h) - \Mbf_h^{n,*}(o_h,a_h)\right)\mathbf{A}_h^n \pi_n(\omega_{h-1}|\tau_{h-1})$ by $(w_{j|i}^{n})^{\top}$.

Then, we have
\begin{align*}
    &\sum_{\tau_H}\left|\hat{\mbf}^{n,k}(\omega_h)^{\top}\left(\hat{\Mbf}_h^{n,k}(o_h,a_h) - \Mbf^{n,*}(o_h,a_h) \right) \psi^{n,*}(\tau_{h-1})\right|\pi_n(\tau_H)\\
     &\quad  \overset{\RM{1} }= \sum_{\tau_H}\left|\hat{\mbf}^{n,k}(\omega_h)^{\top}\left(\hat{\Mbf}_h^{n,k}(o_h,a_h) - \Mbf^{n,*}(o_h,a_h) \right) \bar{\psi}^{n,*}(\tau_{h-1}) \pi_n(\omega_{h-1}|\tau_{h-1})\right| \Pb_{\theta_n^*}^{\pi_n}(\tau_{h-1})\\
     &\quad = \sum_{i}\Pb_{\theta_n^*}^{\pi_n}(i) \sum_{j}\left|(w_{j|i}^n)^{\top}x_i^n\right|  \\
     &\quad = \Eb_{i\sim \Pb_{\theta_n^*}^{\pi_n}} \left[ (x_i^n)^{\top} \left(\sum_{j}w_{j|i}^n\mathtt{sgn}((w_{j|i}^n)^{\top}x_i^n)\right) \right]\\
     &\quad \overset{\RM{2} }\leq \Eb_{i\sim\Pb_{\theta_n^*}^{\pi_n}} \left[ \left\|x_i^n\right\|_{(U_{k,h-1}^n)^{-1}} \left\|\sum_j w_{j|i}^n\mathtt{sgn}((w_{j|i}^n)^{\top}x_i^n)\right\|_{U_{k,h-1}^n} \right],
\end{align*}
where $\RM{1}$ follows from the property of the projection $\mathbf{A}_h^n$ and definition of the prediction feature $\bar{\psi}^{n,*}$, and $\RM{2}$ is due to the Cauchy's inequality.

Fix an index $i=i_0$. We aim to analyze $\big\|\sum_j w_{j|i_0}^n\mathtt{sgn}((w_{j|i_0}^n)^{\top}x_{i_0}^n)\big\|_{U_{k,h-1}^n}$. We have
\begin{align*}
    &\left\|\sum_j w_{j|i_0}^n\mathtt{sgn}((w_{j|i_0}^n)^{\top}x_{i_0}^n)\right\|^2_{U_{k,h-1}^n} \\
    &\quad = \underbrace{\lambda \left\| \sum_{j}\mathtt{sgn}((w_{j|i_0}^n)^{\top}x_{i_0}^n) w_{j|i_0}^n  \right\|_2^2}_{I_1}  + \underbrace{\sum_{t<k}\mathop{\Eb}_{ i \sim \Pb_{\theta_n^*}^{\nu^{\pi^{n,t}}_{h}}} \left(\sum_{j}\mathtt{sgn}((w_{j|i_0}^n)^{\top}x_{i_0}^n)(w_{j|i_0}^n)^{\top}x_i^n  \right)^2}_{I_2}.
\end{align*}

For the first term $I_1$, we have
\begin{align*}
    \sqrt{I_1} & = \sqrt{\lambda}\max_{x\in\Rb^r: \|x\|_2=1} \left| \sum_j \mathtt{sgn}((w_{j|i_0}^n)^{\top}x_{i_0}^n) (w_{j|i_0}^n)^{\top} x  \right|\\
    & \leq  \sqrt{\lambda}   \max_{x\in\Rb^r: \|x\|_2=1} \sum_{\omega_{h-1}} \pi_n(\omega_{h-1}|i_0 ) \left|\hat{\mbf}^{n,k}(\omega_h)^{\top} \left(\hat{\Mbf}_h^{n,k}(o_h,a_h) - \Mbf_h^{n,*} (o_h,a_h)\right) \mathbf{A}_h^n x \right|  \\
    & \leq   \sqrt{\lambda}    \max_{x\in\Rb^r: \|x\|_2=1} \sum_{\omega_{h-1}} \pi_n(\omega_{h-1}|i_0 ) \left|\hat{\mbf}^{n,k}(\omega_{h-1})^{\top} \mathbf{A}_h^n  x  \right| \\
    &\quad +  \sqrt{\lambda}    \max_{x\in\Rb^r: \|x\|_2=1} \sum_{\omega_{h-1}} \pi_n(\omega_{h-1}| i_0 ) \left|\hat{\mbf}^{n,k}(\omega_h)^{\top}   \Mbf_h^{n,*}(o_h,a_h)  \mathbf{A}_h^n x \right|  \\
    & \overset{\RM{1}}{\leq}  \frac{ \sqrt{\lambda}}{\gamma}  \max_{x\in\Rb^r: \|x\|_2=1} \| \mathbf{A}_h^n x\|_1   + \frac{\sqrt{\lambda}}{\gamma}  \max_{x\in\Rb^r: \|x\|_2=1}   \sum_{o_h,a_h} \pi_n(a_h|o_h, i_0) \left\|\Mbf^{n,*}(o_h,a_h) \mathbf{A}_h^n x \right\|_1  \\
    & \overset{\RM2}{\leq}  \frac{2\sqrt{\lambda r} Q_A}{\gamma^2},
\end{align*}
where %$\RM{1}$ follows from plugging into the definition of $x_{i_0}^n$ and $w^n_{j|i_0}$, $\RM{2}$ follows from triangle inequality and $\hat{\mathbf{m}}^{n,k}(\omega_{h-1})=\hat{\mathbf{M}}_h^{n,k}(o_h,a_h)\hat{\mathbf{m}}^{n,k}(\omega_{h})$, 
$\RM{1}$ follows from \Cref{assmp:well-condition}, and $\RM{2}$ follows from the property of $\mathbf{A}_h^n$ stated before \Cref{lemma:Up TV<feature norm}.

For the second term $I_2$, we have
\begin{align*}
    I_2&\leq \sum_{t<k}\mathop{\Eb}_{\tau_{h-1}\sim \Pb_{\theta_n^*}^{\pi^{n,t}}} \left(\sum_{\omega_{h-1}} \pi_n(\omega_{h-1}|i_0) \left|\hat{\mbf}^{n,k}(\omega_h)^{\top} \left(\hat{\Mbf}_h^{n,k}(o_h,a_h) - \Mbf_h^{n,*}(o_h,a_h)\right) \bar{\psi}^{n,*}(\tau_{h-1}) \right|  \right)^2 \\
    & \leq \sum_{t<k}\mathop{\Eb}_{\tau_{h-1}\sim \Pb_{\theta_n^*}^{\pi^{n,t}}} \left(\sum_{\omega_{h-1}} \pi_n(\omega_{h-1}|i_0) \left|\hat{\mbf}^{n,k}(\omega_h)^{\top}  \hat{\Mbf}^{n,k}(o_h,a_h) \left( \bar{\psi}^{n,*}(\tau_{h-1}) - \bar{\hat{\psi}}^{n,k}(\tau_{h-1})\right)\right|  \right.\\
    &  \quad \quad \quad +   \left. \sum_{\omega_{h-1}} \pi_n(\omega_{h-1}|i_0) \left|\hat{\mbf}^{n,k}(\omega_h)^{\top}  \left(\hat{\Mbf}^{n,k}(o_h,a_h)\bar{\hat{\psi}}^{n,k}(\tau_{h-1}) - \Mbf^{n,*}(o_h,a_h) \bar{\psi}^{n,*}(\tau_{h-1}) \right) \right|  \right)^2 \\
    & \overset{(a)}\leq \frac{1}{\gamma^2} \sum_{t<k}\mathop{\Eb}_{\tau_{h-1}\sim \Pb_{\theta_n^*}^{\pi^{n,t}}} \left[ \bigg( \left\|       \bar{\psi}^{n,*}(\tau_{h-1}) - \bar{\hat{\psi}}^{n,k}(\tau_{h-1})   \right\|_1 \right.   \\
    &\quad\quad\quad  + \left. \left. \sum_{o_h,a_h} \pi_n(a_h|o_h,i_0) \left\|    \Pb_{\hat{\theta}_n^{k}}(o_h|\tau_{h-1})\bar{\hat{\psi}}^{n,k}(\tau_{h}) - \Pb_{\theta_n^*}(o_h|\tau_{h-1})\bar{\psi}^{n,*}(\tau_{h}) \right\|_1  \right)^2 \right] \\
    & \leq \frac{2}{\gamma^2} \sum_{t<k} \underbrace{\mathop{\Eb}_{\tau_{h-1}\sim \Pb_{\theta_n^*}^{\pi^{n,t}}} \left[  \left\|       \bar{\psi}^{n,*}(\tau_{h-1}) - \bar{\hat{\psi}}^{n,k}(\tau_{h-1})   \right\|_1^2 \right] }_{I_{21}}   \\
    &\quad  + \frac{2}{\gamma^2} \sum_{t<k} \underbrace{ \mathop{\Eb}_{\tau_{h-1}\sim \Pb_{\theta_n^*}^{\pi^{n,t}}} \left[ \left( \sum_{o_h,a_h} \pi_n(a_h|o_h,i_0) \left\|    \Pb_{\hat{\theta}_n^{k}}(o_h|\tau_{h-1})\bar{\hat{\psi}}^{n,k}(\tau_{h}) - \Pb_{\theta_n^*}(o_h|\tau_{h-1})\bar{\psi}^{n,*}(\tau_{h}) \right\|_1  \right)^2 \right] }_{ I_{22} },
\end{align*}
where $(a)$ is due to \Cref{eqn:Mpsi is conditional probability} and \Cref{assmp:well-condition}.

Recall that the $\ell$-th coordinate of a prediction feature $\bar{\psi}(\tau_{h-1})$ is the conditional probability of core test $\mathbf{o}_{h-1}^{\ell}$. Hence, for the term $I_{21}$, we have
\begin{align*}
    I_{21} &= \mathop{\Eb}_{\tau_{h-1}\sim \Pb_{\theta_n^*}^{\pi^{n,t}}} \left[ \left(  \sum_{\ell=1}^{d_{h-1}^n} \left| \Pb_{\hat{\theta}_n^k}(\mathbf{o}_{h-1}^{n,\ell} | \tau_{h-1}, \mathbf{a}_{h-1}^{n,\ell} )  - \Pb_{\theta_n^*}(\mathbf{o}_{h-1}^{n,\ell} | \tau_{h-1}, \mathbf{a}_{h-1}^{n,\ell} )\right| \right)^2 \right] \\
    & \leq Q_A^2 \mathop{\Eb}_{\tau_{h-1}\sim \Pb_{\theta_n^*}^{\pi^{n,t}}} \left[ \left(  \Eb_{\mathbf{a} \sim \mathtt{u}_{\mathcal{Q}_{h-1}^{n,A} } } \sum_{ \mathbf{o}_{h-1} }\left| \Pb_{\hat{\theta}_n^k}(\mathbf{o}_{h-1} | \tau_{h-1}, \mathbf{a} )  - \Pb_{\theta_n^*}(\mathbf{o}_{h-1}  | \tau_{h-1}, \mathbf{a} )\right| \right)^2 \right]\\
    &= Q_A^2 \mathop{\Eb}_{\tau_{h-1}\sim \Pb_{\theta_n^*}^{\pi^{n,t}}}  \mathtt{D}_{\TV}^2 \left( \Pb_{\hat{\theta}_n^k}^{\mathtt{u}_{\mathcal{Q}_{h-1}^{n,A}}} (\omega_{h-1} | \tau_{h-1} ) , \Pb_{\theta_n^*}^{\mathtt{u}_{\mathcal{Q}_{h-1}^{n,A}}} (\omega_{h-1} | \tau_{h-1} )\right) \\
    &\overset{(a)}\leq Q_A^2 |\mathcal{A}| \mathop{\Eb}_{\tau_{h-2},o_{h-1}\sim \Pb_{\theta_n^*}^{\pi^{n,t}} } \mathop{\Eb}_{a_{h-1}\sim\mathtt{u}_{\mathcal{A}}} \mathtt{D}_{\mathtt{H}}^2 \left( \Pb_{\hat{\theta}_n^k}^{\mathtt{u}_{\mathcal{Q}_{h-1}^{n,A}}} (\omega_{h-1} | \tau_{h-1} ) , \Pb_{\theta_n^*}^{\mathtt{u}_{\mathcal{Q}_{h-1}^{n,A}}} (\omega_{h-1} | \tau_{h-1} )\right) \\
    &\overset{(b)}\leq Q_A^2 |\mathcal{A}| \mathtt{D}_{\mathtt{H}}^2 \left( \Pb_{\hat{\theta}_n^k}^{\nu_{h-1}^{\pi^{n,t}} } ( \tau_{H} ) , \Pb_{\theta_n^*}^{\nu_{h-1}^{\pi^{n,t}}} ( \tau_{H} )\right),
\end{align*}
where $(a)$ and $(b)$ follow from \Cref{lemma:TV and hellinger}.

In addition, we can bound $I_{22}$ as follows.

\begin{align*}
    I_{22} &\leq 2   \mathop{\Eb}_{\tau_{h-1}\sim \Pb_{\theta_n^*}^{\pi^{n,t}}} \left(\sum_{o_h,a_h} \pi_n(a_h|o_h,\tau_{h-1})   \left| \Pb_{\hat{\theta}_n^k}(o_h|\tau_{h-1})  - \Pb_{\theta_n^*}( o_h | \tau_{h-1}  ) \right| \|\bar{\hat{\psi}}^{n,k}(\tau_h)\|_1  \right)^2 \\
    &\quad + 2  \mathop{\Eb}_{\tau_{h-1}\sim \Pb_{\theta_n^*}^{\pi^{n,t}}} \left(\sum_{o_h,a_h} \pi_n(a_h|o_h,\tau_{h-1})   \Pb_{\theta_n^*}( o_h | \tau_{h-1}  )  \left\|\bar{\hat{\psi}}^{n,k}(\tau_h) - \bar{\psi}^{n,*}(\tau_{h}) \right\|_1   \right)^2 \\
    &\overset{\RM{1}}\leq 2 Q_A^2 \mathop{\Eb}_{\tau_{h-1}\sim \Pb_{\theta_n^*}^{\pi^{n,t}}} \mathtt{D}_{\TV}^2 \left(\Pb_{\hat{\theta}_n^{k}}(o_h|\tau_{h-1}), \Pb_{\theta^*}(o_h|\tau_{h-1})  \right) \\
    &\quad + 2 Q_A^2 \mathop{\Eb}_{\tau_{h-1}\sim \Pb_{\theta}^{\pi^{n,t}}} \mathop{\Eb}_{o_h\sim \Pb_{\theta_n^*}(\cdot|\tau_{h-1})} \mathop{\Eb}_{a_h\sim\pi_n}\mathtt{D}_{\TV}^2 \left( \Pb_{\hat{\theta}_n^k}^{\mathtt{u}_{\mathcal{Q}_h^{n,A}}}( \omega_{h} | \tau_{h}  ) , \Pb_{\theta_n^*}^{\mathtt{u}_{\mathcal{Q}_h^{n,A}}} ( \omega_{h} | \tau_{h} ) \right)  \\
    &\overset{\RM{2}}\leq 2Q_A^2 |\mathcal{A}| \mathop{\Eb}_{\tau_{h-2}, o_{h-1}\sim \Pb_{\theta_n^*}^{\pi^{n,t}}} \mathop{\Eb}_{a_{h-1}\sim\mathtt{u}_{\mathcal{A}}} \mathtt{D}_{\mathtt{H} }^2 \left(\Pb_{\hat{\theta}_n^{k}}(o_h|\tau_{h-1}), \Pb_{\theta^*}(o_h|\tau_{h-1})  \right) \\
    &\quad + 2 Q_A^2|\mathcal{A}| \mathop{\Eb}_{\tau_{h-1}\sim \Pb_{\theta}^{\pi^{n,t}}} \mathop{\Eb}_{o_h\sim \Pb_{\theta_n^*}(\cdot|\tau_{h-1})} \mathop{\Eb}_{a_h\sim \mathtt{u}_{\mathcal{A}} }\mathtt{D}_{\mathtt{H}}^2 \left( \Pb_{\hat{\theta}_n^k}^{\mathtt{u}_{\mathcal{Q}_h^{n,A}}}( \omega_{h} | \tau_{h}  ) , \Pb_{\theta_n^*}^{\mathtt{u}_{\mathcal{Q}_h^{n,A}}} ( \omega_{h} | \tau_{h} ) \right)  \\
    &\overset{\RM{3}}\leq 2Q_A^2|\mathcal{A}| \left( \mathtt{D}_{\mathtt{H}}^2 \left( \Pb_{\hat{\theta}_n^k}^{\nu_{h-1}^{\pi^{n,t}}}(\tau_H), \Pb_{\theta_n^*}^{\nu_{h-1}^{\pi^{n,t}}}(\tau_H) \right) + \mathtt{D}_{\mathtt{H}}^2 \left( \Pb_{\hat{\theta}_n^k}^{\nu_{h}^{\pi^{n,t}}}(\tau_H), \Pb_{\theta_n^*}^{\nu_{h}^{\pi^{n,t}}}(\tau_H) \right) \right),
\end{align*}
where $\RM{1}$ follows from that the coordinate of $\bar{\psi}$ takes the value on the conditional probability over core test, $\RM{2}$ and $\RM{3}$ follow from \Cref{lemma:TV and hellinger}.

Substituting the upper bounds of $I_{21}$ and $I_{22}$ into $I_2$, we obtain that
\begin{align*}
    I_2 \leq \frac{6Q_A^2|\mathcal{A}|}{\gamma^2} \sum_{t<k}\left( \mathtt{D}_{\mathtt{H}}^2 \left( \Pb_{\hat{\theta}_n^k}^{\nu_{h-1}^{\pi^{n,t}}}(\tau_H), \Pb_{\theta_n^*}^{\nu_{h-1}^{\pi^{n,t}}}(\tau_H) \right) + \mathtt{D}_{\mathtt{H}}^2 \left( \Pb_{\hat{\theta}_n^k}^{\nu_{h}^{\pi^{n,t}}}(\tau_H), \Pb_{\theta_n^*}^{\nu_{h}^{\pi^{n,t}}}(\tau_H) \right) \right).
\end{align*}

Denote $\mathtt{D}_{\mathtt{H}}^2 \left( \Pb_{\hat{\theta}_n^k}^{\nu_{h}^{\pi^{n,t}}}(\tau_H), \Pb_{\theta_n^*}^{\nu_{h}^{\pi^{n,t}}}(\tau_H) \right) $ by $\zeta_{t,h}^n$. Therefore, 
\begin{align*}
    \mathtt{D}_{\TV} \left( \Pb_{\hat{\theta}_n^k}^{\pi_n}, \Pb_{\theta^*}^{\pi_n}  \right) \leq \sum_{h} \Eb_{\tau_{h-1}\sim \Pb_{\theta_n^*}^{\pi_n}} \left[ \sqrt{C_{\lambda} +\sum_{t<k} \zeta_{t,h-1}^n  + \sum_{t<k}\zeta_{t,h}^n } \left\|(\mathbf{A}_h^n)^{\dagger} \bar{\psi}^{n,*}(\tau_{h-1}) \right\|_{(U_{k,h}^n)^{-1}} \right].
\end{align*}

Summing over $n$, we have
\begin{align*}
    \sum_{n} & \mathtt{D}_{\TV} \left( \Pb_{\hat{\theta}_n^k}^{\pi_n}, \Pb_{\theta^*}^{\pi_n}  \right) \\
    &\leq \sum_{n}\sum_h \sqrt{ C_{\lambda} + \sum_{t<k} \zeta_{t,h-1}^n  + \sum_{t<k}\zeta_{t,h}^n } \Eb_{\tau_{h-1}\sim \Pb_{\theta_n^*}^{\pi_n}} \left[ \left\|(\mathbf{A}_h^n)^{\dagger} \bar{\psi}^{n,*}(\tau_{h-1}) \right\|_{(U_{k,h}^n)^{-1}} \right] \\
    &\leq \sqrt{NHC_{\lambda} + \sum_n\sum_h\sum_{t<k} (\zeta_{t,h-1}^n + \zeta_{t,h}^n)  } \sqrt{ \sum_{n}\sum_h \left( \Eb_{\tau_{h-1}\sim \Pb_{\theta_n^*}^{\pi_n}} \left[ \left\|(\mathbf{A}_h^n)^{\dagger} \bar{\psi}^{n,*}(\tau_{h-1}) \right\|_{(U_{k,h}^n)^{-1}} \right] \right)^2 } \\
    &\overset{\RM{1}}\leq \frac{ Q_A\sqrt{|\mathcal{A}|\beta^{(N)}} }{\gamma}   \sqrt{ \sum_{n}\sum_h \left( \Eb_{\tau_{h-1}\sim \Pb_{\theta_n^*}^{\pi_n}} \left[ \left\|(\mathbf{A}_h^n)^{\dagger} \bar{\psi}^{n,*}(\tau_{h-1}) \right\|_{(U_{k,h}^n)^{-1}} \right] \right)^2 },
\end{align*}
where $\RM{1}$ is due to the estimation guarantee \Cref{lemma:Est guarantee upstream}.

%TV choose to use multi-task overall guarantee rather than treating tasks individually, which leads to charaterzation of the benefit of multi-task.

\end{proof}

\begin{theorem}[Restatement of \Cref{thm:upstream}]
   For any fixed $\delta \geq 0$, let $\boldsymbol{\Theta}$ be the multi-task parameter space,  $\beta^{(N)} = c_1(\log\frac{KHN}{\delta} + \log\mathcal{N}_{\eta}(\boldsymbol{\Theta}))$, where $c_1 \geq 0$ and $\eta\leq \frac{1}{KHN}$. Then, under \Cref{assmp:well-condition}, with probability at least $1-\delta$, UMT-PSR finds a multi-task model $\overline{\boldsymbol{\theta}} = (\bar{\theta}_1,\ldots,\bar{\theta}_N)$ such that
    \begin{align}
       \sum_{n=1}^{N}  \max_{\pi^n}\mathtt{D}_{\TV}&\left(\Pb_{ \bar{\theta}_n}^{\pi^n}, \Pb_{\theta_n^*}^{\pi^n}\right) \leq \tilde{O}\left(  \frac{Q_A}{\gamma}\sqrt{ \frac{rH|\mathcal{A}|N \beta^{(N)} }{ K} }  \right).
    \end{align}
    In addition, if $K=\frac{c_2r|\mathcal{A}|Q_A^2 H \beta^{(N)} }{ N\gamma^2\epsilon^2}$ for some $c_2$, UMT-PSR produces a multi-task policy $\overline{\boldsymbol{\pi}} = (\bar{\pi}^1,\ldots,\bar{\pi}^N)$ such that the average sub-optimality gap is at most $\epsilon$, i.e.
    \begin{align}
        \frac{1}{N} \sum_{n=1}^N \left( \max_{\pi} V_{\theta_n^*, R_n}^{\pi} - V_{\theta^*, R_n}^{\bar{\pi}^n} \right) \leq \epsilon.
    \end{align}
\end{theorem}

\begin{proof}
Note that $\bar{\boldsymbol{\theta}} \in \boldsymbol{\mathcal{B}}_k$ for all $k\in[K+1]$. Therefore, by \Cref{lemma:Up TV<feature norm}, we have
\begin{align*}
    K \sum_n & \mathtt{D}_{\TV} \left(\Pb_{\bar{\theta}_n}^{\pi_n }, \Pb_{\theta_n^*}^{\pi_n }\right) \\
    &\overset{\RM{1}}\leq \sum_{k\leq K} \sum_{n\leq N} \max_{\hat{\theta}_n^k, \tilde{\theta}_n^k \in \boldsymbol{\mathcal{B}}_k } \mathtt{D}_{\TV} \left(\Pb_{\hat{\theta}_n^{k} }^{\pi^{n,k} }, \Pb_{ \tilde{\theta}_n^k }^{\pi^{n,k} }\right) \\
    &\leq 2\sum_{k}\sum_{n} \max_{\hat{\theta}_n^k   \in \boldsymbol{\mathcal{B}}_k } \mathtt{D}_{\TV} \left(\Pb_{\hat{\theta}_n^{k} }^{\pi^{n,k} }, \Pb_{  \theta_n* }^{\pi^{n,k} }\right) \\
    & \leq \sum_{k\leq K} \frac{8Q_A\sqrt{|\mathcal{A}| \beta^{(N)} }}{\gamma} \sqrt{ \sum_{n\leq N} \sum_{h\leq H} \Eb_{\tau_{h}\sim\Pb_{\theta_n^*}^{\pi^{n,k}}} \left[  \left\| (\mathbf{A}_h^n)^{\dagger} \bar{\psi}^{n,*}(\tau_{h-1})\right\|_{(U_{k,h}^{n})^{-1}}^2 \right] } \\
    &\overset{\RM{2}}\leq   \frac{8Q_A\sqrt{ K |\mathcal{A}| \beta^{(N)} }}{\gamma} \sqrt{ \sum_{k\leq K} \sum_{n\leq N} \sum_{h\leq H} \Eb_{\tau_{h}\sim\Pb_{\theta_n^*}^{\pi^{n,k}}} \left[ \left\| (\mathbf{A}_h^n)^{\dagger} \bar{\psi}^{n,*}(\tau_{h-1})\right\|_{(U_{k,h}^{n})^{-1}}^2 \right] } \\
    &\overset{\RM{3}}\leq \frac{8Q_A\sqrt{ K |\mathcal{A}| \boldsymbol{\beta} }}{\gamma} \sqrt{rNH\log(1 + r K/\lambda ) },
\end{align*}
where $\RM{1}$ follows from the fact that $\bar{\theta}_n\in\boldsymbol{\mathcal{B}}_k$ for all $k\in[K]$, $\RM{2}$ is due to the Cauchy's inequality, and $\RM{3}$ follows from \Cref{lemma:elliptical potential lemma}.

Hence,
\begin{align*}
    \sum_{n} \mathtt{D}_{\TV} \left(\Pb_{\bar{\theta}_n}^{\pi_n }, \Pb_{\theta_n^*}^{\pi_n }\right) \leq \tilde{O} \left( \frac{ Q_A \sqrt{r|\mathcal{A}|NH \beta^{(N)}\log(1+rK/\lambda)} }{\gamma\sqrt{K} } \right) .
\end{align*}

\end{proof}

\section{Downstream learning: Proof of \Cref{thm:downstream l2}}\label{sec:prove downstream}

In this section, we first provide the full algorithm of OMLE. Then, we prove a new estimation guarantee under the presence of approximation error of $\hat{\Theta}_0^{\mathrm{u}}$. Finally, we provide the proof of \Cref{thm:downstream l2}.

\subsection{Optimistic model-based algorithm}

In this section, we provide the full algorithm of OMLE~\citep{liu2022optimistic,chen2022partially} given a model class $\hat{\Theta}$ and an estimation margin $\beta_0$, for the completeness of the paper.

First, OMLE seeks a exploration policy $\pi^k$ that maximizes the largest total variation distance between any two model parameters $\theta$ and $\theta'$ within a confidence set $\boldsymbol{\mathcal{B}}_k$. Then, OMLE uses policies adapted from $\pi^k$ to collect data. Finally, using the collected sample trajectories, OMLE constructs a confidence set which includes the true model parameter. The pseudo code is provided in \Cref{alg:DMT-PSR}.

\begin{algorithm}[h]
\caption{Downstream multi-task PSR (OMLE)}\label{alg:DMT-PSR}
\begin{algorithmic}[1]
\State {\bf Input:} $\boldsymbol{\mathcal{B}}_1= \hat{\Theta}_0^{\mathrm{u}}$, estimation margin $\beta_0$.

\For{$k=1,\ldots, K_{\mathrm{Down}}$}
\State \[\pi^k = \arg\max_{ \pi \in \Pi} \max_{ \theta, \theta' \in\boldsymbol{\mathcal{B}}_k  }  \mathtt{D}_{\TV}\left(\Pb_{\theta}^{\pi}, \Pb_{ \theta'}^{\pi}\right)\]

\For{$h\in[H]$}
\State Use $\nu_h^{\pi^{k}} $ to collect data $\tau_H^{k,h} $.
 
\EndFor

\State Construct confidence set $\boldsymbol{\mathcal{B}}_{k+1}:$

\begin{align*}
\textstyle
    \boldsymbol{\mathcal{B}}_{k+1} = \bigg\{ {\theta} \in \hat{\Theta}:  &\sum_{t< k}\sum_h \log \Pb_{\theta}^{\nu_h^{\pi^t}}(\tau_H^{t,h})  \geq \max_{ \theta'\in\hat{\Theta} } \sum_{t< k}\sum_h \log \Pb_{\theta}^{\nu_h^{\pi^t}}(\tau_H^{t,h})- \beta_0 \bigg\}\cap\boldsymbol{\mathcal{B}}_k
\end{align*}

\EndFor
\State {\bf Output:} Any $\bar{\theta}_0 \in \boldsymbol{\mathcal{B}}_{K_{\mathrm{Down}}+1}$, and a greedy policy $\bar{\pi}_0 =\arg\max_{{\pi}} V_{\bar{\theta}_0, R_0}^{\pi} $.

\end{algorithmic}

\end{algorithm}

\subsection{Estimation Guarantee of OMLE}

Recall that $\epsilon_0 = \mathrm{e}_{\alpha}(\hat{\Theta}_0^{\mathrm{u}}) = \min_{ \theta_0 \in \hat{\Theta}_0^{\mathrm{u}} } \max_{\pi} \mathtt{D}_{\mathrm{R},\alpha}(\Pb_{\theta_0^*}^{\pi}, \Pb_{\theta_0}^{\pi})$ is the approximation error of the model class $\hat{\Theta}_0^{\mathrm{u}}$. In this section, let $\theta_0^{\epsilon_0}  = \arg\min_{ \theta_0 \in \hat{\Theta}_0^{\mathrm{u}} } \max_{\pi} \mathtt{D}_{\mathrm{R},\alpha}(\Pb_{\theta_0^*}^{\pi}, \Pb_{\theta_0}^{\pi})$.

The following lemma is from Proposition B.1 in \citet{liu2022optimistic}.
\begin{lemma}\label{lemma: prop B.1}
    Let $\eta\leq \frac{1}{KH}$. With probability at least $1-\delta$, for any $\theta_0 \in \hat{\Theta}_0^{\mathrm{u}}$, we have 
    \begin{align*}
        \sum_{t<k}\sum_h \log \frac{ \Pb_{ \theta_0 }^{\nu_h^{\pi^t}}(\tau_H^{t,h}) }{\Pb_{\theta_0^{*}}^{\nu_h^{\pi^t}}(\tau_H^{t,h}) } \leq \log(\mathcal{N}_{\eta}(\hat{\Theta}_0^{\mathrm{u}})) + \log\frac{eK}{\delta}.
    \end{align*}
\end{lemma}

Then, we show that the log-likelihood of model $\theta_0^{\epsilon_0}$ is sufficiently large.

\begin{lemma}\label{lemma:theta^c close enough}
    With probability at least $1-\delta$, we have
    \begin{align*}
         \sum_{t<k}\sum_h \log \frac{ \Pb_{ \theta_0^* }^{\nu_h^{\pi^t}}(\tau_H^{t,h} ) }{\Pb_{\theta_0^{\epsilon_0}}^{\nu_h^{\pi^t}}(\tau_H^{t,h}) } \leq \epsilon_0 KH + \frac{ \mathbf{1}_{\{\epsilon_0\neq 0\}} }{\alpha-1}\log\frac{K}{\delta}.
    \end{align*}
\end{lemma}

\begin{proof}

By the definition of $\theta_0^{\epsilon_0}$, we have

\begin{align} \frac{1}{\alpha-1} \log \Eb_{\Pb_{\theta^*}^{\pi}} \left[ \left(\frac{ \Pb_{\theta^*}^{\pi'}(\tau_H) }{\Pb_{\theta^{\epsilon_0}}^{\pi'}(\tau_H) }\right)^{\alpha-1}\right] \leq \epsilon_0. \label{eqn:renyi}
\end{align}

If $\epsilon_0 = 0$, then the proof is trivial. In the following, we mainly consider the case when $\epsilon_0>0$.

By the Markov's inequality, for any $x\in\mathbb{R}$, we have
\begin{align*}
    \Pb &\left( \sum_{t<k}\sum_h \log \frac{\Pb_{\theta^*}^{\pi^{t,h}}(\tau_H^{t,h}) }{ \Pb_{\theta^{\epsilon_0}}^{\pi^{t,h}}(\tau_H^{t,h}) }  \geq x \right)\\
    & = \Pb\left( \prod_{t<k}\prod_h \left(\frac{\Pb_{\theta^*}^{\pi^{t,h}}(\tau_H^{t,h}) }{ \Pb_{\theta^{\epsilon_0}}^{\pi^{t,h}}(\tau_H^{t,h}) } \right)^{\alpha-1} \geq e^{(\alpha-1)x} \right)\\
    &\leq e^{-(\alpha-1)x} \Eb\left[ \prod_{t<k}\prod_h \left(\frac{\Pb_{\theta^*}^{\pi^{t,h}}(\tau_H^{t,h}) }{ \Pb_{\theta^{\epsilon_0}}^{\pi^{t,h}}(\tau_H^{t,h}) } \right)^{\alpha-1} \Eb \left[ \left(\frac{\Pb_{\theta^*}^{\pi^{k,h}}(\tau_H^{k,h}) }{ \Pb_{\theta^{\epsilon_0}}^{\pi^{k,h}}(\tau_H^{k,h}) } \right)^{\alpha-1} \bigg | \pi^{k,h} \right] \right] \\
    &\overset{\RM1}\leq e^{-(\alpha-1)x} e^{(\alpha-1)KH\epsilon_0}\\
    & = e^{-(\alpha-1)(x-KH\epsilon_0)},
\end{align*}
where $\RM1$ follows from \Cref{eqn:renyi}.

By choosing $x = \epsilon_0 KH + \frac{1}{\alpha-1}\log(K/\delta)$ and taking union bound over $k$, we conclude that, with probability at least $1-\delta$, 
\[  \sum_{t<k}\sum_h \log \frac{\Pb_{\theta^*}^{\pi^{t,h}}(\tau_H^{t,h}) }{ \Pb_{\theta^{\epsilon_0}}^{\pi^{t,h}}(\tau_H^{t,h}) }  \leq \epsilon_0 KH + \frac{1}{\alpha-1}\log\frac{K}{\delta} .\]

\end{proof}

Combining \Cref{lemma: prop B.1} and \Cref{lemma:theta^c close enough}, we immediately obtain that with probability at least $1-\delta/2$, the following bound holds.
\begin{align*}
    \sum_{t<k}\sum_h \log \Pb_{\theta_0^{\epsilon_0}}^{\nu_h^{\pi^t}}(\tau_H^{t,h}) &\geq \max_{\theta_0\in\hat{\Theta}_0^{\mathrm{u}}} \sum_{t<k}\sum_h \log \Pb_{\theta_0 }^{\nu_h^{\pi^t}}(\tau_H^{t,h}) \\
    &\quad- \left( \log\mathcal{N}_{\eta}(\hat{\Theta}_0^{\mathrm{u}}) + \log\frac{4eK}{\delta} + \epsilon_0KH + \frac{\mathbf{1}_{\{\epsilon_0\neq 0\}} }{\alpha-1} \log\frac{4K}{\delta}\right),
\end{align*}
where $\eta\leq \frac{1}{KH}$.

Setting $\beta_0 =  \log\mathcal{N}_{\eta}(\hat{\Theta}_0^{\mathrm{u}}) + \log\frac{4eK}{\delta} + \epsilon_0KH + \frac{\mathbf{1}_{\{\epsilon_0\neq 0\}} }{\alpha-1} \log\frac{4K}{\delta} $, we conclude that $\theta_0^{\epsilon_0}\in\boldsymbol{\mathcal{B}}_k$ for all $k\in[K]$. Based on this fact, we have the following estimation guarantee.

\begin{lemma}
    With probability at least $1-\delta$, for any $k\in[K]$ and $\theta_0 \in \boldsymbol{\mathcal{B}}_k$, we have
    \begin{align*}
        \sum_{t<k}\sum_h \mathtt{D}_{\TV}^2 \left( \Pb_{\theta_0}^{\nu_h^{\pi^t}}, \Pb_{\theta_0^*}^{\nu_h^{\pi^t}} \right) \leq 2\beta_0.
    \end{align*}
\end{lemma}

\begin{proof}

We follow the same argument as in \Cref{prop: hellinger< log}, except setting $N=1$ and $\boldsymbol{\Theta}_{\mathrm{u}} = \hat{\Theta}_0^{\mathrm{u}}$. Then, we obtain that, with probability at least $1-\delta/2$, the following inequality holds.

\begin{align*}
    \sum_{t<k} &\sum_h \mathtt{D}_{\mathtt{H}}^2 \left( \Pb_{\theta_0}^{\nu_h^{\pi^t}}, \Pb_{\theta_0^*}^{\nu_h^{\pi^t}} \right) \leq \sum_{t\leq k} \sum_h \log\frac{ \Pb_{\theta_0^*}^{\nu^{\pi^{t}}_{h}} (\tau_H^{t,h}) }{ \Pb_{\theta_0}^{ \nu^{\pi^{t}}_{h} } (\tau_H^{t,h}) } + \log \frac{K\mathcal{N}_{\eta}(\hat{\Theta}_0^{\mathrm{u}})}{\delta}.
\end{align*}

Since $\theta_0, \theta_0^{\epsilon_0}\in\boldsymbol{\mathcal{B}}_k$, by the optimality of $\boldsymbol{\mathcal{B}}_k$, we further have
\begin{align*}
    \sum_{t<k} &\sum_h \mathtt{D}_{\mathtt{H}}^2 \left( \Pb_{\theta_0}^{\nu_h^{\pi^t}}, \Pb_{\theta_0^*}^{\nu_h^{\pi^t}} \right) \\
    &\leq  \sum_{t\leq k} \sum_h \log\frac{ \Pb_{\theta_0^{*}}^{\nu^{\pi^{t}}_{h}} (\tau_H^{t,h}) }{ \Pb_{\theta_0^{\epsilon_0}}^{ \nu^{\pi^{t}}_{h} } (\tau_H^{t,h}) } +  \beta_0 +  \log \frac{K\mathcal{N}_{\eta}(\hat{\Theta}_0^{\mathrm{u}})}{\delta} \\
    &\overset{\RM{1}}\leq \epsilon_0KH + \frac{\mathbf{1}_{\{\epsilon_0\neq 0\}} }{\alpha-1}\log\frac{K}{\delta} 
    + \beta_0 + \log \frac{K\mathcal{N}_{\eta}(\hat{\Theta}_0^{\mathrm{u}})}{\delta}\\
    &\leq 2\beta_0,
\end{align*}
where $\RM{1}$ is due to \Cref{lemma:theta^c close enough}.

\end{proof}

\subsection{Proof of \Cref{thm:downstream l2}}

\begin{theorem}[Restatement of \Cref{thm:downstream l2}] 
    Fix $\alpha>1$. Let $\epsilon_0 = \mathrm{e}_{\alpha}(\hat{\Theta}_0^{\mathrm{u}})$, $\beta_0 = O(\log\frac{KH}{\delta} + \log\mathcal{N}_{\eta}(\hat{\Theta}_0^{\mathrm{u}}) + \epsilon_0 KH + \frac{\mathbf{1}_{\{\epsilon_0\neq0\}}}{\alpha -1})$, where $\eta\leq\frac{1}{KH}$. Under \Cref{assmp:well-condition}, with probability at least $1-\delta$, the output of \Cref{alg:DMT-PSR} satisfies that
    \begin{align}
         \textstyle \max_{\pi \in \Pi}\mathtt{D}_{\TV}\left(\Pb_{ \bar{\theta}_0}^{\pi}, \Pb_{\theta_0^*}^{\pi}\right) \leq \tilde{O}\left(  \frac{Q_A}{\gamma}\sqrt{ \frac{r|\mathcal{A}|H  \beta_0 }{K} }  +\sqrt{\epsilon_0} \right).
    \end{align}
\end{theorem}

\begin{proof}
    First, we follow the proof in \Cref{lemma:Up TV<feature norm}, except setting $N=1$. We obtain that
\begin{align*}
    \mathtt{D}_{\TV} & \left(\Pb_{\hat{\theta}_0^k}^{\pi}(\tau_H), \Pb_{\theta_0^*}^{\pi} (\tau_H) \right) \leq \frac{Q_A|\mathcal{A}|}{\gamma}\sqrt{C_{\lambda} + \sum_{t<k}\zeta_{t,h}^0}\sum_h \Eb_{\tau_{h-1}\sim\Pb_{\theta_0^*}^{\pi}}\left[ \|(\mathbf{A}_h^0)^{\dagger}\bar{\psi}^{0,*}(\tau_{h-1})\|_{(U_{k,h}^0)^{-1}} \right]\\
    &\overset{\RM{1}}\leq O\left(\frac{Q_A\sqrt{|\mathcal{A}|\beta_0}}{\gamma}\sum_h \Eb_{\tau_{h-1}\sim\Pb_{\theta_0^*}^{\pi}}\left[ \|(\mathbf{A}_h^0)^{\dagger}\bar{\psi}^{0,*}(\tau_{h-1})\|_{(U_{k,h}^0)^{-1}} \right] \right),
\end{align*}
where 
\begin{equation}
    \left\{\begin{aligned}
        &C_{\lambda} = \frac{\lambda r Q_A^2 |\mathcal{A}|}{\gamma^4}, \\
        &\lambda = \frac{\gamma^4\beta_0}{rQ_A^2|\mathcal{A}|}, \\
        &\zeta_{t,h}^0 = \mathtt{D}_{\mathtt{H}}^2\left( \Pb_{\hat{\theta}_0^k}^{\nu_h^{\pi^t}}(\tau_H^{t,h}) , \Pb_{\theta_0^*}^{\nu_h^{\pi^t}}(\tau_H^{t,h}) \right),\\
        & U_{k,h}^0 = \lambda I + (\mathbf{A}_h^0)^{\dagger}\sum_{t<k} \Eb_{\tau_{h-1}\sim \Pb_{\theta_0^*}}^{\nu_h^{t} } \bar{\psi}^{0,*}(\tau_{h-1})\bar{\psi}^{0,*}(\tau_{h-1})^{\top} ((\mathbf{A}_h^0)^{\dagger} )^{\top} ,
    \end{aligned}
    \right.
\end{equation}
and $\RM{1}$ is due to the estimation guarantee.

Therefore, 
\begin{align*}
    K \mathtt{D}_{\TV}& \left( \Pb_{\bar{\theta}_0}^{\pi}, \Pb_{\theta_0^*}^{\pi} \right) \\
    &\leq K \mathtt{D}_{\TV} \left( \Pb_{\bar{\theta}_0}^{\pi}, \Pb_{\theta_0^{\epsilon_0}}^{\pi} \right) + K \mathtt{D}_{\TV} \left( \Pb_{\theta_0^{\epsilon_0}}^{\pi}, \Pb_{\theta_0^*}^{\pi} \right)\\
    &\leq \sum_{k}\max_{\hat{\theta}_0^k, \tilde{\theta}_0^k\in \boldsymbol{\mathcal{B}}_k }  \mathtt{D}_{\TV} \left( \Pb_{\hat{\theta}_0^k}^{\pi^k}, \Pb_{\tilde{\theta}_0^k}^{\pi^k} \right) + K\sqrt{\epsilon_0}\\
    &\leq 2 \sum_{k}\max_{\hat{\theta}_0^k \in \boldsymbol{\mathcal{B}}_k }  \mathtt{D}_{\TV} \left( \Pb_{\hat{\theta}_0^k}^{\pi^k}, \Pb_{  \theta_0^*}^{\pi^k} \right)  + K\sqrt{\epsilon_0}\\
    &\leq \frac{Q_A\sqrt{|\mathcal{A}|(\beta_0+\epsilon_0KH + \frac{1}{\alpha-1}\log(K/\delta))}}{\gamma} \sqrt{rHK\log(1+rK/\lambda)} + K\sqrt{\epsilon_0}.
\end{align*}

\end{proof}

\section{Bracketing numbers of Examples and Missing Proofs in \Cref{subsec: example-up} and \Cref{sec:downstream example}}\label{Appd: F bracketing number}
 In the section, we present bracketing numbers of examples and missing Proofs in \Cref{subsec: example-up} and \Cref{sec:downstream example}.
 Because $\phi_h^\top=\sum_{(o_h,a_h)\in\Oc\times\Ac}\phi_{h+1}^{\top}\Mbf_h(o_h,a_h)$  for each $h \in [H-1]$, $\phi_h$ can be decided by $\phi_H$ and $\{\Mbf_h\}_{h=1}^H$. For simplicity, in this section, we reparameterize the PSRs parameters as $\theta= \{\phi_H,\{\Mbf_h\}_{h=1}^H\}$. Without loss of generality, from Theorem C.1 and C.7 in \cite{liu2022optimistic}, we assume for any $(o,a) \in \Oc \times \Ac$, $\Mbf_h(o,a) \in \Rb^{r \times r}$, and the rank of $\Mbf_h(o,a)$ is $r$.

\subsection{Bracketing Numbers of Basic single-task PSRs}
 For completeness, in this subsection, we first present the analysis of bracketing number of basic single-task PSRs.

\begin{lemma}[Bracketing number of single-task PSRs] \label{lemma: bracketing number of single-task}
Let $\Theta$ be the collection of PSR parameters of all rank-$r$ sequential decision making problems with obseration space $\Oc$, action space $\Ac$ and horizon $H$. Then we have 
\begin{align*}
    \log \Nc_{\eta}(\Theta) \leq O(r^2H^2|\Oc||\Ac|\log(\frac{|\Oc||\Ac|}{\eta})).
\end{align*}
\end{lemma}
\begin{proof}
    We assume $\psi_0$ is known,\footnote{Such assumption does not influence the order of bracketing number since the model complexity related to $\Psi_0$ does not dominate.} and $\norm{\psi_0}_2 \leq \sqrt{|\Ac|^H}$. By Corollary C.8 from \cite{liu2022optimistic}, the PSR model class have following form:
    \begin{align*}
        {\Theta}=\left\{ \theta:  \theta= \{\phi_H,\{\Mbf_h\}_{h=1}^H\},\norm{\Mbf_h(o,a)}_2 \leq 1\hspace{1pt}\text{for}\hspace{1pt}(o,a) \in \Oc \times \Ac, \norm{\phi_H}_2 \leq 1\right\}.
    \end{align*} 
    Denote $\widetilde{\Theta}_\delta$ as the $\delta$-cover of $\Theta$ w.r.t $\ell_\infty$-norm with $\delta=\frac{\eta}{(|\Oc||\Ac|)^{cH}}$ for some large $c >0$. Mathematically, for any $\theta = \{\phi_H, \{\Mbf_h\}_{h=1}^H\} \in \Theta$,  there exists $\tilde{\theta} = \{\tilde{\phi}_H, \{\widetilde{\Mbf}_h\}_{h=1}^H\} \in \widetilde{\Theta}_\delta$, such that for any $(o,a) \in \Oc \times \Ac$,
    \begin{align}
        \norm{\phi_H-\tilde{\phi}_H}_{\infty} \leq \delta, \quad
        \norm{\textbf{Vec}(\Mbf_h(o,a))-\textbf{Vec}(\widetilde{\Mbf}_h(o,a))}_{\infty} \leq \delta. \nonumber
    \end{align}
    
Next, we show that $\widetilde{\Theta}_\delta$ can constitute an $\eta$-bracket for $\Theta$. For any policy $\pi$, we have
    \begin{align}
        \sum_{\tau_H \in (\Oc \times \Ac)^H}&\left|\Pb_\theta(\tau_H)-\Pb_{\tilde{\theta}}(\tau_H)\right|\times \pi(\tau_H) \nonumber\\
        & \leq \sum_{\tau_H \in (\Oc \times \Ac)^H}\sum_{h=1}^H\left| \mbf_h(\omega_h)^{\top} \left( \widetilde{\Mbf}_h(o_h,a_h) - \Mbf_h(o_h,a_h) \right) \psi_{h-1}(\tau_{h-1}) \right|\times\pi(\tau_H) \nonumber\\
        & \overset{\RM{1}}{\leq} \sum_{\tau_H \in (\Oc \times \Ac)^H}\sum_{h=1}^H\norm{ \mbf_h(\omega_h)}_2 \norm{\left( \widetilde{\Mbf_h}(o_h,a_h) - \Mbf_h(o_h,a_h) \right)}_2\norm{ \psi_{h-1}(\tau_{h-1})}_2 \times\pi(\tau_H)\nonumber\\
         & \overset{\RM{2}}{\leq} \sum_{\tau_H \in (\Oc \times \Ac)^H}\sum_{h=1}^H \norm{\left( \widetilde{\Mbf}_h(o_h,a_h) - \Mbf_h(o_h,a_h) \right)}_2\sqrt{|\Ac|^H} \times\pi(\tau_H)\nonumber\\
         & \overset{\RM{3}}{\leq} \sum_{\tau_H \in (\Oc \times \Ac)^H}\sum_{h=1}^H \sqrt{r}\norm{\left( \widetilde{\Mbf_h}(o_h,a_h) - \Mbf_h(o_h,a_h) \right)}_{\infty}\sqrt{|\Ac|^H} \times\pi(\tau_H)\label{Eq: lemma0single-1}\\
         & {\leq} \sum_{\tau_H \in (\Oc \times \Ac)^H}\sum_{h=1}^H \sqrt{r^3}\norm{\left( \textbf{Vec}(\widetilde{\Mbf_h}(o_h,a_h) )- \textbf{Vec}(\Mbf_h(o_h,a_h)) \right)}_{\infty}\sqrt{|\Ac|^H} \times\pi(\tau_H)\nonumber\\
         & \leq H\sqrt{r^3|\Ac|^H}\delta \leq \eta, \nonumber
    \end{align}
    where $\RM{1}$ follows from the property of operation norm of matrix, $\RM{2}$ follows from the fact that $\norm{\mbf_h(\omega_h)}_2\leq\norm{\phi_H}_2\norm{\Mbf_H(o_H,a_H)}_2\ldots\norm{\Mbf_{h+1}(o_{h+1},a_{h+1})}_2 \leq 1$ and $\norm{\psi_{h-1}(\tau_{h-1})} \leq \norm{\Mbf_{h-1}(o_{h-1},a_{h-1})}_2\ldots\norm{\Mbf_1(o_1,a_1)}_2\norm{\psi_0}_2 \leq \sqrt{|\Ac|^H}$, and $\RM{3}$ follows from the relationship between the operation norm induced by $\ell_2$ norm and $\ell_\infty$ norm.
    
    By \Cref{lem:aux:cov},
    $$ |\widetilde{\Theta}_\delta| \leq \left(\frac{1+2\sqrt{r}}{\delta}\right)^{r+r^2\times H|\Oc||\Ac|}=\left(3\sqrt{r}\frac{(|\Oc||\Ac|)^{cH}}{\eta}\right)^{r+r^2\times H|\Oc||\Ac|}$$ and $$\log|\widetilde{\Theta}_\delta|=O(r^2H^2|\Oc||\Ac|\log(\frac{|\Oc||\Ac|}{\eta})),$$
    which equals to the log $\eta$-bracketing number.

\end{proof}

% As a direct result of \Cref{lemma: bracketing number of single-task}, we have following corollary.
% \begin{coro}[Bracketing number of indepedent learning single-task PSRs] \label{lemma: bracketing number of multi-task}
% Let $\Theta$ be the collection of PSR parameters of all rank-$r$ sequential decision making problems with obseration space $\Oc$, action space $\Ac$ and horizon $H$, then we have 
% \begin{align*}
%     \log \Nc_{\eta}(\Theta) \leq 
% \end{align*}
% \end{coro}

% \jing{    \textbf{multitask:} $\forall$ task $n$: if only Eqn.\ref{eqn:constraint} is satisfied, the total dimension of parameter space is $2r+r(r-1)\times HOA$.} \jing{?? }

\subsection{Bracketing Number of Upstream examples}

\textbf{\Cref{eg:same transtion}(Multi-task POMDP with same transition kernels):} The multi-task parameter space is $$\left\{(\mathbb{T}_{h,a}, \mathbb{O}_h^1,\ldots,\mathbb{O}_h^N): \mathbb{T}_{h,a}\in\mathbb{R}^{|\mathcal{S}|\times|\mathcal{S}|}, \mathbb{O}_h^i\in\mathbb{R}^{|\mathcal{O}|\times|\mathcal{A}|}, \forall i\in[N] \right\}_{h\in[H],a\in\mathcal{A}}.$$

Note that the value of each coordinate of these matrices are probabilities, and thus are bounded within [0,1]. Therefore, the $\eta$-bracketing number in this case is $O(H(|\mathcal{S}|^2|\mathcal{A}| + N|\mathcal{O}||\mathcal{S}|)\log\frac{H|\mathcal{S}||\mathcal{O}||\mathcal{A}|}{\eta} )$ 

\textbf{\Cref{eg:M+delta}(Multi-task PSR with perturbed models):}
     Suppose there exist a latent base task $\mathtt{P}_\mathrm{b}$, and a noisy perturbation space $\boldsymbol{\Delta}$. Each task $n \in [N]$ is a noisy perturbation of the latent base task and can be parameterized into two parts: the base task plus a task-specified noise term. Specifically, for each step $h \in [H]$ and task $n \in [N]$, any $(o,a) \in \Oc \times \Ac$, we have 
     \begin{align*}
         \textstyle \Mbf_h^n(o_h,a_h)=\Mbf_h^\mathrm{b}(o_h,a_h)+\Delta^n_h(o_h,a_h), \quad \Delta^n_h \in \boldsymbol{\Delta}.
     \end{align*}    
Such a multi-task PSR satisfies that $\beta^{(N)} \leq O( \log\frac{KHN}{\delta}+r^2 |\mathcal{O}| |\mathcal{A}| H^2\log\frac{ |\mathcal{O}| |\mathcal{A}| }{\eta}   + HN\log|\boldsymbol{\Delta}|)$, whereas $\beta^{(1)}$ for a single task is given by $ O( r^2 |\mathcal{O}| |\mathcal{A}| H^2\log\frac{ |\mathcal{O}| |\mathcal{A}| }{\eta}   + H(N-1)\log|\boldsymbol{\Delta}| )$.
Clearly, $\beta^{(N)} \ll N\beta^{(1)}$  holds if $\log |\boldsymbol{\Delta}| \ll \tilde{O}(r^2|\mathcal{O}||\Ac|H)$, which can be easily satisfied for low perturbation environments. In such a case, the multi-task PSR  benefits from  a significantly reduced sample complexity compared to single-task learning.

\begin{proof}[Proof of \Cref{eg:M+delta}]\label{Appd_sub: F-3}
        Suppose there exist a latent base task model space:
        \begin{align*}
           {\Theta^\mathrm{b}}=\left\{ \theta:  \theta= \{\phi_H,\{\Mbf_h\}_{h=1}^H\},\norm{\Mbf_h(o,a)}_2 \leq 1\hspace{1pt}\text{for}\hspace{1pt}(o,a) \in \Oc \times \Ac, \norm{\phi_H}_2 \leq 1\right\}.
        \end{align*}
        A base task model is selected: $\theta^\mathrm{b}=\{\phi^\mathrm{b}_H,\{\Mbf^\mathrm{b}_h\}_{h=1}^H\}$. The  parameters of each task $n$ in multi-task PSR models are as follows: 
    \begin{align*}
    {\Theta^n}& =\left\{ \theta:  \theta= \{\phi^\mathrm{b}_H,\{\Mbf^\mathrm{b}_h+\Delta_h^n\}_{h=1}^H\},\Delta_h^n \in \mathbf{\Delta}, \norm{\phi_H^n}_2 \leq 1\right\},\\
    \end{align*}
    where $\Delta_h^n(\cdot,\cdot): \mathcal{O}\times\mathcal{A}\rightarrow\mathbb{R}^{d_h\times d_{h-1}}$ for any $h \in [H]$, and $n \in [N]$, and $\mathbf{\Delta}$ is the noisy perturbation space with finite cardinality. 
    
    Let $\widetilde{\Theta}_\delta^\mathrm{b}$ be the $\delta$-cover of $\Theta^\mathrm{b}$ w.r.t $\ell_\infty$-norm with $\delta=\frac{\eta}{(|\Oc||\Ac|)^{cH}}$. From \Cref{lemma: bracketing number of single-task}, we have $|\widetilde{\Theta}_\delta^\mathrm{b}|=\left(\frac{(|\Oc||\Ac|)^{cH}}{\eta}\right)^{2r+r^2\times H|\Oc||\Ac|}$. For each $n \in [N]$, denote 
    \begin{align*}
        \widetilde{\Theta}_\delta^n &= \widetilde{\Theta}_\delta^\mathrm{b} + \mathbf{\Delta}\\
        &:= \left\{\theta: \theta=\left(\widetilde{\phi}^\mathrm{b}_H,\{\widetilde{\Mbf}^n_h+\Delta_h^n\}_{h=1}^H\right);\left(\widetilde{\phi}^\mathrm{b}_H,\{\widetilde{\Mbf}_h\}_{h=1}^H\right) \in \widetilde{\Theta}_\delta^\mathrm{b}; \Delta_h^n \in \mathbf{\Delta} \right\}.
    \end{align*}
    Obviously, $|\widetilde{\Theta}_\delta^n|=\left(3\sqrt{r}\frac{(|\Oc||\Ac|)^{cH}}{\eta}\right)^{r+r^2\times H|\Oc||\Ac|} \times |\mathbf{\Delta}|^H$. Then denote the multi-task $\delta$-cover as $\widetilde{\ThetaB}_\delta=\widetilde{\Theta}_\delta^1\times\cdots\times\widetilde{\Theta}_\delta^N$. Following from \Cref{lemma: bracketing number of single-task}, for any task $n \geq 1$, $\widetilde{\Theta}_\delta^n$ can constitute an $\eta$-bracket for $\Theta^n$, so $\widetilde{\ThetaB}_\delta$ can constitute an $\eta$-bracket for $\ThetaB$. By noticing that each $\bar{\Theta}_\delta^n$ has a common part, and the changing part is only related to $\Delta_h^n$. The corresponding $\eta$-bracketing number is $\left(3\sqrt{r}\frac{(|\Oc||\Ac|)^{cH}}{\eta}\right)^{r+r^2\times H|\Oc||\Ac|} \times |\mathbf{\Delta}|^{HN}$, and the log $\eta$-bracketing number is at most $O\left(r^2 |\mathcal{O}| |\mathcal{A}| H^2\log\frac{ |\mathcal{O}| |\mathcal{A}| }{\eta}   + H(N-1)\log|\boldsymbol{\Delta}|\right)$.  
\end{proof}

\textbf{\Cref{eg: linear com}(Multi-task PSRs: Linear combination of core tasks):}
Suppose that the multi-task PSR lies in the linear span of $m$ core tasks, i.e., there exist a set of core tasks indexed by $\{1,2,\ldots,m\}$ such that each PSR can be represented as a linear combination of those $m$ core tasks. Specifically, for each task $n \in [N]$, there exists a coefficient  vector $\boldsymbol{\alpha}^n=(\alpha_1^n,\cdots,\alpha_m^n)^\top \in \Rb^m$ s.t. for any $h \in [H]$ and $(o_h,a_h) \in \Oc \times \Ac$,
    \begin{align*}
        \textstyle \phi_h^n(o_h,a_h)=\sum_{l=1}^m\alpha_l^n \phi_h^l(o_h,a_h), \quad \Mbf_h^n(o_h,a_h)=\sum_{l=1}^m\alpha_l^n \Mbf_h^l(o_h,a_h).
    \end{align*}
    For regularization, we assume $0 \leq \alpha^n_l \leq 1$ for all $l \in [m]$ and $n\in [N]$, and $\sum_{l=1}^m\alpha^n_l=1$ for all $n \in [N]$. It can be shown that 
    %the log bracketing number is 
    $\beta^{(N)}=O(   m( r^2  |\mathcal{O}| |\mathcal{A}| H^2+ N) \log\frac{ |\mathcal{O}| |\mathcal{A}| }{\eta})$, whereas $\beta^{(1)}$ for a single task is given by $r^2|\mathcal{O}||\mathcal{A}|H^2\log\frac{|\mathcal{O}||\mathcal{A}|}{\eta}$.
Clearly, $\beta^{(N)} \ll N\beta^{(1)}$ holds if $m \leq \min\{N,r^2  |\mathcal{O}| |\mathcal{A}| H^2\}$, which is satisfied  in practice.

\begin{proof}[Proof of \Cref{eg: linear com}]
     Denote the core tasks model class $\ThetaB^0 = \Theta^{0,1} \times \cdots \times \Theta^{0,m}$, and the multi-task model class $\ThetaB=\Theta^1\times\cdots\times\Theta^N$, where for any $n \in [N]$, $\Theta^n$ is defined as:  
    \begin{align*}
    {\Theta}^n&= \left\{ \left(\sum_{l=1}^m{\alpha}_l^n{\phi}_H^l,\left\{\sum_{l=1}^m{\alpha}_l^n{\Mbf}^l_h\right\}_{h=1}^H\right):{\boldsymbol{\alpha}}^n=({\alpha}^n_1,\cdots,{\alpha}^n_m)^\top \in \Rb^m;\left({\phi}_H^l,\{{\Mbf}^l_h\}_{h=1}^H\right) \in {\Theta}^{0,l}
    % ,\norm{\tilde{\phi}_H^t-\phi_H^t}_\infty \leq 
    % \delta, \phi_H^t \in \Rb^r 
    \right\}.
    \end{align*}
    
    For any $\delta \geq 0$, we first consider the $\delta$-cover of the model class of core tasks $\widetilde{\ThetaB}^0_{{\delta}}=\widetilde{\Theta}^{0,1}_{{\delta}} \times \cdots \times \widetilde{\Theta}^{0,m}_{{\delta}}$, where for each base tasks $l \in [m]$, $\widetilde{\Theta}^l_{{\delta}}$ is a $\delta$-cover for ${\Theta}^m$. Similar to proof of \Cref{lemma: bracketing number of single-task}, $|\widetilde{\Theta}^{0,l}_{\delta}|=(3\sqrt{r} \times \frac{1}{\delta})^{r^2H^2|\Oc||\Ac|}$ and $|\widetilde{\ThetaB}^0_{\delta}|=(3\sqrt{r} \times \frac{1}{\delta})^{r^2H^2|\Oc||\Ac|m}$. 
    
    Then we consider the cover for multi-task parameter space. Denote $\boldsymbol{C}_{\delta}^m$ as a $\delta$-cover for the unit ball in $\Rb^m$ w.r.t. $\ell_1$ norm. Mathematically, for any $n \in [N]$ and vector $\boldsymbol{\alpha}^n \in \Rb^m$, there exists an $\tilde{\boldsymbol{\alpha}}^n \in \boldsymbol{C}_{\delta}^m$ such that $\norm{\boldsymbol{\alpha}^n-\tilde{\boldsymbol{\alpha}}^n}_1 \leq \delta$.  In addition, the cardinality $|\boldsymbol{C}_{\delta}^m|=(\frac{3}{\delta})^m$.

    Define the multi-task model class $\widetilde{\ThetaB}_\delta=\widetilde{\Theta}_\delta^1\times\cdots\times\widetilde{\Theta}_\delta^N$, where for any $n \in [N]$, $\widetilde{{\Theta}}^n_\delta$ is defined as 
    \begin{align*}
    \widetilde{{\Theta}}^n_\delta&= \left\{ \left(\sum_{l=1}^m\tilde{\alpha}_l^n\widetilde{\phi}_H^l,\left\{\sum_{l=1}^m\tilde{\alpha}_l^n\widetilde{\Mbf}^l_h\right\}_{h=1}^H\right):\tilde{\boldsymbol{\alpha}}^n=(\tilde{\alpha}^n_1,\cdots,\tilde{\alpha}^n_m)^\top \in \boldsymbol{C}_{\delta}^m;\left(\widetilde{\phi}_H^l,\{\widetilde{\Mbf}^l_h\}_{h=1}^H\right) \in \widetilde{\Theta}^{0,l}_{\delta}
    % ,\norm{\tilde{\phi}_H^t-\phi_H^t}_\infty \leq 
    % \delta, \phi_H^t \in \Rb^r 
    \right\}.
    \end{align*}

    % $C_\delta=C^0_{\delta} \bigoplus C^{\perp}_{\delta}$. After an approximation to base task is selected from $C^0_{\delta}$, in fact the cover set of parameter spaces of remaining tasks become much smaller. We only need to cover the coefficient $\mathbf{\alpha}^n=(\alpha_1^n, \cdots, \alpha_k^n) \in \Rb^k$. choose a $\frac{\delta}{\sqrt{r}}$ cover w.r.t. $\ell_1$ norm: the cardinality of cover set is $(\frac{\sqrt{r}}{\delta})^k$.
    We next show that $\widetilde{\ThetaB}_\delta$ is a $\eta$-bracket of $\ThetaB_\delta$.
    
    By definition, for any model of base task $l \in [m]$: $\theta^l=\left({\phi}_H^l,\{{\Mbf}^l_h\}_{h=1}^H\right) \in \Theta^l$, there exists $\tilde{{\theta}}^l = \left(\widetilde{\phi}_H^l,\{\widetilde{\Mbf}^l_h\}_{h=1}^H\right) \in \widetilde{\Theta}^l_{\delta}$, such that for any $(o,a) \in \Oc \times \Ac$, 
    \begin{align}
        \norm{\phi^l_H-\tilde{\phi}^l_H}_{\infty} \leq \delta, 
        \norm{\textbf{Vec}(\Mbf^l_h(o,a))-\textbf{Vec}(\widetilde{\Mbf}^l_h(o,a))}_{\infty} \leq \delta. \nonumber
    \end{align}

    Then, for any $\theta^l \in \Theta^l,\boldsymbol{\alpha}^n \in \Rb^m$, there exist $\tilde{\theta}\in \tilde{\Theta}^l_{\delta}$ and $\tilde{\boldsymbol{\alpha}}^n \in \boldsymbol{C}_\delta$ such that
    \begin{align*}
    &\norm{\sum_{l=1}^m{\alpha}_l^n{\phi}_H^l-\sum_{l=1}^m\tilde{\alpha}_l^n\widetilde{\phi}_H^l}_\infty\\
    & \quad \leq \sum_{l=1}^m|\alpha_l^n|\norm{{\phi}_H^l-\widetilde{\phi}_H^l}_\infty+\sum_{l=1}^m|\widetilde{\alpha}_l^n-\alpha_l^n|\norm{\widetilde{\phi}_H^l}_\infty\\
    & \quad \leq \sum_{l=1}^m\alpha_l^n\delta+\sum_{l=1}^m|\widetilde{\alpha}_l^n-\alpha_l^n|\leq 2\delta,
    \end{align*}
    and for any $(o_h,a_h)\in \Oc \times \Ac$
    \begin{align*}
        &\norm{\sum_{l=1}^m\alpha_l^n\Mbf_h^l- \sum_{l=1}^m\widetilde{\alpha}_l^n\widetilde{\Mbf}_h^l}_\infty\\
        &\quad\leq \sum_{l=1}^m|\alpha_l^n|\norm{\Mbf_h^l-\widetilde{\Mbf}_h^l}_\infty+\sum_{l=1}^m|\widetilde{\alpha}_l^n-\alpha_l^n|\norm{\widetilde{\Mbf}_h^l}_\infty\\
        &\quad\leq \sum_{l=1}^m\alpha_l^n\sqrt{r}\norm{\textbf{Vec}(\Mbf_h^l)-\textbf{Vec}(\widetilde{\Mbf}_h^l)}_\infty+\sum_{l=1}^m|\widetilde{\alpha}_l^n-\alpha_l^n|\norm{\widetilde{\Mbf}_h^l}_\infty\\
        & \quad \leq \sum_{l=1}^m\alpha_l^n\sqrt{r}\delta+\sum_{l=1}^m\sqrt{r}|\tilde{\alpha}_l^n-\alpha_l^n|\\
        & \quad = \sqrt{r}\delta+\sqrt{r}\norm{\tilde{\alpha}^n-\alpha^n}_1\leq 2\sqrt{r}\delta.
    \end{align*}

    Similar to the analysis in the proof of \Cref{lemma: bracketing number of single-task}, specifically, \Cref{Eq: lemma0single-1}, $\widetilde{\Theta}^t_\delta$ can constitute an $\eta$-bracket for $\Theta^t$ with $\delta=\frac{\eta}{2\sqrt{r}(|\Oc||\Ac|)^{cH}}$. 
    
    The cardinality of the cover of multi-task model class is  $|\widetilde{\ThetaB}_\delta|=|\widetilde{\ThetaB}_\delta^0||\boldsymbol{C}_\delta|^N=(3\sqrt{r} \times \frac{1}{\delta})^{r^2H^2|\Oc||\Ac|m} \times
    (\frac{3}{\delta})^{mN}$.
    
    In conclusion, the log $\eta$-bracketing number is     
    $O\left(r^2H^2|\Oc||\Ac|m\log(\frac{rH|\Oc||\Ac|}{\eta})+mN\log(\frac{rH|\Oc||\Ac|}{\eta})\right) $.
\end{proof}

\subsection{ Bracketing Number of Downstream Examples }\label{Appd_sub: F-4}

\textbf{ \Cref{eg:same transtion} } Note that the model class for downstream learning is $\hat{\Theta}_0^{\mathrm{u}} = \{\mathbb{O}_h\}_{h\in[H]}$. Thus, we immediately obtaint that $\log\mathcal{N}_{\eta}(\hat{\Theta}_0^{\mathrm{u}}) = O(H|\mathcal{O}||\mathcal{S}|\log\frac{|\mathcal{O}||\mathcal{S}|}{\eta})$

\textbf{\Cref{eg:M+delta-ds}(Multi-task PSR with perturbed models):}
Similar to the upstream tasks, the downstream task $0$ is also a noisy perturbation of the latent base task. Specifically, for each step $h \in [H]$, any $(o,a) \in \Oc \times \Ac$, we have 
     \begin{align}
         \textstyle \phi_H^0=\phi_H^{\mathrm{b}}, \Mbf_h^{0}(o_h,a_h)=\Mbf_h^\mathrm{b}(o_h,a_h)+\Delta^{0}_h(o_h,a_h), \quad \Delta^{0}_h \in \boldsymbol{\Delta}. \label{eq:eg3-1}
     \end{align} 
The log $\eta$-bracketing number is at most $H\log|\boldsymbol{\Delta}|$. 
\begin{proof}
    Suppose the estimated model parameter of the base task is $\bar{\theta}^{\mathrm{b}}=\{\bar{\phi}_H^{\mathrm{b}},\{\overline{\Mbf}_h^{\mathrm{b}}\}_{h=1}^H\}$.
     Because the downstream task model parameters satisfy \Cref{eq:eg3-1}, the empirical candidate model class can be characterized as
     \begin{align*}
         \hat{\Theta}_0^{\mathrm{u}}&=\left.\{\theta: \theta=\{\phi_H^{0},\{\Mbf^0_h\}_{h=1}^H\};\phi_H^0=\bar{\phi}_H^{\mathrm{b}};\right.\\
         &\qquad\qquad\qquad\left.
         \Mbf_h^{0}(o_h,a_h)= \overline{\Mbf}_h^{\mathrm{b}}(o_h,a_h)+\Delta_h^0(o_h,a_h), (o_h,a_h)\in \Oc\times\Ac;\Delta_h^0\in\boldsymbol{\Delta} \right.\}.
     \end{align*}
     If $\bar{\theta}^{\mathrm{b}}$ is given, then the candidate model class is decided by $\boldsymbol{\Delta}$. Then for any $\eta > 0$, the $\eta$-bracketing number is $|\boldsymbol{\Delta}|^H$.
\end{proof}

\textbf{\Cref{eg: linear com-ds}(Multi-task PSRs: Linear combination of core tasks):} Suppose the downstream task $0$ also lies in the linear span of $m$ core tasks same as the upstream. Moreover, assume the upstream tasks are diverse enough to span the whole core tasks space. As a result, the downstream task can be represented as a linear combination of a subset of the upstream tasks. Specifically, there exists a constant $L$ satisfying $m \leq L \leq N$ and a coefficient vector $\boldsymbol{\alpha}^{0}=(\alpha_1^{0},\cdots,\alpha_L^{0})^\top \in \Rb^L$ s.t. for any $h \in [H]$ and $(o_h,a_h) \in \Oc \times \Ac$,
\begin{align}
        \textstyle \phi_H^{0}=\sum_{l=1}^L\alpha_l^{0} \phi_H^l, \quad \Mbf_h^{0}(o_h,a_h)=\sum_{l=1}^L\alpha_l^{0} \Mbf_h^l(o_h,a_h). \label{eq:eg4-1}
\end{align}
    For regularization, we assume $0 \leq \alpha^{0}_l \leq 1$ for all $l \in [L]$, and $\sum_{l=1}^L\alpha^{0}_l=1$. It can be shown that 
    %the log bracketing number is 
    $\beta_0= O(LH\log(\frac{r|\Oc||\Ac|}{\eta}))$, whereas $\beta^{(1)}$ for learning without prior information is given by $O(r^2|\mathcal{O}||\mathcal{A}|H^2\log\frac{|\mathcal{O}||\mathcal{A}|}{\eta})$.
Clearly, $\beta_0 \ll \beta^{(1)}$ holds if $m \leq \min\{N,r^2  |\mathcal{O}| |\mathcal{A}| H^2\}$, which is satisfied  in practice. 

\begin{proof}
     Suppose for the upstream task $l\in[L]$, the estimated model parameter is $\bar{\theta}_l=\{\bar{\phi}_H^l,\{\overline{\Mbf}_h^l\}_{h=1}^H\}$.
     Because the downstream task model parameters satisfy \Cref{eq:eg4-1}, the empirical candidate model class can be characterized as
     \begin{align*}
         \hat{\Theta}_0^{\mathrm{u}}&=\left.\{\theta: \theta=\{\phi_H^{0},\{\Mbf^0_h\}_{h=1}^H\};\phi_H^{0}=\sum_{l=1}^L\alpha_n^{0} \bar{\phi}_H^l;\right.\\
         &\qquad\qquad\qquad\left.\Mbf_h^{0}(o_h,a_h)=\sum_{l=1}^L\alpha_l^{0} \overline{\Mbf}_h^l(o_h,a_h), (o_h,a_h)\in \Oc\times\Ac;\boldsymbol{\alpha}^{0}\in\Rb^L \right.\}.
     \end{align*}
    Then we consider the cover for $\hat{\Theta}_0^{\mathrm{u}}$. Denote $\boldsymbol{C}_{\delta}^L$ as a $\delta$-cover for the unit ball in $\Rb^L$ w.r.t. $\ell_1$ norm. Mathematically, for any vector $\boldsymbol{\alpha}^{0} \in \Rb^L$, there exists an $\tilde{\boldsymbol{\alpha}}^{0} \in \boldsymbol{C}_{\delta}^L$ such that $\norm{\boldsymbol{\alpha}^{0}-\tilde{\boldsymbol{\alpha}}^{0}}_1 \leq \delta$.  In addition, the cardinality $|\boldsymbol{C}_{\delta}^L|=(\frac{3}{\delta})^L$. 
    
    Define 
    \begin{align*}
         \widetilde{\Theta}_\delta^{\mathrm{u}}&=\left.\{\theta: \theta=\{\phi_H^{0},\{\Mbf_h\}_{h=1}^H\};\phi_H^{0}=\sum_{l=1}^L\tilde{\alpha}_l^{0} \bar{\phi}_H^l;\right.\\
         &\qquad\qquad\qquad\left.\Mbf_h^{0}(o_h,a_h)=\sum_{l=1}^L\tilde{\alpha}_l^{0} \overline{\Mbf}_h^l(o_h,a_h), (o_h,a_h)\in \Oc\times\Ac;\tilde{\boldsymbol{\alpha}}^{0}\in \boldsymbol{C}_{\delta}^L\right.\}.
     \end{align*}
     Then, for any $\theta^0 \in \hat{\Theta}_0$ with $\boldsymbol{\alpha}^0 \in \Rb^L$, there exist $\tilde{\theta}\in \tilde{\Theta}^{\mathrm{u}}_{\delta}$ with $\tilde{\boldsymbol{\alpha}}^0 \in \boldsymbol{C}_\delta^L$ such that
    \begin{align*}
    &\norm{\sum_{l=1}^L\alpha_l^{0} \bar{\phi}_H^l-\sum_{l=1}^L\tilde{\alpha}_l^{0} \bar{\phi}_H^l}_\infty
    \leq \sum_{l=1}^L|\alpha_l^{0}-\tilde{\alpha}_l^{0}|\norm{ \bar{\phi}_H^l}_\infty \leq \norm{\boldsymbol{\alpha}^0-\tilde{\boldsymbol{\alpha}}^0}_1 \leq \delta,
    \end{align*}
    and for any $(o_h,a_h)\in \Oc \times \Ac$
    \begin{align*}
        &\norm{\sum_{l=1}^L{\alpha}_l^{0} \overline{\Mbf}_h^l(o_h,a_h)-\sum_{l=1}^L\tilde{\alpha}_l^{0} \overline{\Mbf}_h^l(o_h,a_h)}_\infty\\
        & \qquad \qquad \leq \sum_{l=1}^L|\alpha_l^0-\tilde{\alpha}_l^0|\norm{\overline{\Mbf}_h^l(o_h,a_h)}_\infty \leq \sqrt{r}\norm{\boldsymbol{\alpha}^0-\tilde{\boldsymbol{\alpha}}^0} \leq \sqrt{r}\delta.
    \end{align*}

     Similar to the analysis in the proof of \Cref{lemma: bracketing number of single-task}, specifically, \Cref{Eq: lemma0single-1}, $\widetilde{\Theta}^{\mathrm{u}}_\delta$ can constitute an $\eta$-bracket for $\widehat{\Theta}_0^{\mathrm{u}}$ with $\delta=\frac{\eta}{2\sqrt{r}(|\Oc||\Ac|)^{cH}}$. 
    
    The cardinality of the cover of multi-task model class is  $|\widetilde{\Theta}^{\mathrm{u}}_\delta|=|\boldsymbol{C}_{\delta}^L|=(\frac{3}{\delta})^L$. In conclusion, the log $\eta$-bracketing number is     
    $O\left(LH\log(\frac{r|\Oc||\Ac|}{\eta})\right) $.
\end{proof}

\section{Examples of Multi-task MDPs from previous work}\label{Appd_sub: F-1}
To demonstrate that our framework encompasses multi-task learning under MDPs, we provide several examples of MDPs from previous work in this subsection. Suppose $\Sc$ is the state space, there exist $N$ source tasks, and $P^{(*,n)}: \Sc \times \Ac \times \Sc \to \Rb$ is the true transition kernel of task $n$.
\begin{figure}
\centering    
\begin{overpic}[width=0.6\textwidth]{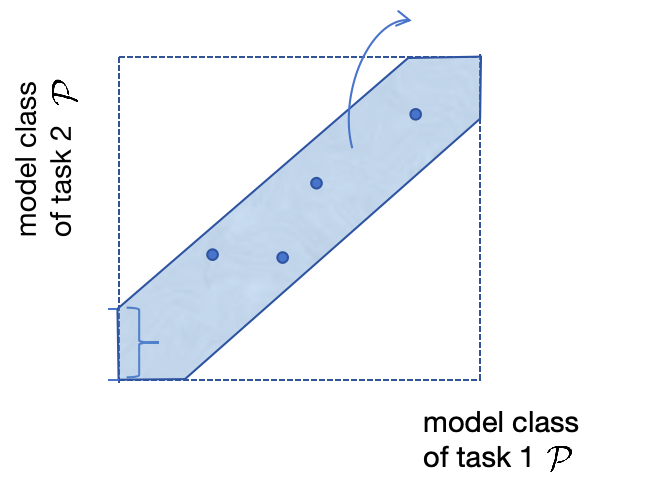} 
    
    \put(63,70){\footnotesize$\Fc_{\mathrm{u}}\subset\mathcal{P}^2$}
    % \put(37,39){\fontsize{2pt}{10pt}\selectfont$\mathbf{f}^*=\left(P^{(*,1)},  P^{(*,2)}\right)$}
    % \put(5,10){\fontsize{2pt}{10pt}\selectfont$\mathbf{f}^{0,1}=(P^0,  P^1)$}
    % \put(5,28){\fontsize{2pt}{10pt}\selectfont$\mathbf{f}^{0,1}=(P^0,  P^1)$}
    % \put(25,21){\fontsize{1pt}{10pt}\selectfont$ \mathtt{D}_{\TV}(P^0(\cdot|s,a),P^1(\cdot|s,a)) \leq \epsilon $}
    \put(37,39){\fontsize{1pt}{10pt}\selectfont$\left(P^{(*,1)},  P^{(*,2)}\right)$}
    \put(9,12){\fontsize{2pt}{10pt}\selectfont$P^0$}
    \put(9,26){\fontsize{2pt}{10pt}\selectfont$P^1$}
    \put(25,21){\fontsize{1pt}{10pt}\selectfont$ \mathtt{D}_{\TV}(P^0(\cdot|s,a),P^1(\cdot|s,a)) \leq \epsilon $}
    
\end{overpic}
\caption{Supplementary illustration of joint model class in two dimensions for MDPs. A concrete example of tasks with similar transition kernels, i.e., any point $(P^0,  P^1)$ in the joint class $ \Fc_{\mathrm{u}}$ satisfies $\max_{(s,a)\in\Sc\times\Ac}\mathtt{D}_{\TV}(P^0(\cdot|s,a),P^1(\cdot|s,a)) \leq \epsilon$ for a small positive constant $\epsilon$.  }
\label{fig:2}  
\end{figure}
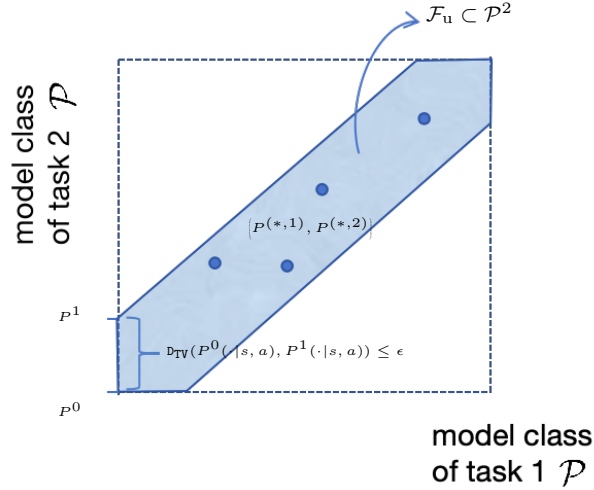

\cite{cheng2022provable} studied multi-task learning under low-rank MDPs in which the transition kernel $P^{(*,n)}$ has a $d$ dimension low-rank decomposition into two embedding functions $\phi^{(*)}: \Sc \times \Ac \to \Rb^d, \mu^{(*,n)}: \Sc \to \Rb^d$ as $P^{(*,n)}(s,a,s^{'})=\langle\phi^{(*)}(s,a),\mu^{(*,n)}(s^{'})\rangle$ for all $(s,a,s^{'}) \in \Sc \times \Ac \times \Sc$ for each task $n$. In this setting, \cite{cheng2022provable} assume the $N$ source tasks share common representations $\phi^{(*)}$ and for each task $n$, $\phi^{(*)} \in \Phi, \mu^{(*,n)} \in \Psi$ for finite model class $\{\Phi,\Psi\}$. Consequently, the $\eta$-bracketing number for the multi-task low-rank MDPs model class is at most $O(H|\Phi||\Psi|^N)$, which is much smaller than the one of the individual single-task with $O(H|\Phi|^N|\Psi|^N)$ if $|\Phi| \gg |\Psi|$ .

\cite{zhang2021provably} studied multi-task learning under tabular MDPs with an assumption that for any two tasks $n_1,n_2 \in [N]$, it holds that $\max_{(s,a) \in \Sc\times \Ac}\DTV(P^{(*,n_1)}(\cdot|s,a)|P^{(*,n_2)}(\cdot|s,a)) \leq \epsilon$ for some small $\epsilon>0$. Then the multi-task model class is much smaller than the model class of $N$ individual single-task (see \Cref{fig:2} for an illustration when $N=2$). Consequently, the $\eta$-bracketing number for the multi-task low-rank MDPs model class is at most $O ( H|\mathcal{S}|^2|\mathcal{A}|(\log\frac{H|\mathcal{A}||\mathcal{S}|}{\eta}+(N-1)\log\frac{H|\mathcal{A}||\mathcal{S}|\epsilon}{\eta}))$. This is smaller than that of each individual single-task, which is $O ( HN|\mathcal{S}|^2|\mathcal{A}|\log\frac{H|\mathcal{A}||\mathcal{S}|}{\eta})$ if $\epsilon \leq \eta$ and $N \geq 1$.

\section{Auxillary lemmas}

The following lemma characterizes the relationship between the total variation distance and the Hellinger-squared distance. Note that the result for probability measures has been proved in Lemma H.1 in \citet{zhong2022posterior}. Since we consider  more general bounded measures, we provide the full proof for completeness.
\begin{lemma}\label{lemma:TV and hellinger}
Given two bounded measures $P$ and $Q$ defined on the set $\mathcal{X}$, let $|P| = \sum_{x\in\mathcal{X}} P(x)$ and $|Q| = \sum_{x\in\mathcal{X}}Q(x).$ We have
\[ \mathtt{D}_{\TV}^2(P,Q)  \leq 4(|P|+|Q|)\mathtt{D}_{\mathtt{H}}^2(P,Q).  \]
In addition, if $P_{Y|X}, Q_{Y|X}$ are two conditional distributions over a random variable $Y$, and $P_{X,Y} = P_{Y|X}P$, $Q_{X,Y}= Q_{Y|X}Q$ are the joint distributions when $X$ follows the distributions $P$ and $Q$, respectively, we have
\[\mathop{\Eb}_{X\sim P }\left[ \mathtt{D}_{\mathtt{H}}^2(P_{Y|X} (\cdot|X),Q_{Y|X}(\cdot|X))\right] \leq 8\mathtt{D}_{\mathtt{H}}^2 (P_{X,Y},Q_{X,Y}). \]
\end{lemma}

\begin{lemma}
Suppose $\Pb$ and $\Qb$ are two probability distributions. For any $\alpha > 1$, we have the following inequality.
    \begin{align*}
        \mathtt{D}_{\TV}(\mathbb{P}, \mathbb{Q}) \leq \sqrt{ \frac{1}{2} \mathtt{D}_{\mathtt{R},\alpha} (\mathbb{P}, \mathbb{Q}) }.
    \end{align*}
\end{lemma}
\begin{proof}
    By Pinsker's inequality, we have 
    \begin{align*}
        \mathtt{D}_{\TV}(\mathbb{P}, \mathbb{Q}) \leq \sqrt{\frac{1}{2}\mathtt{D}_{\mathtt{KL}}(\mathbb{P}, \mathbb{Q})}.
    \end{align*}
    By Theorem 5 from \cite{van2014renyi}, we have
    \begin{align*}
        \mathtt{D}_{\mathtt{KL}}(\mathbb{P}, \mathbb{Q}) = \lim_{\alpha \uparrow 1} \mathtt{D}_{\mathtt{R},\alpha} (\mathbb{P}, \mathbb{Q}) \leq \inf_{\alpha >1} \mathtt{D}_{\mathtt{R},\alpha} (\mathbb{P}, \mathbb{Q}).
    \end{align*}
\end{proof}

\begin{lemma}[Elliptical potential lemma  ]\label{lemma:elliptical potential lemma}
    For any sequence of vectors $\mathcal{X} = \{x_1,\ldots,x_n,\ldots\}\subset\mathbb{R}^d$, let $U_k = \lambda I + \sum_{t<k}x_kx_k^{\top}$, where $\lambda $ is a positive constant, and $B>0$ is a real number.  If the rank of $\mathcal{X}$ is at most $r$, then, we have
    \[ 
    \begin{aligned}
        & \sum_{k=1}^K \min\left\{\|x_k\|_{U_k^{-1}}^2 , B\right\}\leq   (1+B)r\log(1+K/\lambda), \\
        &\sum_{k=1}^K \min\left\{\|x_k\|_{U_k^{-1}} , \sqrt{B}\right\} \leq \sqrt{(1+B)rK\log(1+K/\lambda)}.
    \end{aligned}
    \]
\end{lemma}
%\jing{why do we need the proof?}\hrqc{slightly general}
%\yl{the big left braket is not necessary}
\begin{proof}
    Note that the second inequality is an immediate result from the first inequality by the Cauchy's inequality. Hence, it suffices to prove the first inequality. To this end, we have
\begin{align*}
    \sum_{k=1}^K  \min\left\{\|x_k\|_{U_k^{-1}}^2 , B \right\} & \overset{\RM{1}}\leq (1+B)\sum_{k=1}^K  \log\left(1 + \|x_k \|_{U_k^{-1}}^2 \right)\\
    & =   (1+B)\sum_{k=1}^K  \log \left( 1+ \mathtt{trace}\left( \left(U_{k+1} - U_k \right) U_k^{-1} \right) \right) \\
    &= (1+B)\sum_{k=1}^K \log\left( 1+ \mathtt{trace}\left( U_k^{-1/2}\left(U_{k+1} - U_k \right) U_k^{-1/2} \right) \right)\\
    &\leq (1+B)\sum_{k=1}^K \log\mathtt{det}\left(I_d + U_k^{-1/2}\left(U_{k+1} - U_k \right) U_k^{-1/2} \right)\\
    &=  (1+B)\sum_{k=1}^K \log \frac{\mathtt{det}\left(U_{k+1}\right)}{\mathtt{det}(U_k)}\\
    & = (1+B) \log \frac{\mathtt{det}(U_{K+1})}{\mathtt{det}(U_1)} \\
    & = (1+B)\log\mathtt{det}\left( I + \frac{1}{\lambda}\sum_{k=1}^{K}x_kx_k^{\top}  \right) \\
    &\overset{\RM{2}}\leq (1+B)r\log(1+K/\lambda),
\end{align*}
where $\RM{1}$ follows because $x\leq (1+B)\log(1+x)$ if $0<x\leq B$, and $\RM{2}$ follows because $\mathtt{rank}(\mathcal{X}) \leq r.$
\end{proof}

% \jing{what are those for??}

%     Consider a set of barycentric spnner of $\{\psi^*(\tau_{h-1})/\|\psi^*(\tau_{h-1})\|_1\}_{\tau_{h-1}}$: $\{ x_{h-1}^{\ell}\}_{i=1}^r$. We have $\|x\|_1 = 1$. 

%     Let $\mathbf{A}_{h-1} = [x_{h-1}^1,\ldots,x_{h-1}^r]\in\mathbb{R}^{|\mathcal{Q}_{h-1}|\times r}$.

% Now let 
% \begin{align}
%     U_{n,h-1} = \lambda \sum_i x_{h-1}^ix_{h-1}^{i\top} + \sum_{t<k}\mathop{\Eb}_{\tau_{h-1}\sim \Pb_{\theta}^{\pi^t}} \bar{\psi}(\tau_{h-1})\bar{\psi}(\tau_{h-1})^{\top}
% \end{align}

To compute the bracketing number of mulit-task model class, we first require a basic result on the covering number of a Euclidean ball as follows. Proof of the lemma can be found in Lemma 5.2 in \cite{vershynin2010introduction}.
\begin{lemma}[Covering Number of Euclidean Ball]\label{lem:aux:cov} 
For any $\epsilon >0$, the $\epsilon$-covering number of the Euclidean ball in $\Rb^d$ with radius $R > 0$ is upper bounded by $(1 + 2R/\epsilon)^d$.
\end{lemma}

\end{document}